\documentclass[twoside]{article}

\usepackage[accepted]{aistats2026}
\usepackage{fancyhdr}
\usepackage{algorithm}
\usepackage{amsfonts}
\usepackage{amsmath, amssymb}
\usepackage{amsthm}
\usepackage{algpseudocode}
\usepackage{float}
\usepackage{comment}
\usepackage{color}
\usepackage{booktabs}
\usepackage{multirow}

\theoremstyle{plain}
\newtheorem{theorem}{Theorem}[section]
\newtheorem{proposition}[theorem]{Proposition}
\newtheorem{corollary}[theorem]{Corollary}
\newtheorem{lemma}[theorem]{Lemma}

\theoremstyle{remark}

\theoremstyle{definition}
\newtheorem{definition}[theorem]{Definition}
\newtheorem{assumption}[theorem]{Assumption}

\newcommand{\R}{\mathbb{R}}
\newcommand{\by}{\mathbf{y}}
\newcommand{\bx}{\mathbf{x}}
\newcommand{\bX}{\mathbf{X}}
\newcommand{\bS}{\mathbf{S}}
\newcommand{\bK}{\mathbf{K}}
\newcommand{\bM}{\mathbf{M}}
\newcommand{\bI}{\mathbf{I}}

\newcommand{\bP}{\mathbf{P}}
\newcommand{\bone}{\mathbf{1}}
\newcommand{\bbeta}{\boldsymbol{\beta}}
\newcommand{\bb}{\mathbf{b}}
\newcommand{\bB}{\mathbf{B}}
\newcommand{\bt}{\mathbf{t}}
\newcommand{\bs}{\mathbf{s}}
\newcommand{\mK}{\mathcal{K}}
\newcommand{\jmat}{\mathbf{J}_n}

\newcommand{\bH}{\mathbf{H}}
\newcommand{\bE}{\mathbf{E}}
\newcommand{\bh}{\mathbf{h}}

\newcommand{\bD}{\mathbf{D}}
\newcommand{\bZ}{\mathbf{Z}}

\newcommand{\bv}{\mathbf{v}}
\newcommand{\f}{\mathbf{f}}
\newcommand{\E}{\mathbb{E}}
\newcommand{\norm}[1]{\left\| #1 \right\|}
\newcommand{\Smatk}{\mathbf{S}^{(k)}}
\newcommand{\Kmatk}{\mathbf{K}^{(k)}}
\newcommand{\Kmat}{\mathbf{K}}
\newcommand{\Smat}{\mathbf{S}}
\newcommand{\sveck}{\mathbf{s}^{(k)}}
\newcommand{\kveck}{\mathbf{k}^{(k)}}
\newcommand{\rveck}{\mathbf{r}^{(k)}}

\usepackage{bbm}
\usepackage[round]{natbib}
\usepackage[skins,theorems]{tcolorbox}
\usepackage[colorlinks=true,allcolors=blue]{hyperref}

\begin{document}
	
	%
	
	%
	
	
	\twocolumn[

    \aistatstitle{Statistical Inference for Explainable Boosting Machines}
    
    \aistatsauthor{
      Haimo Fang$^{1}$ \And
      Kevin Tan$^{2}$ \And
      Jonathan Pipping-Gam\'{o}n$^{2}$ \And
      Giles Hooker$^{2}$
    }
    
    \aistatsaddress{
      $^{1}$School of Economics, Fudan University \\
      $^{2}$Department of Statistics and Data Science, The Wharton School, University of Pennsylvania
    }

    ]

	\begin{abstract}
		Explainable boosting machines (EBMs) are popular ``glass-box'' models that learn a set of univariate functions using boosting trees. These achieve explainability through visualizations of each feature’s effect. However, unlike linear model coefficients, uncertainty quantification for the learned univariate functions requires computationally intensive bootstrapping, making it hard to know which features truly matter. We provide an alternative using recent advances in statistical inference for gradient boosting, deriving methods for statistical inference as well as end-to-end theoretical guarantees. Using a moving average instead of a sum of trees (Boulevard regularization) allows the boosting process to converge to a feature-wise kernel ridge regression. This produces asymptotically normal predictions that achieve the minimax-optimal MSE for fitting Lipschitz GAMs with $p$ features of $O(p\, n^{-2/3})$, successfully avoiding the curse of dimensionality. We then construct prediction intervals for the response and confidence intervals for each learned univariate function with a runtime independent of the number of datapoints, enabling further explainability within EBMs. Code is available at \href{https://github.com/hetankevin/ebm-inference}{github.com/hetankevin/ebm-inference}.

	\end{abstract}
	
	\section{Introduction}
	Modern machine learning methods have proven to be highly successful, achieving impressively high performance even for complicated prediction tasks. Still, this comes at a cost. These methods have been critiqued for their lack of transparency, often described as ``black-box'' models. As such, much work has been done in an attempt to provide explanations for their behavior.\footnote{Popular methods include LIME \cite{ribeiro2016should}, SHAP \cite{lundberg2017unified}, ALE \cite{apley2020visualizing} and partial dependence \cite{friedman2001greedy} plots.}Another line of work attempts to convince practitioners to instead use inherently transparent models that are algebraically simple enough to be inspected by humans \cite{rudin2019stop}. Neither is optimal. Explaining black-box models is a difficult problem \cite{hooker2021unrestrictedpermutationforcesextrapolation, fryer2021shapleyvaluesfeatureselection}, and inherently interpretable models can sacrifice performance. 
	
	In light of this, Explainable Boosting Machines (EBM) \cite{yin2012intelligible, nori2019interpretmlunifiedframeworkmachine} have emerged as a popular ``glass-box'' model to strike a balance between both. An EBM learns a Generalized Additive Model (GAM) \cite{hastie1986gams}: $\f(\bx) = \sum_{k=1}^{p}\f^{(k)}(\bx^{(k)})$, where each $\f^{(k)}$ is assumed to be a univariate function of the $k$-th feature. However, instead of fitting splines, it instead fits a set of univariate gradient boosting regression trees. The tree-based boosting procedure attempts to recover the expressiveness of state-of-the-art algorithms \cite{chen2016xgboost, ke2017lightgbm}, while the decomposition into additive components allows the measurement of the marginal contribution of each feature like linear models.  While EBMs impose strong structural constraints on the resulting prediction, their ability to reproduce nonlinearities, along with appropriate choices of additive structure, provides comparable performance to SOTA boosting \cite{nori2019interpretmlunifiedframeworkmachine}.

	\begin{figure*}[ht!]
    \vspace{-2ex}
		\centering
		\includegraphics[width=\linewidth]{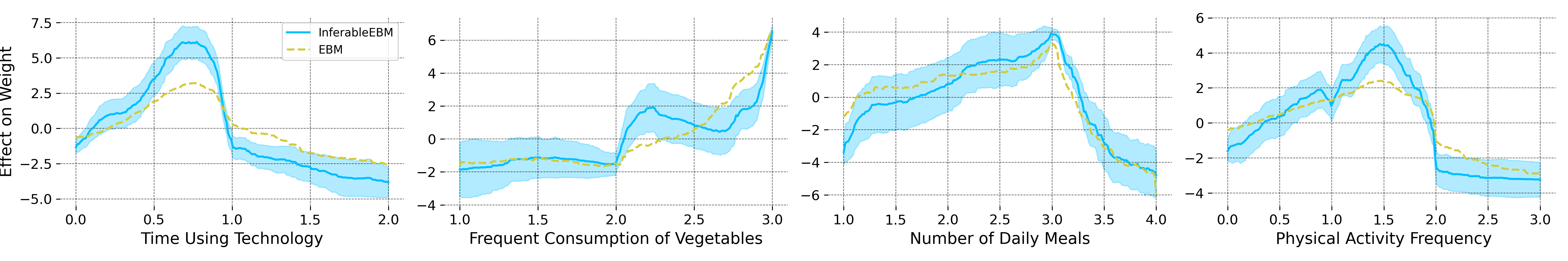}
		\vspace{-4ex}
		\caption{Example of the feature-specific confidence intervals generated by Algorithm \ref{alg:brebm_ident}, compared to the point estimates from EBM \cite{nori2019interpretmlunifiedframeworkmachine}. This visualizes centered feature effects on the UCI Machine Learning Reposity Obesity \cite{estimation_of_obesity_levels_based_on_eating_habits_and_physical_condition__544} dataset, when predicting weight as the response. }
		\label{fig:ebm-effects-wide}
        \vspace{-2ex}
	\end{figure*}
	
	While EBMs provide a clear graphical representation for the effects of each feature, these are still point estimates. The central question we ask for uncertainty quantification can be phrased as this: \textit{Given a model design, if we collect a new dataset and retrain the model, how different would our new prediction be?} In response to this question, vanilla EBMs \cite{yin2012intelligible} compute variance estimates via bootstrapping, though this is computationally intensive and only heuristically justified. Our primary objective is hence \emph{well-calibrated frequentist uncertainty quantification} for EBMs while retaining their interpretability and scalability.

	A canonical alternative to bootstrapping is to construct confidence and prediction intervals using the asymptotic properties of the model. This is the case with GAMs fitted via the backfitting (but not boosting) algorithm, which enjoy a rich asymptotic theory that enables the construction of confidence and prediction intervals for them \cite{2f43af29-e27b-3b6f-9c3a-d7c3bd53adac}. Another widely used alternative for prediction intervals is conformal prediction, which provides finite-sample marginal coverage with minimal assumptions. However, conformal methods do not directly yield feature-wise confidence intervals for the GAM components, which we aim to explore in this work.
	
	Unfortunately, the asymptotic behavior of most gradient boosting algorithms (of which EBMs are one), is not known. \cite{ustimenko2023gradientboostingperformsgaussian} show that a randomized version of boosting converges to a Gaussian process. On the other hand, \cite{zhou2019boulevardregularizedstochasticgradient} show that taking an average instead of a sum of trees (in a process which they call Boulevard regularization) converges to a kernel ridge regression, yielding asymptotic normality. This was recently leveraged by \cite{fang2025gradient} to construct confidence intervals for the underlying nonparametric function $\f$ and prediction intervals for the labels, with extensions to random dropout and parallel boosting.

    \paragraph{Our contributions.}
     Unfortunately, the underlying nonparametric function itself can be highly complicated and uninterpretable in general. Restricting this to the class of GAMs, as with EBMs, yields interpretability. However, the existing Boulevard theory \cite{zhou2019boulevardregularizedstochasticgradient} does not directly apply to the class of algorithms we consider here. These results need to be extended to allow different structure matrices (resulting in different kernels) for each feature, along with enforcing the GAM identifiability conditions. While we sketch an algorithm that can be analyzed largely within the framework of \cite{fang2025gradient}, further algorithmic variants that either employ backfitting directly or employ parallelization methods result in different algorithmic limits requiring separate analysis. Unlike prior Boulevard analyses, EBMs induce \emph{feature-specific} kernels. This makes the theory substantially more delicate: one must control their interaction under centering and GAM identifiability, and establish asymptotic normality for each additive component, not just the overall predictor. To handle these additional challenges, we make the following contributions:
     
	\vspace{-2ex}
	\begin{itemize}
		\item \textbf{Asymptotic Theory.} We establish central limit theorems for three variations of EBMs (parallelized, backfitting-like parallelized, and sequential), showing that the Boulevard-regularized ensemble converges to a feature-wise kernel ridge regression. The resulting predictions are asymptotically normal and achieve the minimax-optimal MSE of $O(pn^{-2/3})$ for fitting Lipschitz $p$-dimensional GAMs \cite{yuan2015minimaxoptimalratesestimation}. This avoids the curse of dimensionality, significantly improving upon the rates of \cite{zhou2019boulevardregularizedstochasticgradient, fang2025gradient}.
		
		\item \textbf{Frequentist Inference.} By leveraging our proposed CLT for baseline vector $\bbeta$ and feature-wise predictions $\widehat \f^{(k)}$, we extend the heuristic bootstrap to a set of inference toolkits. We provide confidence intervals for each underlying function $\f^{(k)}$ (see Figure \ref{fig:ebm-effects-wide} for an example), and prediction intervals for the response. This also allows a direct test of variable importance following the methods of \cite{mentch2015quantifyinguncertaintyrandomforests}. We round out those propositions by numerical experiments to back up the validity of our algorithms.
		
		\item \textbf{Computational Efficiency.} The derived procedures above rely on a kernel ridge regression, with a prohibitively expensive $O(n^3)$ runtime. \cite{fang2025gradient} employ the Nyström approximation to produce intervals in $O(np^2)$ kernel precomputation and $O(p^2)$ query time, but the recursive implementation of \cite{musco2017recursivesamplingnystrommethod} can be slow in practice. Given that EBMs utilize histogram trees, we instead work out the computations in histogram bin-space.
        This bin-based implementation runs in an additional lower-order $O(pm^2)$ time per training round, $O(pm^3)$ kernel precomputation time, and $O(pm^2)$ interval query time, where $p$ is the number of features and $m$ the maximum number of bins per-feature. As the number of bins is usually capped at $255$ or $511$, this is independent of the number of samples $n$!
	\end{itemize}
	\begin{figure*}[ht!]
    \vspace{-3ex}
    \centering
     \includegraphics[width=0.8\linewidth]{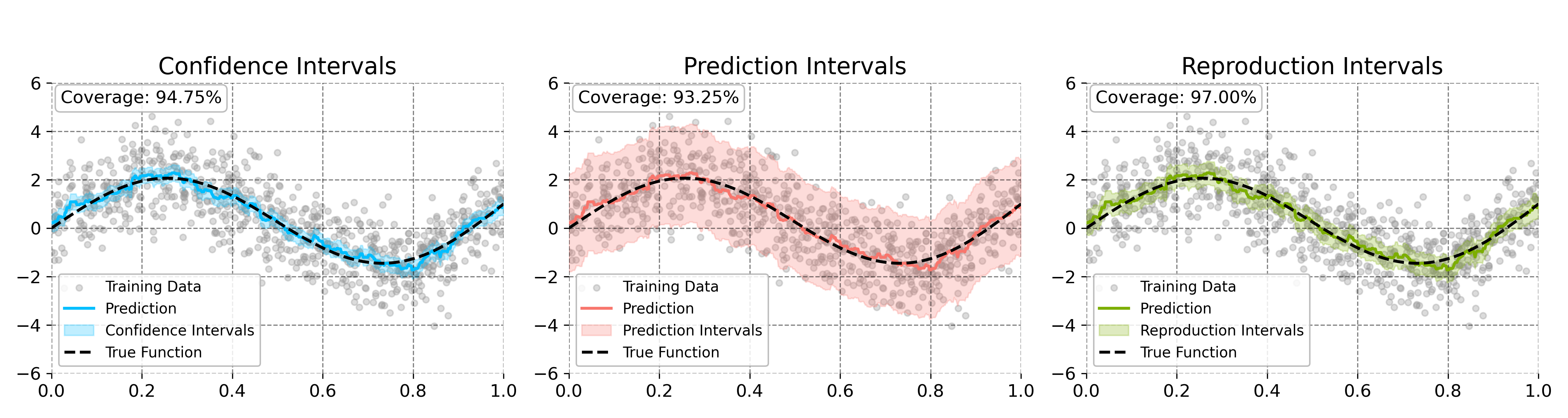}
        \vspace{-2ex}
	    \caption{Example of confidence and prediction intervals on a simple 1D function $y=2\sin(2\pi x) + x^2$.}
	    \label{fig:coverage-1d} \vspace{-2ex}
	\end{figure*}
	
	
	\section{Setup and Notation}
	\label{sec:setup}
	
	
	Let $(\bX_n, \by_n) = (\bx_i,y_i)_{i=1}^n$ are i.i.d.\ samples with $\bx_i=(\bx_i^{(1)},\ldots,\bx_i^{(p)})^\top$. Let $\jmat=\bI_n-\frac{1}{n}\bone\bone^\top$ denote the centering projector. We adopt the standard GAM identifiability convention that each component is centered on the training set: $\sum_{i=1}^n \f^{(k)}(\bx_i^{(k)})=0$ for all $k$, so the intercept equals the sample mean. Then:
	\begin{assumption}[Additive Lipschitz Nonparametric Regression]
		\label{aspt:truth}
		Assume that the response satisfies
		$$
		y \;=\; \bbeta \;+\; \sum_{k=1}^p \mathbf{f}^{(k)}(\mathbf{x}^{(k)}) \;+\; \epsilon : = \f(\bx) + \epsilon,
		$$
		where $y, \bbeta \in\mathbb{R}$, $\mathbf{x}=(\bx^{(1)},\dots,\bx^{(p)})^\top\in\mathbb{R}^p$, and $\varepsilon$ are i.i.d.\ sub–Gaussian errors with mean zero and variance $\sigma^2<\infty$.  
		Assume the covariates have a bounded density on their support: $0<c_1\leq \mu(\mathbf{x}) \leq c_2<\infty$.  Each component function $\f^{(k)}:\mathbb{R}\to\mathbb{R}$ is mean–zero with respect to the marginal of $\bx^{(k)}$, and is $L$-Lipschitz. The intercept $\bbeta$ is defined as the unconditional mean of the response: $\bbeta := \E_{\bx, \epsilon}[y] = \int y \, d\mu_{\bx}(x)\, d\mu_\epsilon(\epsilon).$
	\end{assumption}
	
	\paragraph{Boulevard Regularization.}\cite{zhou2019boulevardregularizedstochasticgradient} replace the usual additive update with a moving average. Instead of adding the next tree: $\widehat{\f}_b \leftarrow \f_{b-1}+\lambda\bt_b$, they average over all previous trees: $\widehat{\f}_b \leftarrow \frac{b-1}{b}\f_{b-1}+\frac{\lambda}{b}\bt_b$. This idea of shrinking the previous ensemble as also been adopted by \cite{ustimenko2022sglbstochasticgradientlangevin, ustimenko2023gradientboostingperformsgaussian}. This ensures convergence to a kernel ridge regression in the limit of infinite boosting rounds, and that the resulting predictions are asymptotically normal.\footnote{A particular advantage of Boulevard regularization is that it is robust to overfitting \cite{fang2025gradient}. Figure \ref{fig:overfit_mse} shows this behavior for Algorithm \ref{alg:brebm_ident}. EBM achieves this here via tiny learning rates when cycling through features in sequence. We do this while fitting to features in parallel.}

	\section{Algorithms} \label{sec:algorithms}
    
	\begin{algorithm}[ht!]
		\caption{Inferable EBM}
		\label{alg:brebm_ident}
		\begin{algorithmic}[1]
			\State \textbf{Input:} $\mathbf{X}\in\mathbb{R}^{n\times p}$, $\mathbf{y}\in\mathbb{R}^n$, subsample rate $\xi$, learning rate $\lambda$, rounds $B$, truncation level $M>0$.
			\State \textbf{Init:} $\widehat{\bbeta}_0 \gets \frac{1}{n}\sum_{i=1}^n \by_i$;\quad $\f_k^{(0)} \gets 0$ for all $k$.
			\For{$b=1$ \textbf{ to } $B$}
			\For{$k=1$ \textbf{ to } $p$ \textbf{ in parallel}}
			\State Sample $G_{b,k}\subset\{1,\dots,n\}$ i.i.d. w.p. $\xi$.
			\State $r_{i,k} \gets \by_i - (\widehat{\bbeta}_{b-1} + \sum_{\ell} \f_{b-1}^{(\ell)}(\bx_i^{(\ell)}))$.
			\State Fit tree $\bt^{(b,k)}$ to $\{(\bx_i^{(k)}, r_{i,k})\}_{i\in G_{b,k}}$.
			\State $\mu_{b,k} \gets \frac{1}{n}\sum_{i=1}^n \bt^{(b,k)}(\bx_i^{(k)})$.
			\State $\widetilde \bt_b^{(k)}(\bx) \gets t_b^{(k)}(\bx) - \mu_{b,k}$.
			\State $\widetilde \bt_b^{(k)}(x) \gets \max\{-M,\;\min\{\widetilde \bt_b^{(k)}(\bx),M\}\}$.
			\State $\widehat{\bbeta}_b \gets \widehat{\bbeta}_{b-1} + \mu_{b,k}$.
			\State $\f_b^{(k)} \gets \frac{b-1}{b} \f_{b-1}^{(k)} + \frac{\lambda}{b}\,\widetilde \bt_b^{(k)}$.
			\EndFor
			\EndFor
			\State \textbf{Output:} $\widehat \f(x) \gets \widehat{\bbeta}_B + \frac{1+\lambda}{\lambda}\sum_{k=1}^p \f_k^{(B)}(\bx^{(k)})$.
		\end{algorithmic}
	\end{algorithm}

    We adapt the above regularized boosting procedure to fitting EBMs below in Algorithm \ref{alg:brebm_ident}, in the hope that doing so allows us to show similar theoretical guarantees for the resulting procedure. 
    
    \paragraph{Parallelization and  inference.} Algorithm \ref{alg:brebm_ident} is a minimal parallelized variation on EBMs. At each round $b$, the new set of $p$ univariate trees $(\bt_b^{(k)})_{k=1}^p$ are fitted in parallel by regressing the current residual $r_{i,k}$ on each of the $p$ features $\bx_i^{(k)}$. The trees are then centered and clipped, before being introduced into the ensemble as part of the average: $\mathbf{f}_b^{(k)} \gets \frac{b-1}{b} \mathbf{f}_{b-1}^{(k)}+\frac{1}{b} \widetilde{\mathbf{t}}_b^{(k)}$.
    
    Averaging yields a kernel ridge regression limit and a CLT, but (as in Boulevard) the fitted signal is asymptotically shrunk, requiring rescaling by $\frac{1+\lambda}{\lambda}$. In Section \ref{sec:finite-sample-convergence}, we show that the prediction made by Algorithm \ref{alg:brebm_ident} converge to a kernel-ridge-regression. Informally:
	$$
	\widehat{\f}_b(\bx) \underset{b\to\infty}{\overset{a.s.}{\longrightarrow}} \bar{\by}+ \sum_{k=1}^{p}\E[\sveck(\bx)]^\top \jmat[\bI + \jmat\bK]^{-1}\by.
	$$
	The down-scaling stems from the matrix $[\bI + \jmat\bK]^{-1}$: When the aggregated kernel $\bK = \sum_{k=1}^{p} \Kmatk$ becomes nearly the identity (i.e. the aggregated decision boundaries are fine enough), we see a $\frac{\lambda}{1+\lambda}$ shrinkage.

    \paragraph{Difficulty in parallelization.} 
    The parallelization in Algorithm \ref{alg:brebm_ident} does not come for free. One has to either use a small learning rate of $1/p$, or make a rather strong assumption on the feature kernels (see Assumption \ref{aspt:proj-orth}) for Algorithm \ref{alg:brebm_ident} to provably converge.
    
    Though in practice this is not an issue, we explore alternatives below for the sake of completeness. One is to abandon parallelism, and randomly choose a single feature to fit a tree to during each boosting round. This is presented in Algorithm \ref{alg:brebm_random_cyclic} (see Appendix \ref{app:algorithms}), which converges to the same fixed point as Algorithm \ref{alg:brebm_ident}, but without the computational gains. 
	
	\paragraph{Leave-one-out procedure.} One problem remains. Algorithms \ref{alg:brebm_ident} and \ref{alg:brebm_random_cyclic} converges to $\frac{\lambda}{1+\lambda} \f \leq \f/2$, and so require rescaling by $\frac{1+\lambda}{\lambda}$. 
    We provide an alternative inspired by the backfitting-like procedure of \cite{fang2025gradient} in Algorithm \ref{alg:ebm_bratp} that does not require rescaling. 
    
    This hybrid backfitting-boosting procedure amounts to not using predictions from trees trained on feature $\bx^{(k)}$ when fitting a new tree $\bt_{b}^{(k)}$ for it\footnote{This can be interpreted as treating other functions $\f^{(\ell)} \neq \f^{(k)}$ as nuisances and learning them via boosting, while learning $\f^{(k)}$ as a random forest.} -- instead using  residuals $\by_i - (\widehat{\bbeta}_b+\sum_{\ell\neq k}\f_{b-1}^{(\ell)}(\bx_i^{(\ell)}))$.
    Although this does not require rescaling, it produces complicated confidence intervals that are harder to compute efficiently. Informally, Algorithm \ref{alg:ebm_bratp} converges to:
	$$
	\bar{\by}+ \sum_{k=1}^{p}\E[\sveck(\bx)]^\top (\bI - \E\Smatk)^{-1} \jmat\Big[\bI-(\bI+\mK)^{-1}\mK\Big]\by,
	$$
    where $\mathcal{K}:=\sum_{k=1}^p(\mathbf{I}-\mathbb{E} \mathbf{S}^{(k)})^{-1} \mathbf{J}_n \mathbb{E} \mathbf{S}^{(k)}$ is a somewhat more complicated combined kernel that unlike Algorithms \ref{alg:brebm_ident} and \ref{alg:brebm_random_cyclic}, does not decompose into the sum of the individual feature kernels. 
	\paragraph{Comparison to vanilla EBMs.} Our algorithms possess advantages over vanilla EBMs. The convergence to a kernel ridge regression allows us to show a central limit theorem for each function $\f^{(k)}$:
    $$\lVert\rveck(\bx)\rVert_2^{-1}\left(c_E\cdot\widehat{\f}^{(k)}(\bx) - \f^{(k)}(\bx)\right)\stackrel{d}{\longrightarrow}\mathcal{N}(0, \sigma^2),$$
    where $c_E$ is a scaling factor that is $\frac{\lambda}{1+\lambda}$ for Algorithms \ref{alg:brebm_ident} and \ref{alg:brebm_random_cyclic}, and $1$ for Algorithm \ref{alg:ebm_bratp}. $\rveck$ are the kernel weights for the kernel ridge regression predictions ${\rveck}^\top\by$. For instance, Algorithms \ref{alg:brebm_ident} and \ref{alg:brebm_random_cyclic} converge to ${\rveck}^\top\by := \mathbb{E}\left[\mathbf{s}^{(k)}(\mathbf{x})\right]^{\top} \mathbf{J}_n\left[\mathbf{I}+\mathbf{J}_n \mathbf{K}\right]^{-1} \by.$
    
    This asymptotic normality yields confidence intervals for each $\f^{(k)}$ and prediction intervals for the response $\by$ -- as long as we can compute $\lVert\rveck(\bx)\rVert_2$ efficiently. It also allows us to prove guarantees on the generalization error of each algorithm. We later show that $\lVert\rveck(\bx)\rVert_2 = \Theta(n^{-1/3})$, allowing us to show that our algorithms achieve the minimax-optimal MSE of $O(p\,n^{-2/3})$ for fitting Lipschitz GAMs. In comparison, vanilla EBMs hold \emph{no such guarantees}.
    

    \paragraph{Comparison to Boulevard.} The original gradient boosting algorithm with Boulevard regularization of \cite{zhou2019boulevardregularizedstochasticgradient} converges to a kernel ridge regression of the form $\E[\bs(\bx)]^\top(\lambda^{-1}\bI + \E[\bS])^{-1}\E[\bS]\by$. This is similar to our result for Algorithms \ref{alg:brebm_ident} and \ref{alg:brebm_random_cyclic}, but we achieve convergence for each centered univariate function while they achieve convergence only for the overall function $\f$. As a consequence, their method has an MSE of $O(n^{-\frac{1}{d+1}})$ that grows exponentially in dimension. This is not even minimax-optimal for the nonparametric Lipschitz regression setting they are in.
    
    The variants of \cite{fang2025gradient} incorporating dropout and parallel training converge to slightly different fixed points, but ultimately achieve the same asymptotic MSE rates exponential in dimension.
    
    
	
    \section{Theory}
    \label{sec:theory}
    We now work towards proving the central limit theorem outlined above. To begin, we introduce a few definitions. First, we use regression trees as base learners. These can be understood as a linear transformation from the observed signal to the predictions via a smoothing tree kernel, formally defined below:
	\begin{definition}[Regression trees]
		\label{def:regression-trees}
		A regression tree $\bt_n$ trained on $n$ datapoints segments the covariate space $[0,1]^d$ into a partition of hyper-rectangles $\{A_i\}_{i=1}^m$. When $A(\bx)$ is the rectangle containing $\bx$, and $s_{n,j}(\bx)$ is the frequency at which training point $\bx_j$ and test point $\bx$ share a leaf, ${s}_{n,j}(\bx) = \frac{\mathbbm{1}(\bx_j \in A(\bx))}{\sum_{k=1}^n \mathbbm{1}(\bx_k \in A(\bx))},$ the regression tree $\bt_n$ predicts $\bt_n(\bx) = \sum_{j=1}^n s_{n,j}(\bx)\cdot y_j$.
	\end{definition}
	This suggests that ${s}_{n,j}(\bx)$, the frequency at which training point $\bx_j$ and test point $\bx$ share a leaf, defines a similarity measure. This observation is insightful. In fact, the frequency at which training points $\bx_i, \bx_j$ share a leaf, $s_{n,j}(\bx_i)$, yields a kernel!\footnote{Theorem 3 of \cite{zhou2019boulevardregularizedstochasticgradient} shows that $\E\bS_n$ is p.s.d. with spectral norm no more than 1.}
	\begin{definition}[Structure vectors and matrices]
		\label{def:structure-vec-mat}
		Let $\bt_n$ be a tree. $\bs_n(\bx) = (s_{n,1}(\bx),...,s_{n,n}(\bx))^\top$ be the structure vector of $\bx$. The matrix $\bS_n = \left(s_{n,j}(\bx_i)\right)_{i,j=1}^n$ with rows $\bs_n(\bx_i)^\top$ is the structure matrix of $\bt_n$.
	\end{definition}
	Following \cite{zhou2019boulevardregularizedstochasticgradient, fang2025gradient}, we require that the tree \emph{structures} be separated from the leaf  \emph{values}, and stabilize after a burn-in period:
	\begin{assumption}[Integrity]
		\label{aspt:integrity}
		Assume
        \vspace{-2ex}
		\begin{enumerate}
			\item (Structure-Value Isolation) The density for $\bs_b(\bx)$ is independent from the terminal leaf value $\by\in\mathbb{R}$.
			\item (Non-adaptivity) For each feature $k$, there exists a set of probability measures $\mathcal{Q}^{(k)}_n$ for tree structures such that all tree structures $\Smatk_b$ are independent draws from $\mathcal{Q}^{(k)}_n$ after some finite $b'$.
		\end{enumerate}
	\end{assumption}
	
	The former can be satisfied by refitting leaf values on a hold-out dataset \cite{wager2017estimationinferenceheterogeneoustreatment}. Non-adaptivity may occur naturally, but can also be enforced by sampling from previous tree structures after a burn-in period. In practice, our algorithms tend to perform well even in the absence of these assumptions.
	
	\paragraph{Tree Geometry}
	We control the geometry and size of leaves to ensure locality and stable averages. These assumptions control the MSE of the final estimate -- \cite{zhou2019boulevardregularizedstochasticgradient} and \cite{fang2025gradient} make spiritually similar assumptions, but with  different scalings in dimension that do not avoid the curse of dimensionality. Again, we avoid this problem due to the GAM structure of our problem setup:
	\begin{assumption}[Bounded Leaf Diameter]
		\label{aspt:leaf-diameter}
		Write $\mathsf{diam}(A)=\sup_{\bx_1,\bx_2\in A}\norm{\bx_1-\bx_2}$. For any leaf $A$ in a tree with structure $\Pi\in Q_n^{(k)}, \forall k = 1, \cdots, p$, we need $\sup_{A \in \Pi} \mathsf{diam}(A) =  O(d_n) = O(n^{-{1}/{3}})=o(\frac{1}{\log n})$.
	\end{assumption}
    Assumption \ref{aspt:ub-minimal-leaf-size} forces trees to perform only \emph{local} averaging by keeping leaves well-balanced. Assumption \ref{aspt:truth} allows depths as large as $\Omega(\log n)$, which is natural for balanced binary trees.
	
	\vspace{-1ex}
	\begin{assumption}[Increased Minimal Leaf Size]
		\label{aspt:ub-minimal-leaf-size}
		For any $\nu >0$, the leaf geometric volume $v_n$ is bounded such that \(v_n =n^{-\frac{2}{3}+\nu}<n^{-\frac{1}{3}}=O(d_n)\).
	\end{assumption}

    This is natural: the leaf count cannot exceed the effective histogram resolution. Enforcing at least $n^{2/3-\nu}$ points per leaf rules out overly fine, high-variance partitions.

	\begin{assumption}[Restricted Tree Support]
		\label{aspt:restricted-tree-space}
		For a tree constructed using feature $k$, the cardinality of the tree space $Q_n^{(k)}$ is bounded by $O\Big(n^{-1/3}\exp(n^{2/3-\epsilon_n})\Big)$, for some small $\epsilon_n \to 0$.
	\end{assumption}

    This is a technically convenient assumption for us to generalize it in random design case. It bounds the effective complexity of the tree class so we can use uniform LLNs and contraction arguments. The assumption still permits an super-polynomial large class in $n$, which is far looser than practice: vanilla EBMs use very shallow trees (typically depth 1--2) on a fixed histogram grid with at most a few hundred bins per feature, so the effective class is much smaller (roughly polynomial in $n$).
    
	Unlike \cite{zhou2019boulevardregularizedstochasticgradient} and \cite{fang2025gradient}, the assumptions above are dimension-free. They enforce asymptotic locality of the partition, allowing for control over bias and variance. Assumption \ref{aspt:ub-minimal-leaf-size} ensures stability of the structure matrix spectrum, allowing us to recover the minimax-optimal rate later.

    \paragraph{Warmup: Theory for Algorithm 
    \ref{alg:brebm_random_cyclic}.}

    As mentioned earlier, the theory for Algorithm \ref{alg:brebm_random_cyclic} is  the easiest to prove. With the following observation, this amounts to a straightforward extension of \cite{zhou2019boulevardregularizedstochasticgradient}. The learner updates a feature's kernel at random each round, so the tree kernel chosen by Algorithm \ref{alg:brebm_random_cyclic} amounts to $\E\bS = \frac{1}{p}\sum_{k=1}^p\E[\Smatk]$. 
    
    This observation allows us to apply the theory of \cite{zhou2019boulevardregularizedstochasticgradient} almost verbatim. As such, Algorithm \ref{alg:brebm_random_cyclic} gives a spiritually similar fixed point iteration $\tilde{\by}^*:=\sum_{k}\hat{\by}_k^*=\lambda\jmat\E\bS(\by - \tilde{\by}^*)$\footnote{We drop intercepts as they are orthogonal to centering.}, modulo centering. This leads to a ensemble-wise fixed point of $\tilde{\by}^* = (\bI + \jmat\E\bS)^{-1}\jmat\E\bS\by$ and the subsequent feature-wise fixed points $\hat{\by}_k^* = \lambda\jmat\E\bS(\bI - (\bI + \jmat\E\bS)^{-1}\jmat\E\bS)\by=\lambda\jmat\E\Smatk(\bI + \jmat\E\bS)^{-1}\by$. Finite sample convergence and asymptotic normality holds following the same recipe discussed in \cite{zhou2019boulevardregularizedstochasticgradient}.

    However, this sequential random update does not parallelize over features. 
    To address these concerns we return to  Algorithm \ref{alg:brebm_ident} and analyze its convergence, which we will discuss in Section \ref{subsec:alg.1}.

    \paragraph{Feature-specific (specialized) kernels.}
	Still, the above intuition is instructive. Unlike \citet{zhou2019boulevardregularizedstochasticgradient} and \citet{fang2025gradient}, which consider a single expected kernel $\E\bS$, EBMs induce \emph{feature-specific} kernels $\Kmatk:=\E[\Smatk]$.  This is a double-edged sword. On one hand, one has to control the way these kernels interact in our algorithms. On the other, it enables a clearer characterization of partial dependence, while allowing us to demonstrate feature-wise convergence, which in turn yields feature-wise confidence intervals. 

    \paragraph{Theory for Algorithm \ref{alg:brebm_ident}.}
    \label{subsec:alg.1}
    
    As mentioned earlier in Section \ref{sec:algorithms}, introducing parallelization introduces theoretical challenges. One way to prove that Algorithm \ref{alg:brebm_ident} converges is to use a learning rate of $\lambda =\frac{1}{p}$. However, this yields no speedup from Algorithm \ref{alg:brebm_random_cyclic} on average, as sampling with probability $1/p$ and averaging across $p$ updates are the same in expectation. 

    Another way, that allows us to use a learning rate of $\lambda \approx 1$, is to make the following assumption:
    
	\begin{assumption}[Projector-like feature kernels and conditional orthogonality]
		\label{aspt:proj-orth}
		Let $\Kmatk=\E[\Smatk]$ and $\jmat:=\bI_n-\frac{1}{n}\bone\bone^\top$. 
		For each feature $k$, let
		$$
		\mathcal H_k \;:=\; \big\{\,\bv\in\R^n:\ \exists\,g:\R\!\to\!\R\text, \bv_i=g(\bx_i^{(k)}),\ \bone^\top \bv=0\,\big\}
		$$
		with $P_k$ the orthogonal projector onto $\mathcal{H}_k$. Assume:
		\begin{itemize}
			\item \textbf{Projector-like approximation.} There exists $\varepsilon_{n,k}\to 0$ with
			$\big\|\, \jmat\,\Kmatk\,\jmat - P_k \,\big\|_{\mathrm{op}} \;\le\; \varepsilon_{n,k}\,.
			$
			\item \textbf{Conditional orthogonality.} For $a\neq k$, the additive subspaces are orthogonal:
			$
			P_a\,P_k \;=\; \mathbf{0}\,.
			$
		\end{itemize}
	\end{assumption}
	Assumption~\ref{aspt:proj-orth} is technically convenient but admittedly strong, especially the conditional orthogonality requirement. We therefore emphasize that it is used only to analyze Algorithm~\ref{alg:brebm_ident}. As a fallback, Algorithm~\ref{alg:ebm_bratp} removes this requirement and needs only Assumption~\ref{aspt:kernel-spec}. Intuitively, the projector-like condition says the centered expected tree kernel for feature $k$ behaves like the orthogonal projector onto mean-zero functions of $\bx^{(k)}$, while conditional orthogonality asserts these per-feature subspaces have negligible overlap across $k$ (similar in spirit to “no concurvity” conditions for uniqueness in additive models, e.g., \citealp{hastie1986gams}); this can hold exactly under feature independence and can be a reasonable approximation under weak dependence when partitions are shallow and not aligned across features.

    Now we discuss the finite-sample convergence for Algorithm \ref{alg:brebm_ident} as promised. A way forward is to first establish almost surely convergence as we keep boosting with $b\to \infty$, then seek asymptotic normality when sample size grows. The first part is usually handled by showing that the difference to a set of postulated fixed points is a stochastic contraction mapping, decaying to 0. However, by the design of Algorithm \ref{alg:brebm_ident}, stepwise mean convergence can be fragile when $\lambda \approx 1$, hence we enforce Assumption \ref{aspt:proj-orth} to ensure stability. Formally, the convergence behavior can be characterized in Theorem \ref{thm:finite-sample-convergence-1}, stated below.
    
    \begin{theorem}
		\label{thm:finite-sample-convergence-1}
		Define $\widehat{\mathbf{y}}_{b}:=\widehat{\bbeta}_{b}+\sum_{a=1}^{p}\widehat{\mathbf{y}}^{(a)}_{b},
		\jmat=\bI-\tfrac{1}{n}\bone\bone^\top.$
		The fixed points are
		\[
		\widetilde{\mathbf{y}}_k^* \;=\; \jmat\,\Kmatk\,[\lambda^{-1}\bI+\jmat\,\bK]^{-1}\mathbf{y},
		\qquad
        \bK\;:=\;\sum_{k=1}^{p}\Kmatk,
        \]
		and $\widehat{\mathbf{y}}^*\;=\;\widehat{\bbeta}^*\bone+\sum_{k=1}^{p}\widetilde{\mathbf{y}}_k^*$,
		with $\widehat{\bbeta}^*=\bar{\mathbf{y}}$. For each $k$,
        $$
        \widehat{\mathbf{y}}^{(k)}_{b}\;\stackrel{a.s.}{\longrightarrow}\; \widetilde{\mathbf{y}}^{(k)\,*},
		\qquad
		\widehat{\bbeta}_{b}\;\stackrel{a.s.}{\longrightarrow}\;\bar{\mathbf{y}}=\tfrac{1}{n}\bone^\top\mathbf{y},
		$$
		and hence $\widehat{\mathbf{y}}_{b}\stackrel{a.s.}{\to} \widehat{\mathbf{y}}^*$ almost surely.
	\end{theorem}

    This fixed point carries two different relevant and practical perspectives. When setting $\lambda = \frac{1}{p}$, one essentially implements a coordinate descent. It also recovers the vanilla boulevard randomized tree kernel by assigning uniform mass on featurewise kernels. Thus convergence can be shown without additional assumptions. Taking $\lambda=1$, however, requires a more delicate characterization of tree behavior specified in Assumption \ref{aspt:proj-orth}. This condition implies that featurewise signals do not overlap with each other asymptotically, removing the need for averaging the final feature functions.\footnote{This can be thought of as parallel coordinate descent. Averaging is conservative, so when can we not do so?}

    \paragraph{Theory for Algorithm \ref{alg:ebm_bratp}.}
    As discussed before, one may argue that Assumption \ref{aspt:proj-orth} is unrealistic. If one is not willing to use $\lambda=1/p$ then, Algorithm \ref{alg:ebm_bratp} poses a remedy. This does not rely on Assumption \ref{aspt:proj-orth}, and only depends on a weaker assumption below.
    \begin{assumption}[Feature-kernel distinctness]
		\label{aspt:kernel-spec}
		Let $\Kmatk=\E[\Smatk]$. There exists $c_3<1$ such that, for all $a\ne k$,
			$
			\big\|\, \E[\mathbf{S}^{(a)}]\,\E[\Smatk] \,\big\|_{\mathrm{op}} \;\le\; c_3,
			\quad
			(p-1)\,c_3 \;<\; 1.
			$
	\end{assumption}
	
    Assumption~\ref{aspt:kernel-spec} is a weaker version of Assumption~\ref{aspt:proj-orth}: instead of requiring (approximate) orthogonality of feature-wise subspaces, it only rules out the degenerate case where different features induce nearly identical smoothing operators in expectation. We verified this assumption empirically with plausible results. See Table \ref{tab:kernel_distinctness}. This milder assumption is sufficient for establishing a similar finite-sample convergence result below.

    \begin{theorem}
		\label{thm:finite-sample-convergence-2}
		Denote the prediction using feature $k$ at boosting round $b$ as $\widehat{\mathbf{y}}^{(k)}_b = \frac{1}{b}\sum_{s=1}^bt^{(s, k)}(\mathbf{x}^{(k)})$.
		Define the aggregated kernel $\mK:=\sum_{k=1}^{p}(\bI-\E\Smatk)^{-1}\jmat\E\Smatk$. Define feature-wise fixed points
		\[
		\widetilde{\mathbf{y}}_k^* = (\bI-\E\Smatk)^{-1} \jmat\E\Smatk\left[\bI - \left(I+\mK\right)^{-1}\mK\right]\mathbf{y}.
		\]
		Then, $\widehat{\mathbf{y}}^{(k)}_b \overset{a.s.}{\to}\widetilde{\mathbf{y}}_k^*$, $\widehat{\mathbf{y}}_b\overset{a.s.}{\to}\widehat{\mathbf{y}}^*$, and $\widehat{\bbeta}^* \overset{a.s.}{\to}\bar{\mathbf{y}}:=\frac{1}{n}\mathbf{1}^\top\mathbf{y}.$
		
	\end{theorem}

    The operator $(\bI-\Smatk)^{-1} (\bI-\frac{1}{n}\mathbf{1}\mathbf{1}^\top)\Smatk$, which we can interpret as a feature-wise KRR, is stacked with the fixed point of the residual created on the whole ensemble. In this product, $(\bI-\Smatk)^{-1} (\bI-\frac{1}{n}\mathbf{1}\mathbf{1}^\top)\Smatk$ measures the marginal contribution of a particular feature $k$ contribute to the ensemble prediction. 

    \paragraph{A Central Limit Theorem.} We are now in a position to provide central limit theorems for all three algorithms. Unlike \cite{zhou2019boulevardregularizedstochasticgradient} and \cite{fang2025gradient}, the GAM structure of our problem allows us to relax the rates within the assumptions on tree geometries. While the following lemma allows us to check the Lindeberg-Feller conditions, it also allows us to recover the optimal minimax rate $n^{-2/3}$ for univariate non-parametric Lipschitz regression \cite{yuan2015minimaxoptimalratesestimation}. Below, we use the subscript $E=\{A,B\}$ to label relevant variables in Algorithm \ref{alg:brebm_ident} and \ref{alg:ebm_bratp}. 

    \begin{lemma}[Rate of Convergence]
		\label{lem:rateofr}
		Let $\mathbf{B}^{(k)}_n := \{i  : \norm{\bx^{(k)}-\bx_i^{(k)}}\leq d_n\}$ be the points within distance $d_n$ from test point $\bx$ on feature $k$. Under Assumption \ref{aspt:restricted-tree-space}, if $\left| \mathbf{B}_n \right|  = \Omega( n \, d_n)$ and $\inf_{A\in \Pi\in Q_n^{(k)}}\sum_{i}\bI(\bx_i\in A) = \Omega(n^{\frac{2}{3}})$, then
		$
		\norm{\kveck}_2,
		\norm{\rveck_E}_2 = \Theta(n^{-\frac{1}{3}})
		$ almost surely.
	\end{lemma}

	The rate of the boulevard vector enables a conditional CLT, informally $\frac{\rveck_E(\bx)^\top\f(\bX)-\f^{(k)}(\bx)}{\norm{\rveck_E(\bx)}}\overset{d}{\to}\mathcal{N}(0, \sigma^2)$. Working with either Assumption \ref{aspt:proj-orth} or Assumption \ref{aspt:kernel-spec}, we analyze the form of the combined kernel, and conclude with the following central limit theorem.

    \begin{theorem}[Asymptotic Normality]
		\label{thm:unconditional-clt}
		Define $c_A = \frac{\lambda}{1+\lambda}, c_B=1 $. For all $k = 1, \cdots, p$, we have that:
		$$\norm{\rveck_E(\bx)}^{-1}\left(\widehat{\f}_E^{(k)}(\bx)-c_E\f^{(k)}(\bx)\right) \overset{d}{\rightarrow} \mathcal{N}(0, \sigma^2)
		$$
	\end{theorem}

    Normality of the intercept term $\hat{\bbeta}$ is also established and discussed in Lemma \ref{thm:baseline-clt}. Similar to \cite{fang2025gradient}, our CLTs yields a subsequent risk bound:
    \begin{corollary}
    \label{cor:risk-bound}
        $\E[(\frac{1}{c_E}\hat{\f}_E^{(k)}(\bx)-\f^{(k)}(\bx))^2] \lesssim \frac{\sigma^2}{c_E^2}p\,n^{-\frac{2}{3}}$.
    \end{corollary}
    This result shows that our algorithms achieve the minimax-optimal MSE of $O(pn^{-2/3})$ for fitting Lipschitz $p$-dimensional GAMs \cite{yuan2015minimaxoptimalratesestimation}. This avoids the curse of dimensionality, significantly improving upon the rates of \cite{zhou2019boulevardregularizedstochasticgradient, fang2025gradient}.

    \begin{figure*}[ht!]
     \vspace{-2ex}\centering\includegraphics[width=0.7\linewidth]{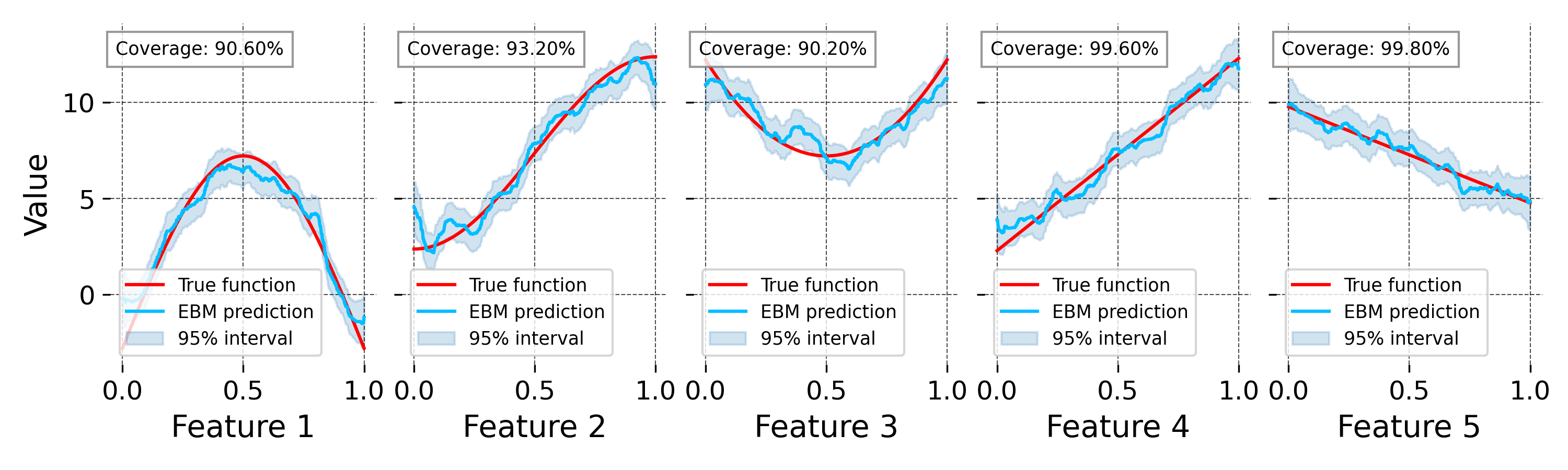}\vspace{-2ex}
        \caption{Per-feature CIs and coverages for $f(x)=-5+10\sin(\pi x^{(1)}) + 5\cos(\pi x^{(2)}) + 20(x^{(2)}-0.5)^2  +10x^{(3)} - 5x^{(4)}$.}
        \label{fig:feature-coverage}
        \vspace{-2ex}
    \end{figure*}

    
	\section{Inference and Computation}
    We can now derive confidence and prediction intervals from the result on asymptotic normality above. Here, we construct practical intervals and tests, and shows how to compute the required standard errors in \emph{bin space} so that inference time is independent of $n$. Throughout the section, we estimate the noise variance with a consistent variance estimates $\hat{\sigma}^2$ via in-sample or out-of-bag estimation.

    \paragraph{Confidence Intervals.}
    We construct $(1-\alpha)$ confidence intervals for each of \{individual feature effects, the intercept, the values of a prediction at a new $\bx$\}.
    \vspace{-2ex}
    \begin{itemize}
        \item $c_E^{-1}\widehat{\f}^{(k)}_E(\bx)\ \pm\ z_{1-\alpha/2}\,c_E^{-1}\hat{\sigma}\,\lVert\rveck_E(\bx)\rVert$ yields a $(1-\alpha)$-confidence interval for the feature effect. By similar arguments in \cite{fang2025gradient}, the coverage rate will converge to $(1-\alpha)$ almost surely. This is validated by an example in Figure  \ref{fig:feature-coverage}.
        \item For the intercept term, by Lemma~\ref{thm:baseline-clt}, $\widehat{\bbeta}$ is asymptotically normal with variance $\sigma^2/n$, yielding the usual interval $\widehat{\bbeta}\pm z_{1-\alpha/2}\,\hat{\sigma}/\sqrt{n}$
        \item For the overall response, write $\widehat{f}_E(\bx)=\widehat{\bbeta}+c_E^{-1}\sum_{k=1}^p \widehat{\f}^{(k)}_E(\bx)$. A $(1-\alpha)$ confidence interval is
    $\widehat{f}_E(\bx)\ \pm\ z_{1-\alpha/2}c_E^{-1}\hat{\sigma}\sqrt{\frac{1}{n}+\|\sum_{k=1}^p \rveck_E(\bx)\|_2^2\ }.
    $
    \end{itemize}
    \vspace{-2ex}
    Recall that $\rveck_E(\bx)$ denotes the per-feature influence vector. Under Assumption~\ref{aspt:proj-orth}, $\lVert\sum_{k=1}^p \rveck_E(\bx)\rVert_2^2
    =\sum_{k=1}^p \lVert{\rveck_E(\bx)}\rVert_2^2,$ so the variance contribution is the sum of feature-wise terms. Without Assumption~\ref{aspt:proj-orth}, this replacement is conservative by Cauchy–Schwarz.

    \cite{fang2025gradient} examines two further variants that are immediately implementable:
    \vspace{-1ex}
    \begin{itemize}
    \item {\bf Reproduction Intervals.} All intervals above can be modified to be {\em reproduction} intervals for the predictions of a new model trained on an independent dataset by scaling their widths by $\sqrt{2}$. 


    
    \item {\bf Prediction Intervals.} We can construct  a prediction interval for a new $y$ given $\bx$ by computing ${c_E}^{-1}{\widehat{\f}^{(k)}_E(\bx)}\ \pm\ z_{1-\alpha/2}\,{c_E}^{-1}{\hat{\sigma}\sqrt{1+\lVert{\rveck_E(\bx)}\rVert_2^2}}$. 
\end{itemize}

    \paragraph{Binning. } The current implementation of EBMs \cite{nori2019interpretmlunifiedframeworkmachine} utilizes histogram trees. As we worked off this implementation, it was natural to construct intervals with the binned data, instead of working with the original data. A more detailed discussion of the discretization error induced by binning is deferred to Appendix~\ref{append:what_we_lose_by_binning}.
    
    As such, we compress all $n$ samples along feature $k$ into $m_k\!\ll\!n$ histogram bins and work entirely in \emph{bin space}. In Appendix~\ref{app:binning}, we construct a decomposition of the structure matrix that allows us to show that computing the norm of $\rveck(\bx)$ reduces to solving a $m_k\times m_k$ system of linear equations, and computing two $m_k$-length dot products. 
    
    After this compression, the dependence on $n$ disappears. Denote $m$ as the maximum number of bins among all $k$ binnings. Caching $\bM$ takes $O(pm^3)$, while each per-point inference query takes only $O(m^2)$ time. As bin numbers are often capped to 255 or 511 (independent of $n$), we can compute our intervals in time independent of the number of samples. This outperforms \cite{fang2025gradient}'s runtime of $O(nd_{\text{eff}}^2)$.

    \paragraph{Practical Guidance.} In practice, Algorithm \ref{alg:brebm_ident} relies on conditional orthogonality; we recommend monitoring simple convergence diagnostics, and if convergence appears unstable, falling back to Algorithm \ref{alg:ebm_bratp}. In our experiments, the number of boosting rounds $B$ in the range 100--300 worked well. A smaller learning rate $\lambda$ typically improves finite-sample stability but can yield slightly wider intervals. Subsample rates in $[0.5,0.8]$ were robust. For binning, we found that using either the usual EBM defaults (e.g., 64 or 255 bins) or scaling the bin count as $m\asymp n^{1/3}$ works well; the latter matches our assumptions.

    \addtocounter{figure}{1}
    \begin{figure*}[ht!]
        \centering\includegraphics[width=0.85\linewidth]{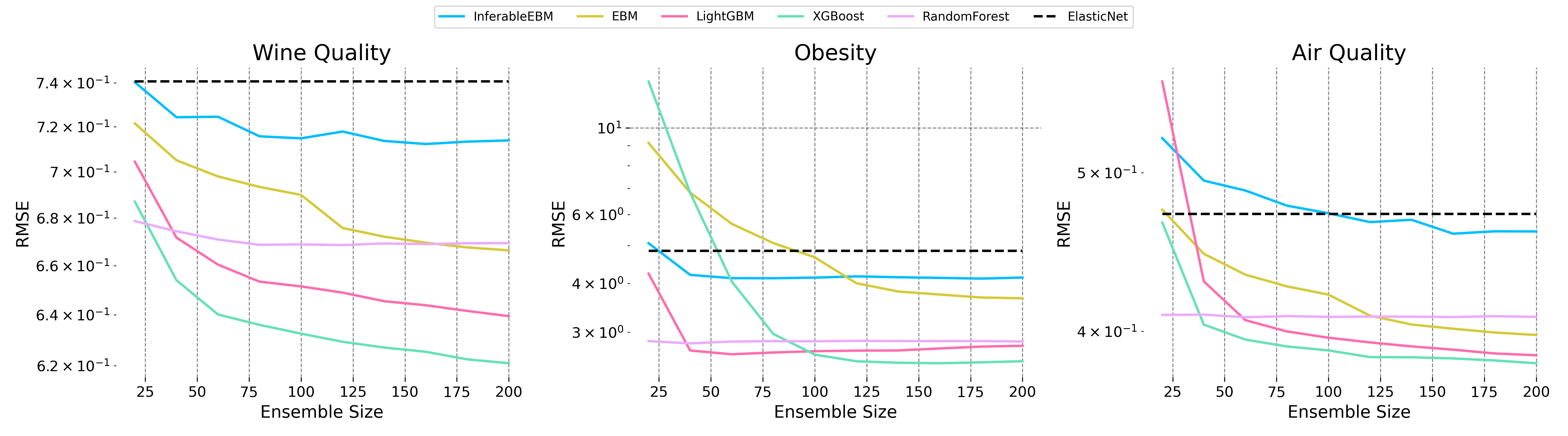}\vspace{-2ex}
        \caption{MSE horse races on UCI machine learning repository datasets. All hyperparameters tuned by Optuna.}
        \label{fig:mse-horse-races}
    \end{figure*}
    \addtocounter{figure}{-1}
    
	\section{Numerical Experiments}
    \label{sec:experiments}
    \paragraph{Predictive Accuracy} We compare Algorithm \ref{alg:brebm_ident} to other benchmark models on real world data from the UCI Machine Learning Repository in Figure \ref{fig:mse-horse-races}. Algorithm \ref{alg:brebm_ident} is competitive with EBMs across the Wine Quality, Obesity and Air Quality datasets, while converging much faster with parallelization. However, these datasets exhibit higher order interactions. Algorithm \ref{alg:brebm_ident} and EBMs ultimately underperform random forests and gradient boosting, though their added flexibility allows them to outperform a linear model. 
    
    However, the above setup had hyperparameters extensively tuned by Optuna over 100 trials. When this is not feasible, gradient boosting can be susceptible to overfitting. Figure \ref{fig:overfit_mse} depicts two examples where this is the case, on $f(x) = \sin(4\pi x)$ and the diabetes dataset from \cite{Efron_2004}. Here, the performances are flipped. Algorithm \ref{alg:brebm_ident} and EBMs perform the best with no need for early stopping, while the gradient boosting methods quickly overfit. EBMs achieve this via a tiny learning rate during cyclic boosting, while our implementation parallelizes over features for faster runtime.
    
    \paragraph{Interval Coverage Rates} To validate the proposed central limit theorems, we first construct feature-wise confidence intervals via binning and compute their coverage rates. Figure \ref{fig:feature-coverage} shows good coverage and width. In combination with Figure \ref{fig:coverage-1d}, we can qualitatively argue that their performance appears desirable.

    Figure \ref{fig:ebm-effects-wide} and Figure \ref{fig:obesity-feature-effects} yield an example on real-world data from the obesity dataset of \cite{estimation_of_obesity_levels_based_on_eating_habits_and_physical_condition__544}, regressing weight on other covariates. Our intervals offer much-needed uncertainty estimates for each effect.
    
    Figure \ref{fig:coverage-rates} compares the coverage and width of the CIs and PIs for the overall response, and RIs for another realization. We use the same conformal correction that \cite{fang2025gradient} use for the PIs only. Coverage is close to nominal for the CIs and PIs, while the RIs exhibit some overcoverage. 

    We further evaluate scalability in higher-dimensional, correlated synthetic additive models for Algorithm~\ref{alg:ebm_bratp} under the realistic Assumption~\ref{aspt:kernel-spec}. In Figure~\ref{tab:vanilla_vs_inferable_gam}, bootstrap EBMs achieve low RMSE but severely under-cover and are slow due to repeated refits, while Algorithm~\ref{alg:ebm_bratp} (and BRATD) attain near-nominal coverage; moreover, Algorithm~\ref{alg:ebm_bratp} yields better RMSE and tighter intervals than BRATD at comparable or lower runtime. Compared with \cite{fang2025gradient}, our intervals are also faster and better-behaved, since we avoid the curse of dimensionality and our runtime is independent of $n$.

    Figure \ref{fig:highdim-real-table} shows that on real datasets with $p\approx 110$--$130$, Algorithm~\ref{alg:ebm_bratp} maintains near-nominal $95\%$ coverage ($0.948$--$0.956$), achieves RMSE comparable to the baselines, and is substantially faster than vanilla EBM. Relative to BRATD, it attains similar RMSE with lower training time on three of the four datasets, especially on the largest example.

    \addtocounter{figure}{-1}
	\begin{figure}[H]
		\hspace{-3ex}\includegraphics[width=1.1\linewidth]{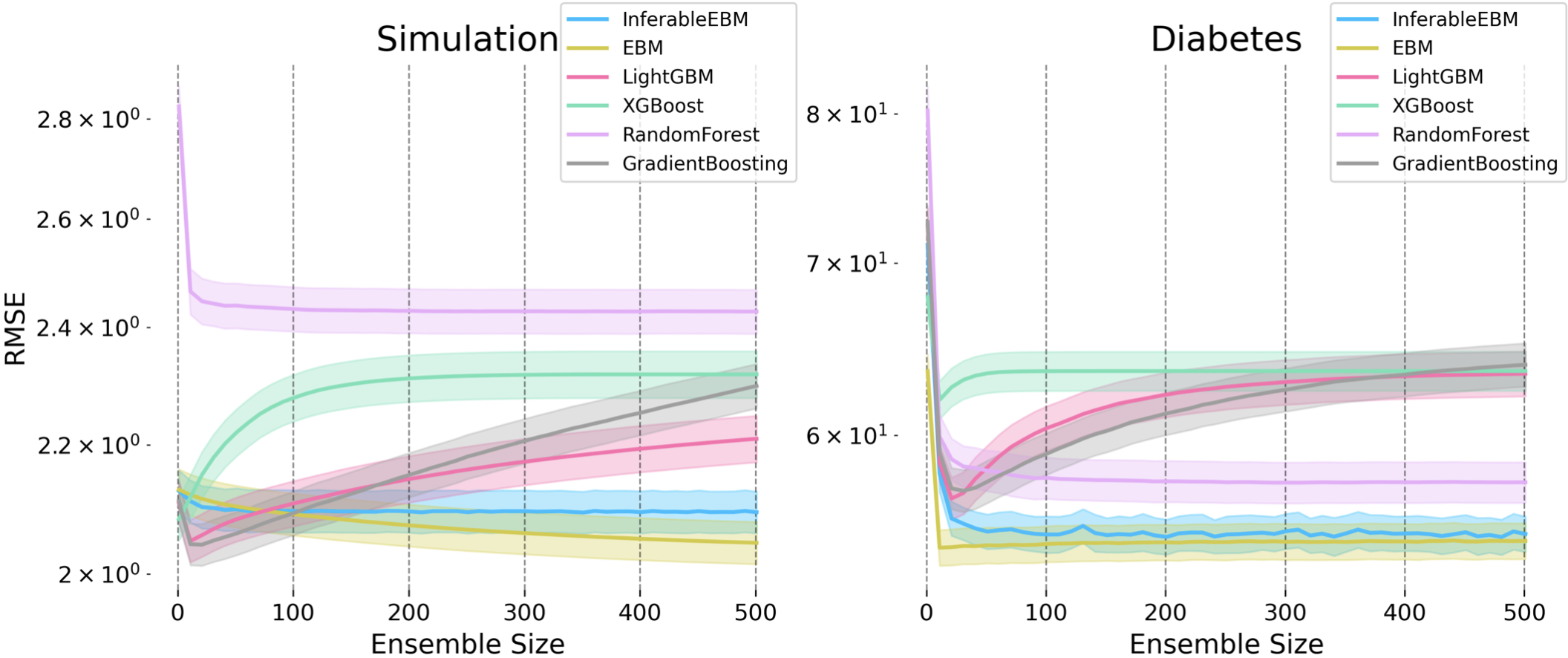}
		\caption{RMSE for Algorithm \ref{alg:brebm_ident} and benchmarks with 2 s.e. bands over 50 trials, on $f(x) = \sin(4\pi x)$ and the diabetes dataset from \cite{Efron_2004}. Both EBM and Algorithm \ref{alg:brebm_ident} are highly resistant to overfitting.}
		\label{fig:overfit_mse}
	\end{figure}
    \addtocounter{figure}{1}

    \section{Conclusion}
	Explainable boosting machines have been an important tool for producing glass-box models, used both to establish insight into the underlying patterns that generate data, and to diagnose potential non-causal behavior prior to deployment. The understanding presented by such methods must be qualified by appropriate uncertainty quantification. Current UQ for EBMs relies on bootstrapping which is both computationally intensive and lacks theoretical backing.
    
    \begin{figure}[H]
        \hspace{-3ex}\includegraphics[width=1.0\linewidth]{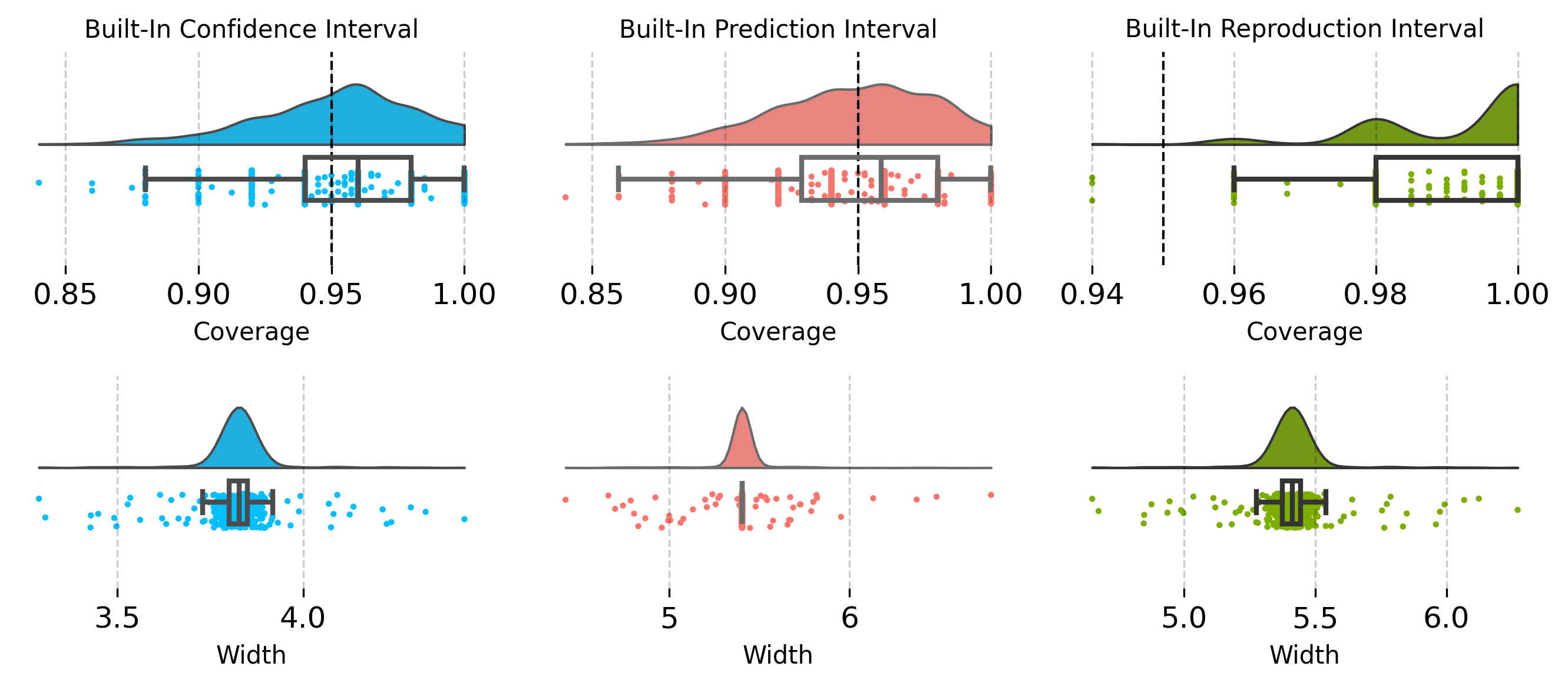}
        \caption{Coverage and width of our intervals over $50$ trials, $n=2000$, on the same function as Figure \ref{fig:feature-coverage}.}
        \label{fig:coverage-rates}
    \end{figure}

    The Boulevard regularization method that we adopt to EBM here both removes the need to choose stopping times (see Figure \ref{fig:overfit_mse}), and generates a limit that can be characterized as a kernel ridge regression. This enables us to directly derive a central limit theorem for predictions of the resulting model and their additive components, and to thereby both derive consistent estimates of their asymptotic variance and provide confidence intervals and tests. 

    This work opens several extensions.  We can readily extend the function class to more complex additive structure; including isolated interaction terms is straightforward, although additive components with shared features require new identifiability conditions \cite[e.g.][]{lengerich2020purifying}. More generally, these methods are extensible to mixtures of learners, such as the class of varying coefficient models discussed in \citet{zhou2022decision}. We can an also broaden the class of models to discrete outcomes and models with non-constant variance. Finally, the kernel form of the limit also admits combinations and comparisons across models as described in \cite{ghosal2022infinitesimal}.
    
    \begin{table}[H]
    \centering
    \scriptsize
    \setlength{\tabcolsep}{3.2pt} 
    \renewcommand{\arraystretch}{1.15} 
    
    \begin{tabular}{c l c c c c}
        \toprule
        $p$ & Method & RMSE & Coverage & Width & Time (s) \\
        \specialrule{0.6pt}{0.35em}{0.35em}
        
        \multirow{3}{*}{5}
        & Bootstrap EBM                    & 0.514 & 0.290 & 0.386 & 26.42 \\
        & Algorithm~\ref{alg:ebm_bratp}     & 0.661 & 0.951 & 2.609 &  3.00 \\
        & BRATD                            & 0.611 & 0.956 & 2.899 & 10.41 \\
        \specialrule{0.6pt}{0.35em}{0.35em}
        
        \multirow{3}{*}{10}
        & Bootstrap EBM                    & 0.537 & 0.373 & 0.545 & 21.24 \\
        & Algorithm~\ref{alg:ebm_bratp}     & 0.788 & 0.951 & 3.069 &  5.90 \\
        & BRATD                            & 0.866 & 0.935 & 3.531 & 13.58 \\
        \specialrule{0.6pt}{0.35em}{0.35em}
        
        \multirow{3}{*}{15}
        & Bootstrap EBM                    & 0.563 & 0.453 & 0.675 & 31.51 \\
        & Algorithm~\ref{alg:ebm_bratp}     & 0.900 & 0.945 & 3.520 &  8.77 \\
        & BRATD                            & 1.112 & 0.946 & 4.398 & 10.29 \\
        \specialrule{0.6pt}{0.35em}{0.35em}
        
        \multirow{3}{*}{20}
        & Bootstrap EBM                    & 0.583 & 0.511 & 0.801 & 43.14 \\
        & Algorithm~\ref{alg:ebm_bratp}     & 1.004 & 0.948 & 3.901 & 11.56 \\
        & BRATD                            & 1.332 & 0.948 & 5.121 & 10.70 \\
        \bottomrule
    \end{tabular}
    
    \vspace{0.25em}
    \caption{Synthetic Lipschitz GAMs (bootstrap EBM vs.\ Algorithm~\ref{alg:ebm_bratp} vs.\ BRATD).
    Synthetic additive model with $n=1000$ and $p\in\{5,10,15,20\}$ (all $p$ features are structural; five base univariate functions are cycled across coordinates):
    $f_1(x)=2\sin(\pi x)$, $f_2(x)=(x-0.5)^2$, $f_3(x)=0.5\times\mathbf{1}\{x>0.3\}$, $f_4(x)=\sqrt{x}$, $f_5(x)=-1.5x$.
    We report mean RMSE, model-level empirical coverage at nominal $95\%$, average interval width, and runtime (seconds) over 10 repetitions ($\alpha=0.05$; 250 bootstrap resamples for EBM).}
    
    \label{tab:vanilla_vs_inferable_gam}
    \end{table}

    \begin{table}[H]
    \centering
    \scriptsize
    \setlength{\tabcolsep}{3.5pt}

    \begin{tabular}{@{}l r r c c c@{}}
        \toprule
        Dataset & $n$ & $p$ &
        \begin{tabular}[c]{@{}c@{}}RMSE\\(Alg.~\ref{alg:ebm_bratp})\end{tabular} &
        \begin{tabular}[c]{@{}c@{}}Cov.\\(Alg.~\ref{alg:ebm_bratp})\end{tabular} &
        \begin{tabular}[c]{@{}c@{}}Time (s)\\(Alg.~\ref{alg:ebm_bratp})\end{tabular} \\
        \midrule
        taiwanese\_bankruptcy   &  6819 &  95  & 0.151 & 0.948 & 18.01 \\
        myocardial\_infarction  &  1700 & 111  & 0.309 & 0.943 &  3.53 \\
        communities\_and\_crime &  1994 & 127  & 0.140 & 0.953 &  5.68 \\
        dota2\_results          & 50000 & 116  & 0.142 & 0.956 & 28.78 \\
        \bottomrule
    \end{tabular}

    \vspace{0.35em}

    \begin{tabular}{@{}l c c c c@{}}
        \toprule
        Dataset &
        \begin{tabular}[c]{@{}c@{}}RMSE\\(EBM)\end{tabular} &
        \begin{tabular}[c]{@{}c@{}}Time (s)\\(EBM)\end{tabular} &
        \begin{tabular}[c]{@{}c@{}}RMSE\\(BRATD)\end{tabular} &
        \begin{tabular}[c]{@{}c@{}}Time (s)\\(BRATD)\end{tabular} \\
        \midrule
        taiwanese\_bankruptcy   & 0.153 &  197.31 & 0.154 & 24.58 \\
        myocardial\_infarction  & 0.309 &   52.62 & 0.311 &  3.78 \\
        communities\_and\_crime & 0.135 &  178.41 & 0.144 &  7.80 \\
        dota2\_results          & 0.013 & 3736.10 & 0.147 & 82.62 \\
        \bottomrule
    \end{tabular}

    \caption{We report RMSE, $95\%$ empirical coverage, and training time for Algorithm~\ref{alg:ebm_bratp} (top), and baseline RMSE and training time for vanilla EBM and BRATD (bottom). Baseline coverage is omitted because large-scale uncertainty estimation is not computationally scalable.}
    \label{fig:highdim-real-table}
    \end{table}
	
	\bibliographystyle{apalike}
	\bibliography{refs}

\section*{Checklist}



\begin{enumerate}

  \item For all models and algorithms presented, check if you include:
  \begin{enumerate}
    \item A clear description of the mathematical setting, assumptions, algorithm, and/or model. [Yes; See Sections \ref{sec:setup} and \ref{sec:algorithms}, and Appendices \ref{app:binning} and \ref{app:algorithms} ]
    \item An analysis of the properties and complexity (time, space, sample size) of any algorithm. [Yes. Otherwise, the remaining algorithms are modification of existing boosting methods and share their properties with those already in the literature -- namely a runtime of $O(Bnp \log n)$, and auxiliary space of $O(Bnp)$.]
    \item (Optional) Annonymized source code, with specification of all dependencies, including external libraries. [Yes. Source code provided as a .zip file. All experiments were run on a single machine with an Intel i9-13900K CPU, 128 GB of RAM, and an NVIDIA GeForce RTX 3090Ti GPU, though no GPU acceleration was employed.]
  \end{enumerate}

  \item For any theoretical claim, check if you include:
  \begin{enumerate}
    \item Statements of the full set of assumptions of all theoretical results. [Yes; See sections \ref{sec:setup} and \ref{sec:theory}]
    \item Complete proofs of all theoretical results. [Yes; see Appendices \ref{sec:finite-sample-convergence}, \ref{sec:limiting-distribution}, and \ref{app:ContractionMapping}]
    \item Clear explanations of any assumptions. [Yes; See section \ref{sec:theory}]     
  \end{enumerate}

  \item For all figures and tables that present empirical results, check if you include:
  \begin{enumerate}
    \item The code, data, and instructions needed to reproduce the main experimental results (either in the supplemental material or as a URL). [Yes; see provided .zip file]
    \item All the training details (e.g., data splits, hyperparameters, how they were chosen). [Yes; see provided .zip file]
    \item A clear definition of the specific measure or statistics and error bars (e.g., with respect to the random seed after running experiments multiple times). [Yes]
    \item A description of the computing infrastructure used. (e.g., type of GPUs, internal cluster, or cloud provider). [Yes]
  \end{enumerate}

  \item If you are using existing assets (e.g., code, data, models) or curating/releasing new assets, check if you include:
  \begin{enumerate}
    \item Citations of the creator If your work uses existing assets. [Not Applicable]
    \item The license information of the assets, if applicable. [Not Applicable]
    \item New assets either in the supplemental material or as a URL, if applicable. [Not Applicable]
    \item Information about consent from data providers/curators. [Not Applicable]
    \item Discussion of sensible content if applicable, e.g., personally identifiable information or offensive content. [Not Applicable]
  \end{enumerate}

  \item If you used crowdsourcing or conducted research with human subjects, check if you include:
  \begin{enumerate}
    \item The full text of instructions given to participants and screenshots. [Not Applicable]
    \item Descriptions of potential participant risks, with links to Institutional Review Board (IRB) approvals if applicable. [Not Applicable]
    \item The estimated hourly wage paid to participants and the total amount spent on participant compensation. [Not Applicable]
  \end{enumerate}

\end{enumerate}

\clearpage
	
	\appendix
	\onecolumn

    \clearpage
\section{Notations}
\label{app:notations}


\begin{table}[H]
    \centering
    \small
    \setlength{\tabcolsep}{6pt}
    \renewcommand{\arraystretch}{1.15}
    \begin{tabular}{p{0.20\linewidth} p{0.72\linewidth}}
    \toprule
    \textbf{Symbol} & \textbf{Meaning / definition (with dimensions when helpful)} \\
    \midrule
    $n$ & Sample size; data $(\bX_n,\by_n)=\{(\bx_i,y_i)\}_{i=1}^n$ are i.i.d. \\
    $p$ & Number of features (coordinates) in $\bx\in\R^p$. \\
    $i,j,\ell$ & Generic indices: $i,j\in\{1,\dots,n\}$ for samples; $\ell\in\{1,\dots,p\}$ for features. \\
    $\bx_i$ & $i$-th covariate vector: $\bx_i=(\bx_i^{(1)},\dots,\bx_i^{(p)})^\top\in\R^p$. \\
    $\bx_i^{(k)}$ & Feature-$k$ coordinate of $\bx_i$ (scalar in the univariate EBM setting). \\
    $\bX\in\R^{n\times p}$ & Design matrix with rows $\bx_i^\top$. \\
    $y_i$ & Response for sample $i$ (scalar). \\
    $\by\in\R^n$ & Response vector $(y_1,\dots,y_n)^\top$. \\
    $\epsilon$ / $\varepsilon$ & Noise; i.i.d. sub-Gaussian errors with $\E[\varepsilon]=0$, $\text{Var}(\varepsilon)=\sigma^2<\infty$. \\
    $\sigma^2$ & Noise variance in Assumption 1. \\
    $\mu(\bx)$ & Density of covariates on their support; bounded as $0<c_1\le \mu(\bx)\le c_2<\infty$. \\
    $\bbeta$ & True intercept: $\bbeta := \E_{\bx,\epsilon}[y]$ (unconditional mean of $y$). \\
    $\widehat{\bbeta}_b$ & Estimated intercept at boosting round $b$ (Algorithm 1 / Alg. in appendix). Initialized as $\widehat{\bbeta}_0=\frac{1}{n}\bone^\top \by$. \\
    $\bar{\by}$ & Sample mean of responses: $\bar{\by}:=\frac{1}{n}\bone^\top \by$. \\
    $\f(\bx)$ & True regression function in GAM form: $\f(\bx)=\bbeta+\sum_{k=1}^p \f^{(k)}(\bx^{(k)})$. \\
    $\f^{(k)}$ & True univariate component function for feature $k$; $L$-Lipschitz and mean-zero w.r.t. marginal of $\bx^{(k)}$. \\
    $\widehat{\f}$ & Generic fitted predictor; in Algorithm 1 output is $\widehat\f(\bx)=\widehat{\bbeta}_B+\frac{1+\lambda}{\lambda}\sum_{k=1}^p \f_B^{(k)}(\bx^{(k)})$. \\
    $\f_b^{(k)}$ & Algorithm’s stored additive component for feature $k$ after round $b$ (a function of $\bx^{(k)}$). Initialized $\f_0^{(k)}\equiv 0$. \\
    $\widehat{\mathbf{y}}_b^{(k)}$ & Feature-$k$ \emph{vector} prediction on training points at round $b$ (used in Theorems): $\widehat{\mathbf{y}}_b^{(k)} \in \R^n$. \\
    $\widehat{\mathbf{y}}_b$ & Total fitted values on training set at round $b$: $\widehat{\mathbf{y}}_b := \widehat{\bbeta}_b + \sum_{a=1}^p \widehat{\mathbf{y}}_b^{(a)}$ (Theorem 1). \\
    $\bone$ & All-ones vector in $\R^n$. \\
    $\bI_n$ / $\bI$ & Identity matrix in $\R^{n\times n}$ (often written $\bI$ when $n$ is clear). \\
    $\jmat$ / $\mathbf{J}_n$ & Centering projector: $\jmat=\bI_n-\frac{1}{n}\bone\bone^\top$. \\
    \bottomrule
    \end{tabular}
    \caption{Basic data \& model notation. }
    \label{tab:notation_basic}
\end{table}

\begin{table}[H]
    \centering
    \small
    \setlength{\tabcolsep}{6pt}
    \renewcommand{\arraystretch}{1.15}
    \begin{tabular}{p{0.20\linewidth} p{0.72\linewidth}}
    \toprule
    \textbf{Symbol} & \textbf{Meaning / definition (with dimensions when helpful)} \\
    \midrule
    $b$ & Boosting round index ($b=1,2,\dots$). \\
    $B$ & Number of boosting rounds (loop bound in pseudocode). \\
    $\lambda$ & Learning rate in Boulevard-style moving average update. \\
    $\xi$ & Subsample rate: each $G_{b,k}\subset\{1,\dots,n\}$ includes each index i.i.d. with prob. $\xi$. \\
    $M$ & Truncation / clipping level: $\widetilde t \gets \max\{-M,\min\{\widetilde t, M\}\}$. \\
    $G_{b,k}$ & Subsample index set used to fit tree for feature $k$ at round $b$. \\
    $r_{i,k}$ & Residual used to fit feature-$k$ tree at round $b$.
    Alg 1: $r_{i,k} = y_i - \big(\widehat{\bbeta}_{b-1}+\sum_{\ell}\f_{b-1}^{(\ell)}(\bx_i^{(\ell)})\big)$.
    Alg 2 (leave-one-out): sum excludes $\ell=k$. \\
    $\bt^{(b,k)}$ & Regression tree trained at round $b$ on feature $k$ using pairs $\{(\bx_i^{(k)}, r_{i,k})\}_{i\in G_{b,k}}$. \\
    $\mu_{b,k}$ & Mean (centering constant) of tree predictions on training set: $\mu_{b,k}=\frac{1}{n}\sum_{i=1}^n \bt^{(b,k)}(\bx_i^{(k)})$. \\
    $\widetilde{\bt}_b^{(k)}$ & Centered (and then clipped) tree: $\widetilde{\bt}_b^{(k)}(\bx)=\bt^{(b,k)}(\bx)-\mu_{b,k}$, then clipped to $[-M,M]$. \\
    \bottomrule
    \end{tabular}
    \caption{Algorithmic hyperparameters \& indices.}
    \label{tab:notation_algo}
\end{table}
\clearpage

\begin{table}[H]
    \centering
    \small
    \setlength{\tabcolsep}{6pt}
    \renewcommand{\arraystretch}{1.15}
    \begin{tabular}{p{0.20\linewidth} p{0.72\linewidth}}
    \toprule
    \textbf{Symbol} & \textbf{Meaning / definition (with dimensions when helpful)} \\
    \midrule
    $\bt_n$ & A regression tree trained on $n$ datapoints; induces a partition $\{A_i\}_{i=1}^m$ of $[0,1]^d$ (Def. 1). \\
    $A(\bx)$ & The leaf (rectangle) containing $\bx$ under the tree partition. \\
    $s_{n,j}(\bx)$ & Structure weight of training point $j$ for query $\bx$:
    $s_{n,j}(\bx)=\frac{\mathbbm{1}(\bx_j\in A(\bx))}{\sum_{k=1}^n \mathbbm{1}(\bx_k\in A(\bx))}$. \\
    $\bs_n(\bx)$ & Structure vector: $\bs_n(\bx)=(s_{n,1}(\bx),\dots,s_{n,n}(\bx))^\top\in\R^n$. \\
    $\bS_n$ & Structure matrix: $\bS_n=(s_{n,j}(\bx_i))_{i,j=1}^n\in\R^{n\times n}$; row $i$ is $\bs_n(\bx_i)^\top$. \\
    $\Smatk_b$ & Feature-$k$ structure matrix at boosting round $b$ (Assumption 2). \textbf{[I interpret $\Smatk_b$ as $\bS_n$ for the univariate tree fit on feature $k$ at round $b$.]} \\
    $\Kmatk$ & Expected feature-$k$ kernel: $\Kmatk := \E[\Smatk]$ (also written $\E[\Smatk]$). \\
    $\bK$ & Aggregated kernel (main text / Thm 1): $\bK := \sum_{k=1}^p \Kmatk$. \\
    $\mathcal{Q}_n^{(k)}$ & Distribution over feature-$k$ tree structures after burn-in (Assumption 2). \\
    $Q_n^{(k)}$ & Tree-structure support / “tree space” for feature $k$ (Assumptions 3--5). \\
    $b'$ & Burn-in iteration after which structures are i.i.d. from $\mathcal{Q}_n^{(k)}$ (Assumption 2). \\
    \bottomrule
    \end{tabular}
    \caption{Structure vectors \& matrices. }
    \label{tab:notation_kernel}
\end{table}

\begin{table}[H]
    \centering
    \small
    \setlength{\tabcolsep}{6pt}
    \renewcommand{\arraystretch}{1.15}
    \begin{tabular}{p{0.20\linewidth} p{0.72\linewidth}}
    \toprule
    \textbf{Symbol} & \textbf{Meaning / definition (with dimensions when helpful)} \\
    \midrule
    $\Pi$ & A tree partition (set of leaves) drawn from $Q_n^{(k)}$. \\
    $\mathsf{diam}(A)$ & Leaf diameter: $\sup_{\bx_1,\bx_2\in A}\|\bx_1-\bx_2\|$. \\
    $d_n$ & Target diameter scale: $d_n=O(n^{-1/3})=o(1/\log n)$ (Assumption 3). \\
    $v_n$ & Leaf geometric volume bound, with $v_n=n^{-2/3+\nu}$ for any $\nu>0$ (Assumption 4). \\
    $\epsilon_n$ & Small sequence $\epsilon_n\to 0$ in restricted tree support bound (Assumption 5). \\
    \bottomrule
    \end{tabular}
    \caption{Geometric \& complexity parameters for trees. }
    \label{tab:notation_geom}
\end{table}

\begin{table}[H]
    \centering
    \small
    \setlength{\tabcolsep}{6pt}
    \renewcommand{\arraystretch}{1.15}
    \begin{tabular}{p{0.20\linewidth} p{0.72\linewidth}}
    \toprule
    \textbf{Symbol} & \textbf{Meaning / definition (with dimensions when helpful)} \\
    \midrule
    $\widetilde{\by}_k^*$ & Feature-$k$ fixed point vector (Thm 1 / Thm 2). In Thm 1:
    $\widetilde{\by}_k^*=\jmat\,\Kmatk[\lambda^{-1}\bI+\jmat\bK]^{-1}\by$. \\
    $\widehat{\bbeta}^*$ & Intercept fixed point: $\widehat{\bbeta}^*=\bar{\by}$. \\
    $\widehat{\by}^*$ & Total fixed point fit on training set: $\widehat{\by}^*=\widehat{\bbeta}^*\bone+\sum_{k=1}^p \widetilde{\by}_k^*$. \\
    $\mK$ & “combined kernel” for Algorithm~\ref{alg:ebm_bratp}:
    $\mathcal{K}:=\sum_{k=1}^p(\bI-\E \Smatk)^{-1}\jmat\E \Smatk$.\\
    $\mK_{\text{bin}}$ & $\mK_{\text{bin}}=\sum_{k=1}^p \E[\bb_n^{(k)}\bB^{(k)}\bb_n^{(k)\top}]$. \\
    $\rveck$ / $\rveck_E(\bx)$ & Kernel-ridge “influence / weight” vector used in CLTs and intervals; satisfies (example for Alg 1):
    ${\rveck}^\top\by = \E[\sveck(\bx)]^\top \jmat[\bI+\jmat\bK]^{-1}\by$.
    Subscript $E\in\{A,B\}$ labels algorithm family (Alg 1 vs Alg 2) in Theorem 3. \\
    $\kveck$ / $\kveck_n(\bx)$ & Kernel vector (expected structure vector / feature-$k$ similarity to training points). In binning appendix:
    $\kveck_n(\bx)=\E[\bb^{(k)}(\bx)^\top \bD^{(k)} \bb_n^{(k)\top}]\in\R^{1\times n}$. \\
    $c_E$ & Algorithm-dependent scaling constant in CLT: $c_A=\lambda/(1+\lambda)$ for Algorithms \ref{alg:brebm_ident}, \ref{alg:brebm_random_cyclic}; $c_B=1$ for Algorithm \ref{alg:ebm_bratp} (Theorem 3). \\
    $z_{1-\alpha/2}$ & Standard normal quantile used for $(1-\alpha)$ intervals. \\
    \bottomrule
    \end{tabular}
    \caption{Fixed points, CLT scaling, weight vectors. }
    \label{tab:notation_clt}
\end{table}
\clearpage

\begin{table}[H]
    \centering
    \small
    \setlength{\tabcolsep}{6pt}
    \renewcommand{\arraystretch}{1.15}
    \begin{tabular}{p{0.20\linewidth} p{0.72\linewidth}}
    \toprule
    \textbf{Symbol} & \textbf{Meaning / definition (with dimensions when helpful)} \\
    \midrule
    $\mathcal{H}_k$ & Centered feature-$k$ additive subspace:
    $\mathcal H_k=\{\bv\in\R^n:\exists g:\R\to\R,\ \bv_i=g(\bx_i^{(k)}),\ \bone^\top\bv=0\}$. \\
    $P_k$ & Orthogonal projector onto $\mathcal{H}_k$. \\
    $\varepsilon_{n,k}$ & Operator-norm approximation error s.t.
    $\|\jmat\Kmatk\jmat - P_k\|_{\mathrm{op}}\le \varepsilon_{n,k}\to 0$. \\
    \bottomrule
    \end{tabular}
    \caption{Additive subspaces and projectors; Assumption~\ref{aspt:proj-orth}). }
    \label{tab:notation_proj}
\end{table}

\begin{table}[H]
    \centering
    \small
    \setlength{\tabcolsep}{6pt}
    \renewcommand{\arraystretch}{1.15}
    \begin{tabular}{p{0.20\linewidth} p{0.72\linewidth}}
    \toprule
    \textbf{Symbol} & \textbf{Meaning / definition (with dimensions when helpful)} \\
    \midrule
    $m_k$ & Number of histogram bins for feature $k$; $m:=\max_k m_k$. \\
    $B_r^{(k)}$ & The $r$-th bin for feature $k$ (a subset/interval of $\R$). \\
    $r_i^{(k)}$ & Bin index of sample $i$ on feature $k$: unique $r$ with $\bx_i^{(k)}\in B_r^{(k)}$. \\
    $\mathcal{L}^{(k)}(r)$ & Set of bin indices in the same tree leaf as bin $r$ (feature $k$):
    $\{s\in\{1,\dots,m_k\}: B_s^{(k)}\in A(B_r^{(k)})\}$. \\
    $n_s^{(k)}$ & Bin count: number of samples in bin $s$ for feature $k$,
    $n_s^{(k)}=\sum_{j=1}^n \bone(\bx_j^{(k)}\in B_s^{(k)})$. \\
    $N_{\mathcal{L}(r)}^{(k)}$ & Leaf sample count for the leaf containing bin $r$:
    $N_{\mathcal{L}(r)}^{(k)}=\sum_{q\in\mathcal{L}^{(k)}(r)} n_q^{(k)}$. \\
    $|\mathcal{L}^{(k)}(r)|$ & Leaf bin count (number of bins in that leaf). \\
    $\bZ_n^{(k)}\in\{0,1\}^{n\times m_k}$ & Sample-to-bin assignment matrix: $\bZ_{i,j}^{(k)}=\bone(r_i^{(k)}=j)=\bone(\bx_i^{(k)}\in B_j^{(k)})$. \\
    $\bD^{(k)}=\mathrm{diag}(\xi_1^{(k)},\dots,\xi_{m_k}^{(k)})$ &
    Diagonal rescaling to convert bin-normalization to sample-normalization, with
    $\xi_r^{(k)}=\sqrt{\frac{|\mathcal{L}^{(k)}(r)|}{N_{\mathcal{L}(r)}^{(k)}}}$. \\
    $\bB^{(k)}\in\R^{m_k\times m_k}$ & Bin-structure matrix: $\bB^{(k)}_{i,j}=\frac{\bone(B_j^{(k)}\in A(B_i^{(k)}))}{\sum_{l=1}^{m_k}\bone(B_l^{(k)}\in A(B_i^{(k)}))}$. \\
    $\bb_n^{(k)}$ & Matrix mapping bin-space to sample-space. Initially given as $\bb_n^{(k)}\in\R^{n\times m_k}$; later simplified to $\bb_n^{(k)}=\bZ_n^{(k)}\bD^{(k)}$. \\
    $\bb^{(k)}(\bx)\in\R^{m_k}$ & Bin-space structure vector for query $\bx$:
    $\bb^{(k)}(\bx)^\top=\left(\frac{\bone(B_j^{(k)}\in A(\bx))}{\sum_{l=1}^{m_k}\bone(B_l^{(k)}\in A(\bx))}\right)_{j=1}^{m_k}$. \\
    $\bh^{(k)}(\bx)$ & Defined in appendix lemma: $\bh^{(k)}(\bx)=\E[(\bD^{(k)})^2\bb^{(k)}(\bx)]\in\R^{m_k}$. \\
    $\bH$ & Sum of expected bin-kernels:
    $\bH=\sum_{k=1}^p \E[\bD^{(k)}\bB^{(k)}\bD^{(k)}]$.\\
    $\mathbf{c}^{(k)}$ & Bin counts vector: $\mathbf{c}^{(k)}:=\bZ_n^{(k)\top}\bone\in\R^{m_k}$. \\
    \bottomrule
    \end{tabular}
    \caption{Binning \& histogram-tree appendix. }
    \label{tab:notation_binning}
\end{table}

	\section{Additional Literature} \label{app:lit}
	
	\paragraph{GAMs.}
	\cite{2f43af29-e27b-3b6f-9c3a-d7c3bd53adac} established the asymptotic distribution of smoothing spline estimators via a Bayesian posterior covariance argument, \cite{https://doi.org/10.1111/j.2517-6161.1985.tb01327.x} developed the equivalent kernel representation and convergence theory, and \cite{Nychka01121988} analyzed the frequentist coverage of Bayesian confidence intervals for smoothing splines. Built on these developments, \cite{https://doi.org/10.1111/j.1467-842X.2006.00450.x} investigated the coverage of penalized regression spline intervals in the GAM setting, while \cite{yoshida2012asymptoticspenalizedsplinesgeneralized} provided a full asymptotic distribution theory for penalized spline estimators, establishing asymptotic normality for each additive component. 

    \section{Algorithms} \label{app:algorithms}

	\begin{algorithm}[H]
		\caption{Leave-one-out Inferable EBM}
		\label{alg:ebm_bratp}
		\begin{algorithmic}[1]
			\State \textbf{Input:} $\mathbf{X}\in\mathbb{R}^{n\times p}$, $\mathbf{y}\in\mathbb{R}^n$, subsample rate $\xi$, rounds $B$, truncation level $M>0$.
			\State \textbf{Init:} $\widehat{\bbeta}_0 \gets \frac{1}{n}\sum_{i=1}^n \by_i$;\quad $\f_k^{(0)} \gets 0$ for all $k$.
			\For{$b=1$ \textbf{ to } $B$}
			\For{$k=1$ \textbf{ to } $p$ \textbf{ in parallel}}
			\State Sample $G_{b,k}\subset\{1,\dots,n\}$ i.i.d. w.p. $\xi$.
			\State $r_{i,k} \gets \by_i - (\widehat{\bbeta}_{b-1} + \sum_{\ell \neq k} \f_{b-1}^{(\ell)}(\bx_i^{(\ell)}))$.
			\State Fit tree $\bt^{(b,k)}$ to $\{(\bx_i^{(k)}, r_{i,k})\}_{i\in G_{b,k}}$.
			\State $\mu_{b,k} \gets \frac{1}{n}\sum_{i=1}^n \bt^{(b,k)}(\bx_i^{(k)})$.
			\State $\widetilde \bt_b^{(k)}(\bx) \gets t_b^{(k)}(\bx) - \mu_{b,k}$.
			\State $\widetilde \bt_b^{(k)}(x) \gets \max\{-M,\;\min\{\widetilde \bt_b^{(k)}(\bx),M\}\}$.
			\State $\widehat{\bbeta}_b \gets \widehat{\bbeta}_{b-1} + \mu_{b,k}$.
			\State $\f_b^{(k)} \gets \frac{b-1}{b} \f_{b-1}^{(k)} + \frac{1}{b}\,\widetilde \bt_b^{(k)}$.
			\EndFor
			\EndFor
			\State \textbf{Output:} $\widehat \f(x) \gets \widehat{\bbeta}_B + 2\sum_{k=1}^p \f_k^{(B)}(\bx^{(k)})$.
		\end{algorithmic}
	\end{algorithm}

	\begin{algorithm}[ht!]
		\caption{Random Cyclic Inferable EBM}
		\label{alg:brebm_random_cyclic}
		\begin{algorithmic}[1]
			\State \textbf{Input:} $\mathbf{X}\in\mathbb{R}^{n\times p}$, $\mathbf{y}\in\mathbb{R}^n$, subsample rate $\xi$, rounds $B$, truncation level $M>0$.
			\State \textbf{Init:} $\widehat{\bbeta}_0 \gets \frac{1}{n}\sum_{i=1}^n \by_i$;\quad $\f_k^{(0)} \gets 0$ for all $k$.
			\For{$b=1$ \textbf{ to } $B$}
			\State Randomly sample $k \sim \text{Categorical}(1/p,...,1/p)$
			\State Sample $G_{b,k}\subset\{1,\dots,n\}$ i.i.d. w.p. $\xi$.
			\State $r_{i,k} \gets \by_i - (\widehat{\bbeta}_{b-1} + \sum_{\ell} \f_{b-1}^{(\ell)}(\bx_i^{(\ell)}))$.
			\State Fit tree $\bt^{(b,k)}$ to $\{(\bx_i^{(k)}, r_{i,k})\}_{i\in G_{b,k}}$.
			\State $\mu_{b,k} \gets \frac{1}{n}\sum_{i=1}^n \bt^{(b,k)}(\bx_i^{(k)})$.
			\State $\widetilde \bt_b^{(k)}(\bx) \gets t_b^{(k)}(\bx) - \mu_{b,k}$.
			\State $\widetilde \bt_b^{(k)}(x) \gets \max\{-M,\;\min\{\widetilde \bt_b^{(k)}(\bx),M\}\}$.
			\State $\widehat{\bbeta}_b \gets \widehat{\bbeta}_{b-1} + \mu_{b,k}$.
			\State $\f_b^{(k)} \gets \frac{b-1}{b} \f_{b-1}^{(k)} + \frac{1}{b}\,\widetilde \bt_b^{(k)}$.
			\EndFor
			\State \textbf{Output:} $\widehat \f(x) \gets \widehat{\bbeta}_B + 2\sum_{k=1}^p \f_k^{(B)}(\bx^{(k)})$.
		\end{algorithmic}
	\end{algorithm}

	\section{Binning} \label{app:binning}
	
	\subsection{Bin-level Decomposition of the Structure Matrix}
	
	In the special case where one uses histogram trees, there is a simpler form for the structure vector and matrix. When one performs binning automatically during pre-fitting, we can store the structure matrix in such a decomposition:
	
	$$
	\Smatk = \bb^{(k)}_n\bB^{(k)}{\bb^{(k)}_n}^\top
	, 
	\bB^{(k)}_{i,j} = 
	\begin{bmatrix}
		\frac{\bone(B_j^{(k)}\in A(B^{(k)}_i))}{\sum_{l=1}^{m_k}\bone(B_l^{(k)}\in A(B_i^{(k)}))}
	\end{bmatrix}_{m_k\times m_k},
	(\bb_n^{(k)})_{i,j}
	= \begin{bmatrix}
		\bone(x_i^{(k)} \in B_j^{(k)})\frac{\sum_{l=1}^{m_k}\bone\left(B_l^{(k)}\in A(B_j^{(k)})\right)}{\sum_{q=1}^{n}\bone\left(\bx_q^{(k)}\in A(B_j^{(k)})\right)}
	\end{bmatrix}_{n \times m_k}, 
	$$
	
	where $A(\cdot)$ is the leaf that the bin (or sample) in its argument is in. The matrix $\bB^{(k)}$ then is the bin-structure matrix where the $i,j$-th entry contains the indicators of whether bin $i$ is in the same leaf as bin $j$, normalized by the number of bins in the same leaf as bin $i$.
	
	We now introduce a few more definitions:
	\begin{enumerate}
		\item \textbf{Bin index of a sample. } For each sample $i$, let $r_i^{(k)} \in\left\{1, \ldots, m_k\right\}$ be the unique bin index with $\bx_i^{(k)} \in B_{r_i^{(k)}}^{(k)}$. That is, this is the index of the histogram bin that sample $i$ belongs to along feature $k$.
		\item \textbf{The set of bins in a given leaf. }
		For a bin index $r$, define the set of bins in the unique tree leaf containing the bin $B_r^{(k)}$ with index $r$ along feature $k$:
		
		$$
		\mathcal{L}^{(k)}(r):=\left\{s \in\left\{1, \ldots, m_k\right\}: B_s^{(k)} \in A\left(B_r^{(k)}\right)\right\} .
		$$
		
		These are exactly the bins merged together by the tree into the same terminal node as bin $r$.
		\item \textbf{ How many samples are in each bin, and how many bins are there in each leaf?}
		
		The bin count $n_s^{(k)}$, or the number of samples in bin $s$ along feature $k$, is given by:
		
		$$
		n_s^{(k)}:=\sum_{j=1}^n \mathbf{1}\left(\bx_j^{(k)} \in B_s^{(k)}\right) .
		$$

		The leaf sample count $N_{\mathcal{L}(r)}^{(k)}$ and leaf bin count $\left|\mathcal{L}^{(k)}(r)\right|$ are given by:
		
		$$
		N_{\mathcal{L}(r)}^{(k)}:=\sum_{q \in \mathcal{L}^{(k)}(r)} n_q^{(k)}, \quad\left|\mathcal{L}^{(k)}(r)\right|:=\sum_{q=1}^{m_k} \mathbf{1}\left(q \in \mathcal{L}^{(k)}(r)\right) .
		$$
		
		The leaf sample count $N_{\mathcal{L}(r)}^{(k)}$ denotes how many training samples ended up in the leaf containing bin $r$. The leaf bin count $\left|\mathcal{L}^{(k)}(r)\right|$ denotes how many bins that the leaf containing bin $r$ aggregates.
		\item \textbf{The assignment matrix of samples to bins. } 
		$$\bZ_n^{(k)} \in\{0,1\}^{n \times m_k}, \text { with } \bZ_{i, j}^{(k)}:=\mathbf{1}\left(r_i^{(k)}=j\right) = \mathbf{1}\left(\bx_i^{(k)} \in B_j^{(k)}\right).$$
		This is the matrix where the $i,j$-th entry is 1 of sample $i$ is in bin $j$, and 0 otherwise. Note that each sample lives in exactly one bin, so each row of $\bZ_n^{(k)}$ has a single 1.
		\item \textbf{Leaf-wise rescaling from bin-space to sample-space. }
		
		Define a diagonal matrix $\bD^{(k)}=\operatorname{diag}\left(\xi_1^{(k)}, \ldots, \xi_{m_k}^{(k)}\right)$ with
		
		$$
		\xi_r^{(k)}:=\sqrt{\frac{\left|\mathcal{L}^{(k)}(r)\right|}{N_{\mathcal{L}(r)}^{(k)}}} = \frac{\sum_{l=1}^{m_k}\bone\left(B_l^{(k)}\in A(B_r^{(k)})\right)}{\sum_{i=1}^{n}\bone\left(\bx_i^{(k)}\in A(B_r^{(k)})\right)}= \sqrt{\frac{\text{Number of bins in leaf containing bin }r}{\text{Number of training samples in leaf containing bin }r}}.
		$$
		
		The intuition is as follows. $\bB^{(k)}$ normalizes over bins in a leaf, while $\bS^{(k)}$ normalizes over samples in a leaf. Some bins have more samples than others. We must therefore correct the bin-structure matrix so that it normalizes over samples rather than bins, and so we construct $\bD^{(k)}$ to do exactly that. Note $\xi_r^{(k)}$ is constant within a leaf, in the sense that if bins $s$ and $r$ are in the same leaf, we will then have $\xi_r^{(k)} = \xi_s^{(k)}$.
	\end{enumerate}
	
	We then have the following result.
	
	\begin{lemma}[Bin-Level Decomposition of Structure Matrix]
		We can decompose the structure matrix as follows:
		$$
		\bS_n^{(k)} = \bb^{(k)}_n\bB^{(k)}{\bb^{(k)}_n}^\top
		,\quad
		\bB^{(k)}_{i,j} = 
		\begin{bmatrix}
			\frac{\bone(B_j^{(k)}\in A(B^{(k)}_i))}{\sum_{l=1}^{m_k}\bone(B_l^{(k)}\in A(B_i^{(k)}))}
		\end{bmatrix}_{m_k\times m_k},
		\quad
		\bb_n^{(k)}
		= \bZ_n^{(k)} \bD^{(k)} \in \R^{n \times m_k}.
		$$
		Similarly, we can also decompose the structure vector into the structure vector in bin space transformed to sample space:
		$$\sveck_n(\bx) = \bb_n^{(k)}\bD^{(k)} \bb^{(k)}(\bx),
		\quad 
		\bb^{(k)}(\bx)^\top = \left(\frac{\bone(B_j^{(k)}\in A(\bx))}{\sum_{l=1}^{m_k}\bone(B_l^{(k)}\in A(\bx))}\right)_{j=1,...,m_k}.$$
	\end{lemma}
	\begin{proof}
		Examine the $i,j$-th entry of the decomposition. 
		$$
		\left(\mathbf{b}_n^{(k)} \bB^{(k)} \mathbf{b}_n^{(k) \top}\right)_{i, j}=\sum_{r=1}^{m_k} \sum_{s=1}^{m_k} \bZ_{i, r}^{(k)} \xi_r^{(k)} \bB_{r, s}^{(k)} \xi_s^{(k)} \bZ_{j, s}^{(k)}
		$$

		Since $\bZ_{i, .}^{(k)}$ and $\bZ_{j, .}^{(k)}$ are both one-hot vectors, the double sum collapses to
		
		$$
		\left(\mathbf{b}_n^{(k)} \bB^{(k)} \mathbf{b}_n^{(k) \top}\right)_{i, j}=\xi_{r_i^{(k)}}^{(k)} \bB_{r_i^{(k)}, r_j^{(k)}}^{(k)} \xi_{r_j^{(k)}}^{(k)} ,
		$$
		where $r_i^{(k)}$ is the bin index of the $i$-th sample across feature $k$.

		There are two cases. First, if the $i$-th and $j$-th samples are not in the same leaf, then $\bB_{r_i^{(k)}, r_j^{(k)}}^{(k)}=0$, so the entry is 0 . This matches $\left(\bS_{n}^{(k)}\right)_{i,j}=0$.
		
		Otherwise, both bins lie in the same leaf $L:=\mathcal{L}^{(k)}\left(r_i^{(k)}\right)=\mathcal{L}^{(k)}\left(r_j^{(k)}\right)$. Then
		
		$$
		\bB_{r_i^{(k)}, r_j^{(k)}}^{(k)}=\frac{1}{|L|}, \quad \xi_{r_i^{(k)}}^{(k)}=\xi_{r_j^{(k)}}^{(k)}=\sqrt{\frac{|L|}{N_L^{(k)}}} .
		$$

		Hence
		
		$$
		\xi_{r_i^{(k)}}^{(k)} \bB_{r_i^{(k)}, r_j^{(k)}}^{(k)} \xi_{r_j^{(k)}}^{(k)}=\left(\sqrt{\frac{|L|}{N_L^{(k)}}}\right) \cdot \frac{1}{|L|} \cdot\left(\sqrt{\frac{|L|}{N_L^{(k)}}}\right)=\frac{1}{N_L^{(k)}} .
		$$

		But
		
		$$
		\left(\bS_{n}^{(k)}\right)_{i,j}
		=\frac{\mathbf{1}\left(\bx_j^{(k)} \in A\left(\bx_i^{(k)}\right)\right)}{\sum_{\ell=1}^n \mathbf{1}\left(\bx_{\ell}^{(k)} \in A\left(\bx_i^{(k)}\right)\right)}=\frac{1}{N_L^{(k)}} \quad \text { when } \bx_j^{(k)} \text { shares leaf } L \text { with } \bx_i^{(k)} .
		$$

		Thus $\left(\mathbf{b}_n^{(k)} B^{(k)} \mathbf{b}_n^{(k) \top}\right)_{i, j}=\left(\bS_{n}^{(k)}\right)_{i,j}
		$ in all cases.
		
		The proof for the structure vector follows analogously. 
	\end{proof}

	\subsection{Computing the Kernel Vector in Bin Space}
	
	We are interested in computing the expectations
	
	$$
	\kveck_n(\bx)^ := \E\left[\bb^{(k)}(\bx)^\top\bD^{(k)} {\bb_n^{(k)}}^\top\right]\in \mathbb{R}^{1\times n}, \quad \Kmatk_n := \E[\Smatk_n] = \mathbb{E}\left[\bb_n^{(k)}\bB^{(k)}{\bb_n^{(k)}}^\top\right], \quad \mK_{\text{bin}} = \sum_{k=1}^{p}\E\left[\bb^{(k)}_n\bB^{(k)}{\bb^{(k)}_n}^\top\right],
	$$
	
	and the norm of the weighting vector
	
	$$
	\rveck(\bx)^\top:=  \E[{\bb_n^{(k)}}\bD^{(k)} \bb^{(k)}(\bx)]^\top\jmat[\bI + \jmat\mK_{\text{bin}}]^{-1},
	$$
	where $\jmat$ is the centering matrix $\jmat = \bI_n - \frac{1}{n}\bone_n \bone_n^\top $.
	
	\paragraph{Structure vectors and matrices. }
    \begin{lemma}[Norm of the weight vector in bin-space]
    Write $\bh^{(k)}(\bx) = \E[{\bD^{(k)}}^2 \bb^{(k)}(\bx)] \in \R^m$, $\bH = \sum_{k=1}^p \mathbb{E}\left[\bD^{(k)}\bB^{(k)}\bD^{(k)}\right]$. The weight vector $\rveck_n(\bx)$ can be written as:
	$$\rveck_n(\bx) = \left[\bI_n + \jmat \bZ_n^{(k)} \bH {\bZ_n^{(k)}}^\top\right]^{-1}\jmat \kveck_n(\bx).$$
    \end{lemma}
    \begin{proof}
        We first turn our attention to the structure matrix. First note that as the expectation is taken over the ensemble, and the bin assignments are the same for every member of the ensemble, we can pull $\bZ_n^{(k)}$ out of the expectation. We then have
	$$\Kmatk_n = \E[\Smatk_n] = \mathbb{E}\left[\bb_n^{(k)}\bB^{(k)}{\bb_n^{(k)}}^\top\right] = \bZ_n^{(k)} \mathbb{E}\left[\bD^{(k)}\bB^{(k)}\bD^{(k)}\right]{\bZ_n^{(k)}}^\top.$$
	
	The overall kernel matrix then can be written as:
	$$\mK_{\text{bin}} = \sum_{k=1}^{p}\E\left[\bb^{(k)}_n\bB^{(k)}{\bb^{(k)}_n}^\top\right] = \bZ_n^{(k)} \left(\sum_{k=1}^p \mathbb{E}\left[\bD^{(k)}\bB^{(k)}\bD^{(k)}\right]\right){\bZ_n^{(k)}}^\top$$
	
	The weight vector can then be written as
	\begin{align*}
		\rveck_n(\bx)^\top
		&:= \E[{\bb_n^{(k)}}\bD^{(k)}  \bb^{(k)}(\bx)]^\top\jmat[\bI + \jmat\mK_{\text{bin}}]^{-1} \\
		&= \left(\bZ_n^{(k)}\E[{\bD^{(k)}}^2 \bb^{(k)}(\bx)]\right)^\top\jmat\left[\bI + \jmat \bZ_n^{(k)} \left(\sum_{k=1}^p \mathbb{E}\left[\bD^{(k)}\bB^{(k)}\bD^{(k)}\right]\right){\bZ_n^{(k)}}^\top\right]^{-1}
	\end{align*}
	Now write $\bh^{(k)}(\bx) = \E[{\bD^{(k)}}^2 \bb^{(k)}(\bx)] \in \R^m$, $\bH = \sum_{k=1}^p \mathbb{E}\left[\bD^{(k)}\bB^{(k)}\bD^{(k)}\right]$, and note that we can express $\kveck_n(\bx) = \bZ_n^{(k)} \bh^{(k)}(\bx) \in \R^n$. The weight vector is then:
	$$\rveck_n(\bx) = \left[\bI_n + \jmat \bZ_n^{(k)} \bH {\bZ_n^{(k)}}^\top\right]^{-1}\jmat \kveck_n(\bx).$$
    \end{proof}

	\paragraph{Bin–space compression (for the norm).}
    \begin{lemma}[Bin-space norm computation]
        Let \(\mathbf{c}^{(k)}:=\mathbf{Z}_n^{(k)\top}\mathbf{1}_n\in\mathbb{R}^{m_k}\) be the bin counts. Also define
        $$
    	\mathbf{M}^{(k)} := \mathbf{H}^{-1} + \operatorname{diag}(\mathbf{c^{(k)}}) - \frac{1}{n}\,\mathbf{c}^{(k)}{\mathbf{c}^{(k)}}^\top,
    	\qquad
    	\mathbf{z}^{(k)}(\mathbf{x}) := \operatorname{diag}(\mathbf{c}^{(k)})\bh^{(k)}(\bx)- \frac{{\mathbf{c}^{(k)}}^\top \bh^{(k)}(\bx)}{n} \mathbf{c}^{(k)}.
    	$$
        Setting \(\mathbf{q}^{(k)}(\mathbf{x}) := \mathbf{h}^{(k)}(\mathbf{x}) - \mathbf{w}^{(k)}(\mathbf{x})\), where $\mathbf{w}^{(k)}$ is the solution to the small $m_k \times m_k$ system of linear equations
	$$\bM^{(k)}\mathbf{w}^{(k)} = \mathbf{z}^{(k)}(\bx),$$
    allows us to express
    \begin{align*}
		\|\mathbf{r}_n^{(k)}(\mathbf{x})\|_2^2
		= \mathbf{q}^{(k)}(\mathbf{x})^\top \operatorname{diag}(\mathbf{c}^{(k)})\,\mathbf{q}^{(k)}(\mathbf{x})
		- \frac{1}{n}\,\left({\mathbf{c}^{(k)}}^\top \mathbf{q}^{(k)}(\mathbf{x})\right)^2.
	\end{align*}
    \end{lemma}
    \begin{proof}
        An application of the Woodbury formula ($(I+U C V)^{-1}=I-U\left(C^{-1}+V U\right)^{-1} V$ with $U=\jmat \bZ_n^{(k)}, C=\bH, V={\bZ_n^{(k)}}^\top$) yields:
	$$\left[\bI_n + \jmat \bZ_n^{(k)} \bH {\bZ_n^{(k)}}^\top\right]^{-1} = \bI_n - \jmat \bZ_n^{(k)}\underbrace{\left[{\bH}^{-1} + {\bZ_n^{(k)}}^\top \jmat \bZ_n^{(k)}\right]^{-1}}_{:= \bM^{(k) -1}}{\bZ_n^{(k)}}^\top.$$
    
	Let \(\mathbf{c}^{(k)}:=\mathbf{Z}_n^{(k)\top}\mathbf{1}_n\in\mathbb{R}^{m_k}\) be the bin counts. Then
	$$
	\mathbf{Z}_n^{(k)\top}\jmat\,\mathbf{Z}_n^{(k)}
	= \bZ_n^{(k)\top}\bZ_n^{(k)} - \frac{1}{n} \mathbf{Z}_n^{(k)\top}\mathbf{1}_n\mathbf{1}_n^\top \bZ_n^{(k)} = \operatorname{diag}(\mathbf{c}^{(k)}) - \frac{1}{n}\,\mathbf{c}^{(k)}{\mathbf{c}^{(k)}}^\top.
	$$
	
	So defining $$\mathbf{M}^{(k)} := \mathbf{H}^{-1} + \operatorname{diag}(\mathbf{c}^{(k)}) - \frac{1}{n}\,\mathbf{c}^{(k)}{\mathbf{c}^{(k)}}^\top,$$
	we can express 
	$$\rveck_n(\bx) = \left[\bI_n - \jmat \bZ_n^{(k)} {\bM^{(k)}}^{-1}{\bZ_n^{(k)}}^\top\right]\jmat \kveck_n(\bx) = \jmat \kveck_n(\bx) - \jmat \bZ_n^{(k)} {\bM^{(k)}}^{-1}{\bZ_n^{(k)}}^\top \jmat \kveck_n(\bx).$$
	
	It now falls to us to note that the quantity on the very right of the above expression can be expressed as:
	\begin{align*}
		\mathbf{z}^{(k)}(\bx) 
		&:= {\bZ_n^{(k)}}^\top \jmat \kveck_n(\bx) \\
		&= {\bZ_n^{(k)}}^\top \jmat \bZ_n^{(k)} \bh^{(k)}(\bx)\\
		&= \left[\operatorname{diag}(\mathbf{c}^{(k)}) - \frac{1}{n}\,\mathbf{c}^{(k)}{\mathbf{c}^{(k)}}^\top\right]\bh^{(k)}(\bx) \\
		&= \operatorname{diag}(\mathbf{c}^{(k)})\bh^{(k)}(\bx)- \frac{{\mathbf{c}^{(k)}}^\top \bh^{(k)}(\bx)}{n} \mathbf{c}^{(k)}.
	\end{align*}
	
	To recall, we now have, for
	$$
	\mathbf{M}^{(k)} := \mathbf{H}^{-1} + \operatorname{diag}(\mathbf{c^{(k)}}) - \frac{1}{n}\,\mathbf{c}^{(k)}{\mathbf{c}^{(k)}}^\top,
	\qquad
	\mathbf{z}^{(k)}(\mathbf{x}) := \operatorname{diag}(\mathbf{c}^{(k)})\bh^{(k)}(\bx)- \frac{{\mathbf{c}^{(k)}}^\top \bh^{(k)}(\bx)}{n} \mathbf{c}^{(k)},
	$$
	we can express
	$$\rveck_n(\bx) = \jmat \bZ_n^{(k)} \bh^{(k)}(\bx) - \jmat \bZ_n^{(k)} {\bM^{(k)}}^{-1}\mathbf{z}^{(k)}(\mathbf{x}).$$
	
	Now, we can solve the small $m_k \times m_k$ system of linear equations for $\mathbf{w}^{(k)}$ given by
	$$\bM^{(k)}\mathbf{w}^{(k)} = \mathbf{z}^{(k)}(\bx),$$
	to obtain the solution $\mathbf{w}^{(k)} := {\bM^{(k)}}^{-1}\mathbf{z}^{(k)}(\bx)$. We can then set \(\mathbf{q}^{(k)}(\mathbf{x}) := \mathbf{h}^{(k)}(\mathbf{x}) - \mathbf{w}^{(k)}(\mathbf{x})\), and observe that
	$$
	\mathbf{r}_n^{(k)}(\mathbf{x}) = \jmat\,\mathbf{Z}_n^{(k)} \mathbf{q}^{(k)}(\mathbf{x}).
	$$
	To find its norm we simply compute:
	\begin{align*}
		\|\mathbf{r}_n^{(k)}(\mathbf{x})\|_2^2
		&= {\mathbf{r}_n^{(k)}}^\top \mathbf{r}_n^{(k)} \\
		&= \mathbf{q}^{(k)}(\mathbf{x})^\top  {\mathbf{Z}_n^{(k)}}^\top  \jmat\,\jmat\,\mathbf{Z}_n^{(k)} \mathbf{q}^{(k)}(\mathbf{x}) \\
		&= \mathbf{q}^{(k)}(\mathbf{x})^\top  {\mathbf{Z}_n^{(k)}}^\top  \jmat\,\mathbf{Z}_n^{(k)} \mathbf{q}^{(k)}(\mathbf{x}) \\
		&= \mathbf{q}^{(k)}(\mathbf{x})^\top\left[\operatorname{diag}(\mathbf{c}^{(k)}) - \frac{1}{n}\,\mathbf{c}^{(k)}{\mathbf{c}^{(k)}}^\top\right]\mathbf{q}^{(k)}(\mathbf{x})  \\
		&= \mathbf{q}^{(k)}(\mathbf{x})^\top \operatorname{diag}(\mathbf{c}^{(k)})\,\mathbf{q}^{(k)}(\mathbf{x})
		- \frac{1}{n}\,\left({\mathbf{c}^{(k)}}^\top \mathbf{q}^{(k)}(\mathbf{x})\right)^2.
	\end{align*}
    \end{proof}

	\paragraph{The final algorithm.}
	
	\begin{enumerate}
		\item During training, construct and cache diagonal matrix $\mathbf{D}^{(k)}=\operatorname{diag}\left(\xi_1^{(k)}, \ldots, \xi_{m_k}^{(k)}\right)$, with
		$$\xi_r^{(k)}:=\sqrt{\frac{\left|\mathcal{L}^{(k)}(r)\right|}{N_{\mathcal{L}(r)}^{(k)}}}=\frac{\sum_{l=1}^{m_k} \mathbf{1}\left(B_l^{(k)} \in A\left(B_r^{(k)}\right)\right)}{\sum_{i=1}^n \mathbf{1}\left(\mathbf{x}_i^{(k)} \in A\left(B_r^{(k)}\right)\right)}=\sqrt{\frac{\text { Number of bins in leaf containing bin } r}{\text { Number of training samples in leaf containing bin } r}}.$$
		\item During training, construct and cache bin-level structure matrix and transform
		$$\mathbf{B}_{i, j}^{(k)}=\left[\frac{\mathbf{1}\left(B_j^{(k)} \in A\left(B_i^{(k)}\right)\right)}{\sum_{l=1}^{m_k} \mathbf{1}\left(B_l^{(k)} \in A\left(B_i^{(k)}\right)\right)}\right]_{m_k \times m_k}, \quad \mathbf{b}_n^{(k)}=\mathbf{Z}_n^{(k)}{\mathbf{D}^{(k)}}^2.$$
		\item During training, construct and cache an estimate of $\mathbf{H}=\sum_{k=1}^p \mathbb{E}\left[\mathbf{D}^{(k)} \mathbf{B}^{(k)} \mathbf{D}^{(k)}\right]$ over the ensemble. 
		\item During training, construct and cache bin-counts $\mathbf{c}^{(k)} := {\bZ_n^{(k)}}^\top \bone_n$.
		\item Post-training, construct and cache $\mathbf{M}^{(k)}:=\mathbf{H}^{-1}+\operatorname{diag}\left(\mathbf{c}^{(k)}\right)-\frac{1}{n} \mathbf{c}^{(k)} \mathbf{c}^{(k)^{\top}}$.
		\item On the fly, index into the row/column of $\bH$ corresponding the bin that test point $\bx$ is in to obtain $\mathbf{h}^{(k)}(\mathbf{x})$.
		\item On the fly, compute $\mathbf{z}^{(k)}(\mathbf{x}):=\operatorname{diag}\left(\mathbf{c}^{(k)}\right) \mathbf{h}^{(k)}(\mathbf{x})-\frac{\mathbf{c}^{(k)^{\top}} \mathbf{h}^{(k)}(\mathbf{x})}{n} \mathbf{c}^{(k)}$.
		\item On the fly, solve $\bM^{(k)}\mathbf{w}^{(k)} = \mathbf{z}^{(k)}(\bx)$ for $\mathbf{w}^{(k)}$.
		\item On the fly, set $\mathbf{q}^{(k)}(\mathbf{x}) := \mathbf{h}^{(k)}(\mathbf{x}) - \mathbf{w}^{(k)}(\mathbf{x})$.
		\item On the fly, compute 
		$$\|\mathbf{r}_n^{(k)}(\mathbf{x})\|_2^2 = \mathbf{q}^{(k)}(\mathbf{x})^\top \operatorname{diag}(\mathbf{c}^{(k)})\,\mathbf{q}^{(k)}(\mathbf{x})
		- \frac{1}{n}\,\left({\mathbf{c}^{(k)}}^\top \mathbf{q}^{(k)}(\mathbf{x})\right)^2.$$
		\item Sum over all features and take the square root:
		$\|\mathbf{r}_n(\mathbf{x})\|_2 = \sqrt{\sum_{k=1}^p\|\mathbf{r}_n^{(k)}(\mathbf{x})\|_2^2}$.
	\end{enumerate}
	
	\subsection{What we lose by binning.}
    \label{append:what_we_lose_by_binning}
    Binning induces a discretization: all points within the same histogram bin are treated as indistinguishable along that feature.
    This can (i) introduce an approximation bias at sub-bin scales, and (ii) prevent recovery of within-bin variation even when $n$ is large.
    However, our bin-space computations do not introduce an additional approximation beyond the estimator we analyze. The reasoning is as follow:
    the practical EBM implementation we build on already trains \emph{histogram trees}, so the learned model is itself defined on bins,
    and our inference targets the sampling variability of this histogram-EBM estimator.
    
    If one instead wants inference for a hypothetical \emph{unbinned} (continuous) version of the model, an additional discretization error must be accounted for.
    A simple bound can be stated under the smoothness assumption in Assumption~\ref{aspt:truth}.
    Suppose the bin width along each feature is $O(1/m)$.
    Then an application of the mean value inequality yields a uniform approximation error
    $\|\hat f - \hat f^{(\mathrm{bin})}\|_\infty = O(1/m)$ (assuming the relevant derivatives/densities are bounded),
    so the accumulated increase in MSE scales as $O(p/m)$.
    In particular, as the bin resolution increases (e.g., $m \to \infty$ with $m \gg n^{1/3}$), this discretization effect becomes negligible,
    recovering the same MSE order as in the unbinned idealization.

	\section{Finite Sample Convergence}
	
	\label{sec:finite-sample-convergence}
	
	\subsection{Structural assumptions and identities for finite-sample convergence}
	\label{subsec:structural-identities}
	
	Throughout the finite-sample convergence analysis we will use a common set of structural assumptions and consequences. Recall $\jmat:=\bI_n-\frac{1}{n}\bone\bone^\top$, $\Kmatk:=\E[\Smatk]$, and $\bK:=\sum_{k=1}^p\Kmatk$.
	For each feature $k$, define the additive (centered) subspace
	\[
	\mathcal H_k \;:=\; \big\{\,\bv\in\R^n:\ \exists\,g:\R\!\to\!\R\text{ s.t. } v_i=g(x_i^{(k)}),\ \bone^\top \bv=0\,\big\},
	\qquad
	\mathcal S \;:=\; \mathrm{span}(\mathcal H_1,\ldots,\mathcal H_p)\subseteq\{\bone\}^\perp,
	\]
	and let $P_k$ denote the orthogonal projector onto $\mathcal H_k$. One can show that under Assumptions \ref{aspt:proj-orth}, the aggregated kernel $\bK$ is asymptotically identity.
	
	\begin{lemma}[Centered+ones decomposition and identity on $\mathcal S$]
		\label{lem:K-decomp-centered-plus-ones}
		Under Assumptions \ref{aspt:proj-orth}, there exist $A_n$ and $\bE_n$ with
		\begin{equation}
			\label{eq:K-decomp}
			\bK \;=\; A_n \;+\; \frac{p}{n}\,\bone\bone^\top \;+\; \bE_n,
			\qquad
			\jmat A_n \jmat = A_n,
			\qquad
			\|\,\bE_n\,\|_{\mathrm{op}} \;\le\; C\!\left(\sum_{k=1}^p \varepsilon_{n,k} + \frac{p}{n}\right),
		\end{equation}
		and, restricted to $\mathcal S$,
		\begin{equation}
			\label{eq:identity-on-S}
			\|\, A_n - \bI \,\|_{\mathcal S\to\mathcal S} \;\le\; \sum_{k=1}^p \varepsilon_{n,k}
			\;\xrightarrow[n\to\infty]{}\; 0.
		\end{equation}
		Consequently,
		\[
		\jmat \bK \jmat \;=\; A_n + \jmat \bE_n \jmat,
		\qquad
		\big\|\, \jmat \bK \jmat - \bI \,\big\|_{\mathcal S\to\mathcal S} \;\le\; C\!\left(\sum_{k=1}^p \varepsilon_{n,k} + \frac{p}{n}\right)=:\eta_n
		\ \xrightarrow[n\to\infty]{}\ 0.
		\]
	\end{lemma}
	
	\begin{proof}
		For any $M$, $M=\jmat M \jmat + \frac{1}{n}\bone\bone^\top + R(M)$ with $\|R(M)\|_{\mathrm{op}}\le c/n$ (the remainder absorbs the $O(1/n)$ column-sum deviation). Apply to $\Kmatk$ and sum over $k$ to get \eqref{eq:K-decomp} with $A_n:=\sum_k \jmat \Kmatk \jmat$ and $\bE_n:=\sum_k R(\Kmatk)$. By Assumption \ref{aspt:proj-orth}, $A_n=\sum_k P_k + \sum_k(\jmat\Kmatk\jmat-P_k)$, so on $\mathcal S$ we have $\sum_k P_k=\bI$ (Assumption~\ref{aspt:proj-orth}) and \eqref{eq:identity-on-S} follows. Left/right multiplying by $\jmat$ kills the rank-one term.
	\end{proof}
	
	Lemma~\ref{lem:K-decomp-centered-plus-ones} supplies the \emph{identity-on-$\mathcal S$} bound
	\[
	\|\,\jmat \bK \jmat - \bI\,\|_{\mathcal S\to\mathcal S}\le \eta_n \to 0,
	\]
	which we plug into the mean contraction step for Algorithm~\ref{alg:brebm_ident}. 
	
	\subsection{A Fixed-point Iteration}
	
	To investigate the finite sample convergence behavior of the algorithm, we begin with a set of postulated fixed point. For each subensemble's prediction on the training set, assume the existence of a fixed point $\widetilde{\mathbf{y}}_k^*$. Then the Inferable EBM update rule leads to the following fixed point iteration:
	
	$$
	\widetilde{\mathbf{y}}_k^* = \frac{1}{b+1}\Big((\bI-\frac{1}{n}\mathbf{1}\mathbf{1}^\top)\Smatk(\by-\widehat{\by})  + b \widetilde{\mathbf{y}}_k^*\Big)
	$$
	
	Rearranging the equation gives the following identity
	
	\begin{equation}
		\label{eqn:eqn1}
		\widetilde{\mathbf{y}}_k^* = (\bI-\frac{1}{n}\mathbf{1}\mathbf{1}^\top) \Smatk(\mathbf{y}-\widehat{\mathbf{y}})
	\end{equation}
	
	Since we do not know $\widehat{\mathbf{y}}$, we will solve for it and plug it back later. Sum the equations over $k$ yields:
	
	\begin{equation*}
		\widehat{\mathbf{y}}:=\widehat{\bbeta}^*+ \sum_{k=1}^p\widetilde{\mathbf{y}}_k^* = \widehat{\bbeta}^*+ (\bI-\frac{1}{n}\mathbf{1}\mathbf{1}^\top)(\sum_{k=1}^{p}\Smatk)(\mathbf{y}-\widehat{\mathbf{y}})
	\end{equation*}
	
	And by the design of identifiability constraint, at fixed points, all residuals should have mean zero:
	
	$$
	\frac{1}{n}\bone^\top(\by - \widehat{\by}) = 0 \rightarrow \frac{1}{n}\bone^\top(\by - \widehat{\bbeta}^* - \sum_{k=1}^{p}\f^{(k)}) = 0 \rightarrow \widehat{\bbeta}^* = \bar{\by} = \frac{1}{n}\bone\bone^\top\by
	$$
	
	Plugging this back we have 
	
	\begin{equation}
		\widehat{\mathbf{y}} = \Big[\bI + (\bI - \frac{1}{n}\bone\bone^\top)(\sum_k\Smatk)\Big]^{-1}\Big[(\bI-\frac{1}{n}\bone\bone^\top)(\sum_k\Smatk) + \frac{1}{n}\bone\bone^\top\Big]\by
	\end{equation}
	
	Plugging it back to equation \ref{eqn:eqn1} produces the fixed-points for each subensemble:
	
	\begin{equation}
		\label{eqn:eqn3}
		\widetilde{\by}^*_k =(\bI - \frac{1}{n}\bone\bone^\top)\Smatk\bigg[\bI - \big[\bI + (\bI - \frac{1}{n}\bone\bone^\top)(\sum_k\Smatk)\big]^{-1}\big[(\bI-\frac{1}{n}\bone\bone^\top)(\sum_k\Smatk)+\frac{1}{n}\bone\bone^\top\big]\bigg]\by
	\end{equation}
	
	The intercept term is canceled out following the algebra below.
	
	First, notice that $\Big(\bI + (\bI -\frac{1}{n}\bone\bone^\top)(\sum_{k=1}^{p}\Smatk)\Big)(\frac{1}{n}\bone\bone^\top)= (\frac{1}{n}\bone\bone^\top)$, since the row sum of $\Smatk$ is always one and thus eliminated by the centering operator. 
	
	Multiplying both sides by $\Big(\bI + (\bI -\frac{1}{n}\bone\bone^\top)(\sum_{k=1}^{p}\Smatk)\Big)^{-1}$ implies $\Big(\bI + (\bI -\frac{1}{n}\bone\bone^\top)(\sum_{k=1}^{p}\Smatk)\Big)^{-1}(\frac{1}{n}\bone\bone^\top)= (\frac{1}{n}\bone\bone^\top)$. This, once again, is killed by the composition $(\bI - \frac{1}{n}\bone\bone^\top)\Smatk$. Hence the intercept would not matter in the final fixed point derivation.
	
	Assuming the invertibility of $\big[\bI + (\bI - \frac{1}{n}\bone\bone^\top)(\sum_k\Smatk)\big]$, the resolvent identity yields
	
	\begin{equation}
		\label{eqn:eqn4}
		\widetilde{\by}^*_k =(\bI - \frac{1}{n}\bone\bone^\top)\Smatk\big[\bI + (\bI - \frac{1}{n}\bone\bone^\top)(\sum_k\Smatk)\big]^{-1}\by : = \jmat\Smatk[\bI + \jmat\bK]^{-1}\by
	\end{equation}
	
	The corresponding pointwise prediction for $\forall \bx \in \mathbb{R}^p$ is 
	
	$$
	\widehat{\f}^{(k)}(\bx) = \sveck(\bx)\jmat[\bI + \jmat\bK]^{-1}\by
	$$

	\subsection{Another Fixed-point Iteration}
	
	Here is a fixed point for the leave-one-out algorithm. For each subensemble's prediction on the training set, assume the existence of a fixed point $\widetilde{\mathbf{y}}_k^*$. Then the update rule leads to the following fixed point iteration:
	
	$$
	\widetilde{\mathbf{y}}_k^* = \frac{1}{b+1}\Big((\bI-\frac{1}{n}\mathbf{1}\mathbf{1}^\top)\E\Smatk(\by-\sum_{a\neq k}\widetilde{\mathbf{y}}_a^*)  + b \widetilde{\mathbf{y}}_k^*\Big)
	$$
	
	Rearranging the equation gives the following identity
	
	\begin{equation}
		\label{eqn:eqn1}
		\widetilde{\mathbf{y}}_k^* = (\bI-\Smatk)^\dagger(\bI-\frac{1}{n}\mathbf{1}\mathbf{1}^\top) \Smatk(\mathbf{y}-\widehat{\mathbf{y}})
	\end{equation}
	
	Since we do not know $\widehat{\mathbf{y}}$, we will solve for it and plug it back later. Sum the equations over $k$ yields:
	
	\begin{equation*}
		\widehat{\mathbf{y}}:=\sum_{k=1}^p\widetilde{\mathbf{y}}_k^* = \sum_{k=1}^p(\bI-\Smatk)^{-1}(\bI-\frac{1}{n}\mathbf{1}\mathbf{1}^\top)\Smatk(\mathbf{y}-\widehat{\mathbf{y}})
	\end{equation*}
	
	which represents our ensemble prediction as a kernel ridge regression, with the group-level kernel being a sum of feature-wise kernel ridge operator (i.e. $\bS_n = \sum_{k=1}^p(\bI-\Smatk)^{-1} \Smatk$)
	
	$$
	\widehat{\mathbf{y}} = \Big(\bI+\sum_{k=1}^p(\bI-S_n^{(k)})^{-1}(\bI-\frac{1}{n}\mathbf{1}\mathbf{1}^\top) \Smatk\Big)^{-1}\Big(\sum_{k=1}^p(\bI-\Smatk)^{-1}(\bI-\frac{1}{n}\mathbf{1}\mathbf{1}^\top) \Smatk\Big)\mathbf{y}
	$$
	
	Plugging it back to equation \ref{eqn:eqn4} produces the fixed-points for each subensemble:
	
	\begin{align*}
		\label{eqn:eqn2}
		\widetilde{\mathbf{y}}_k^* &= (\bI-\Smatk)^{-1} (\bI-\frac{1}{n}\mathbf{1}\mathbf{1}^\top\Smatk\Big[\bI - \Big(\bI+\sum_{k=1}^p(\bI-\Smatk)^{-1}(\bI-\frac{1}{n}\mathbf{1}\mathbf{1}^\top) \Smatk\Big)^{-1}\Big(\sum_{k=1}^p(\bI-\Smatk)^{-1} (\bI-\frac{1}{n}\mathbf{1}\mathbf{1}^\top)\Smatk\Big)\Big]\mathbf{y}\\
		&:=(\bI-\Smatk)^{-1} (\bI-\frac{1}{n}\mathbf{1}\mathbf{1}^\top)\Smatk R\mathbf{y}
	\end{align*}
	\subsection{Contraction} \label{app:contraction}
	
	\subsubsection{Algorithm \ref{alg:brebm_ident}}
	
	For the vanilla Inferable EBM, we make use of the fixed point iteration conjecture and prove that all feature functions do converge to the set of fixed points. Formally, we state the theorem below:
	
	\begin{theorem}
		Define $\widehat{\mathbf{y}}_{b}:=\widehat{\bbeta}_{b}+\sum_{a=1}^{p}\widehat{\mathbf{y}}^{(a)}_{b},
		\jmat=\bI-\tfrac{1}{n}\bone\bone^\top.$Let the per-feature Boulevard averages of the new algorithm be
		$$
		\widehat{\mathbf{y}}^{(k)}_{b+1}
		\;=\;
		\frac{b}{b+1}\,\widehat{\mathbf{y}}^{(k)}_{b}
		\;+\;\frac{\lambda}{b+1}\,\jmat\,\mathbf{S}^{(b+1,k)}
		\Big(\mathbf{y}-\widehat{\mathbf{y}}_{b}\Big),
		\qquad
		$$
		With Assumption \ref{aspt:integrity}~ \ref{aspt:proj-orth} hold,
		Let the postulated fixed points be
		\[
		\widetilde{\mathbf{y}}_k^* \;=\; \jmat\,\Kmatk\,[\lambda\bI+\jmat\,\bK]^{-1}\mathbf{y},
		\qquad
		\bK\;:=\;\sum_{k=1}^{p}\Kmatk,
        \]
		and $\widehat{\mathbf{y}}^*\;=\;\widehat{\bbeta}^*\bone+\sum_{k=1}^{p}\widetilde{\mathbf{y}}_k^*$,
		with $\widehat{\bbeta}^*=\bar{\mathbf{y}}$. Then, almost surely, for each $k$,
		
        $$
        \widehat{\mathbf{y}}^{(k)}_{b}\;\longrightarrow\; \widetilde{\mathbf{y}}^{(k)\,*},
		\qquad
		\widehat{\bbeta}_{b}\;\longrightarrow\;\bar{\mathbf{y}}=\tfrac{1}{n}\bone^\top\mathbf{y},
		$$
		and hence $\widehat{\mathbf{y}}_{b}\to \widehat{\mathbf{y}}^*$ almost surely.
	\end{theorem}
	
	\begin{proof}
		Since $\Smatk\bone=\bone$ and $\jmat\bone=0$, we have $\jmat\Smatk=\jmat\Smatk\jmat$.
		Thus inserting or omitting $\widehat{\bbeta}_b\bone$ in the residual does not affect
		updates; we drop the intercept in what follows.
		
		Fix the deterministic fixed points $\{\widetilde{\mathbf{y}}_k^*\}$ above and set
		$\mathbf{D}^{(k)}_b:=\widehat{\mathbf{y}}^{(k)}_{b}-\widetilde{\mathbf{y}}_k^*$,
		$\mathbf{D}_b:=[(\mathbf{D}^{(1)}_b)^\top,\ldots,(\mathbf{D}^{(p)}_b)^\top]^\top$.
		
		\begin{align*}
			\mathbb{E}[\mathbf{D}_{b+1}^{(k)}\mid \mathcal{F}_b]
			&= \mathbb{E}[\widehat{\mathbf{y}}_{b+1}^{(k)} - \widetilde{\mathbf{y}}_k^*\mid \mathcal{F}_b]\\
			&=\mathbb{E}\Big[\frac{b}{b+1}\widehat{\mathbf{y}}_{b}^{(k)} + \frac{1}{b+1}\big(\bI-\tfrac{1}{n}\mathbf{1}\mathbf{1}^\top\big)\mathbf{S}_{n}^{b+1, (k)}\!\Big(\mathbf{y} - \Gamma_M\Big(\sum_{a=1}^p \widehat{\mathbf{y}}_{b}^{(a)}\Big)\Big) - \widetilde{\mathbf{y}}_k^* \,\Big|\, \mathcal{F}_b\Big]\\
			&=\frac{b}{b+1}\widehat{\mathbf{y}}_{b}^{(k)}+\frac{1}{b+1}\big(\bI-\tfrac{1}{n}\mathbf{1}\mathbf{1}^\top\big)\,\mathbb{E}[\mathbf{S}_{n}^{b+1, (k)}\mid \mathcal{F}_b]\!\Big(\mathbf{y} - \Gamma_M\Big(\sum_{a=1}^p \widehat{\mathbf{y}}_{b}^{(a)}\Big)\Big) - \widetilde{\mathbf{y}}_k^*\\
			&= \frac{b}{b+1}\mathbf{D}_{b}^{(k)}
			+ \frac{1}{b+1}\Big(\big(\bI-\tfrac{1}{n}\mathbf{1}\mathbf{1}^\top\big)\big(\mathbb{E}[\mathbf{S}_n^{b+1,(k)} \mid \mathcal{F}_b]-\mathbb{E}[\mathbf{S}_n^{(\infty,k)}]\big)\big(\mathbf{y}-\sum_{a=1}^p\widetilde{\mathbf{y}}_a^*\big)\\
			&\qquad\qquad\qquad\qquad\quad
			+ \big(\bI-\tfrac{1}{n}\mathbf{1}\mathbf{1}^\top\big)\mathbb{E}[\mathbf{S}_n^{b+1,(k)} \mid \mathcal{F}_b]\Big(\sum_{a=1}^p\widetilde{\mathbf{y}}_a^*-\Gamma_M\Big(\sum_{a=1}^p\widehat{\mathbf{y}}_{b}^{(a)}\Big)\Big)\Big)\\
			&= \frac{b}{b+1}\mathbf{D}_{b}^{(k)}
			+ \frac{1}{b+1}\big(\bI-\tfrac{1}{n}\mathbf{1}\mathbf{1}^\top\big)\E\Smatk\,\Delta_b,
		\end{align*}
		$$
		\Delta_b \;:=\; \sum_{a=1}^p\widetilde{\mathbf{y}}_a^*-\Gamma_M\Big(\sum_{a=1}^p\widehat{\mathbf{y}}_{b}^{(a)}\Big),
		\qquad
		\|\Delta_b\| \;\le\; \Big\|\sum_{a=1}^p \mathbf{D}_b^{(a)}\Big\|.
		$$
		
		Now we stack the coordinates. The associated $np\times np$ block operator is
		\[
		\mathbf L \;=\; \jmat\otimes \mathbf S,\qquad
		(\mathbf S\mathbf D)_k \;:=\; \big(\jmat\,\Kmat^{(k)}\,\jmat\big)\!\left(\sum_{a=1}^p \mathbf D^{(a)}\right),\quad k=1,\dots,p,
		\]
		so that
		\[
		\mathbb{E}[\mathbf D_{b+1}\mid\mathcal F_b]
		\;=\; \Big(\frac{b}{b+1}\mathbf I_{np} + \frac{1}{b+1}\mathbf L\Big)\mathbf D_b.
		\]
		Since $\|\jmat\|_{\mathrm{op}}=1$, it suffices to bound $\|\mathbf S\|_{\mathrm{op}}$.
		
		By Assumption~\ref{aspt:proj-orth}, for each $k$ we have
		\(J\,\Kmat^{(k)}\,J = P_k + E_k\) with \(\|E_k\|_{\mathrm{op}}\le \varepsilon_{n,k}\to0\),
		and \(P_aP_k=0\) for \(a\neq k\).
		Fix any centered block vector \(\mathbf D=(\mathbf D^{(1)},\ldots,\mathbf D^{(p)})\) and set
		\(S:=\sum_{a=1}^p \mathbf D^{(a)}\).
		Then
		\[
		\|\mathbf S \mathbf D\|_2^2
		= \sum_{k=1}^p \|(P_k+E_k)S\|_2^2
		\;\le\; \sum_{k=1}^p \big(\|P_k S\|_2 + \|E_k\|\,\|S\|_2\big)^2
		\;\le\; \sum_{k=1}^p \|P_k S\|_2^2 \;+\; 2\varepsilon_n \|S\|_2 \sum_{k=1}^p \|P_k S\|_2 \;+\; p\varepsilon_n^2 \|S\|_2^2,
		\]
		where \(\varepsilon_n:=\max_k \varepsilon_{n,k}\).
		By orthogonality of the ranges of the \(P_k\)’s,
		\(\sum_{k=1}^p \|P_k S\|_2^2 = \|\sum_{k=1}^p P_k S\|_2^2 = \|S\|_2^2\),
		and \(\sum_{k=1}^p \|P_k S\|_2 \le \sqrt{p}\,\|S\|_2\).
		Hence
		\[
		\|\mathbf S \mathbf D\|_2^2
		\;\le\; \big(1 + 2\sqrt{p}\,\varepsilon_n + p\varepsilon_n^2\big)\,\|S\|_2^2
		\;\le\; \big(1 + c_p\,\varepsilon_n\big)\,\Big\|\sum_{a=1}^p \mathbf D^{(a)}\Big\|_2^2,
		\]
		for a constant \(c_p=2\sqrt{p}+p\varepsilon_n\).
		In particular, for all sufficiently large \(n\),
		\(\|\mathbf S\|_{\mathrm{op}} \le 1 - \delta_n\) with \(\delta_n := 1 - \sqrt{1+c_p\,\varepsilon_n}>0\) and \(\delta_n\to 0\).
		Therefore
		\[
		\|\mathbf L\|_{\mathrm{op}} \;=\; \|J\|_{\mathrm{op}}\,\|\mathbf S\|_{\mathrm{op}}
		\;\le\; 1 - \delta_n \;<\; 1,
		\]
		and the mean update is a contraction:
		\[
		\Big\|\mathbb{E}[\mathbf D_{b+1}\mid\mathcal F_b]\Big\|
		\;\le\; \frac{b+\|\mathbf L\|_{\mathrm{op}}}{b+1}\,\|\mathbf D_b\|
		\;\le\; \Big(1 - \frac{\delta_n}{b+1}\Big)\|\mathbf D_b\|.
		\]
		
		For the deviation condition, if leaf predictions are bounded by $M$, then
		$\|\epsilon_{b+1}\| \le \tfrac{2M\sqrt{p}}{b+1}$, so
		$\sup_b \|\epsilon_b\|\to 0$ and $\sum_b \mathbb{E}[\|\epsilon_b\|^2]<\infty$. 
		This verifies the stochastic contraction mapping conditions.
		
		At the fixed point $\frac{1}{n}\bone^\top\widetilde{\mathbf{y}}^{(k)\,*}=0$ for all $k$, so
		$\frac{1}{n}\bone^\top \widehat{\mathbf{y}}^*=\widehat{\bbeta}^*$.
		The residual mean at convergence is $0$, hence $\widehat{\bbeta}^*=\bar{\mathbf{y}}$.
		The online update for $\widehat{\bbeta}_b$ therefore converges almost surely to $\bar{\mathbf{y}}$.
	\end{proof}
	
	\begin{corollary}
		The feature-wise fixed point prediction:
		
		$$
		\widehat{\mathbf{f}}_A^{(k)}(\mathbf{x}) = \E[\sveck(\bx)]^\top \jmat[\bI + \jmat\bK]^{-1}\by
		$$
		
		and the ensemble prediction
		
		$$
		\widehat{\mathbf{f}}_A(\mathbf{x}):=\bar{\by}+ \sum_{k=1}^{p}\E[\sveck(\bx)]^\top \jmat[\bI + \jmat\bK]^{-1}\by
		$$
	\end{corollary}
	
	\subsubsection{Algorithm \ref{alg:ebm_bratp}}
	
	Under stated assumptions, we can prove that Algorithm \ref{alg:brebm_ident}'s prediction converge almost surely to a postulated fixed point. Formally, we make use of Theorem \ref{thm:stochastic-contraction-mapping}, called a stochastic contraction mapping theorem, to demonstrate the difference between the current prediction and the fixed point is a stochastic contraction mapping. 
	
	\begin{theorem}
		\label{thm:finite-sample-convergence-2}
		Denote the prediction using feature $k$ at boosting round $b$ as $\widehat{\mathbf{y}}^{(k)}_b = \frac{1}{b}\sum_{s=1}^bt^{(s, k)}(\mathbf{x}^{(k)})$.
		Define the aggregated kernel $\mK:=\sum_{k=1}^{p}(\bI-\E\Smatk)^{-1}\jmat\E\Smatk$. Suppose for each feature, there exists a fixed point:
		
		$$
		\widetilde{\mathbf{y}}_k^* = (\bI-\E\Smatk)^{-1} \jmat\E\Smatk\Big[\bI - \Big(I+\mK\Big)^{-1}\mK\Big]\mathbf{y}
		$$
		
		which implies a global fixed point
		
		$$
		\widehat{\mathbf{y}}^* = \widehat{\bbeta}^* + \mK\Big[\bI - \Big(\bI + \mK\Big)^{-1}\mK\Big]\by
		$$
		
		Then $\widehat{\mathbf{y}}^{(k)}_b \overset{a.s.}{\to}\widetilde{\mathbf{y}}_k^*$ and $\widehat{\mathbf{y}}_b\overset{a.s.}{\to}\widehat{\mathbf{y}}^*$ under stated assumptions in \ref{sec:theory}. In the meantime, $\widehat{\bbeta}^* \overset{a.s.}{\to}\bar{\mathbf{y}}:=\frac{1}{n}\mathbf{1}^\top\mathbf{y}.$
		
	\end{theorem}
	\begin{proof}
		Denote $\mathbf{D}^{(k)}_b = \hat{\mathbf{y}}^{(k)}_b - \tilde{\mathbf{y}}_k^*$. 
		Define the stacked distance vector $\mathbf{D}_b = [\mathbf{D}_b^{(1)}, \cdots, \mathbf{D}_b^{(p)}]^\top \in \mathbb{R}^{np}$. 
		Showing this random vector is a stochastic contraction mapping suffices to establish coordinate-wise convergence.
		
		For each coordinate, the conditional mean satisfies
		
		\begin{align*}
		\mathbb{E}[\mathbf{D}_{b+1}^{(k)}\mid \mathcal{F}_b] &= \mathbb{E}[\hat{\mathbf{y}}_{b+1}^{(k)} - \tilde{\mathbf{y}}_k^*\mid \mathcal{F}_b]\\
		&=\mathbb{E}\Big[\frac{b}{b+1}\hat{\mathbf{y}}_{b}^{(k)} + \frac{1}{b+1}(\bI-\frac{1}{n}\mathbf{1}\mathbf{1}^\top)\mathbf{S}_{n}^{b+1, (k)}(\mathbf{y} - \Gamma_M\Big(\sum_{a\neq k}\hat{\mathbf{y}}_{b}^{(a)}\Big)) - \tilde{\mathbf{y}}_k^* \mid \mathcal{F}_b\Big]\\
		&=\frac{b}{b+1}\hat{\mathbf{y}}_{b}^{(k)}+\frac{1}{b+1}\mathbb(\bI-\frac{1}{n}\mathbf{1}\mathbf{1}^\top){E}[\mathbf{S}_{n}^{b+1, (k)}\mid \mathcal{F}_b](\mathbf{y} - \Gamma_M\Big(\sum_{a\neq k}\hat{\mathbf{y}}_{b}^{(a)}\Big)) - \tilde{\mathbf{y}}_k^*\\
		&= \frac{b}{b+1}\mathbf{D}_{b}^{(k)} + \frac{1}{b+1}\Big((\bI-\frac{1}{n}\mathbf{1}\mathbf{1}^\top)\mathbb{E}[\mathbf{S}_{n}^{b+1, (k)}\mid \mathcal{F}_b](\mathbf{y} - \Gamma_M\Big(\sum_{a\neq k}\hat{\mathbf{y}}_{b}^{(a)}\Big))-\tilde{\mathbf{y}}_k^*\Big)\\
		&= \frac{b}{b+1}\mathbf{D}_{b}^{(k)}
		+ \frac{1}{b+1}(\bI-\frac{1}{n}\mathbf{1}\mathbf{1}^\top)\big(\mathbb{E}[\mathbf{S}_n^{(b,k)} \mid \mathcal{F}_b]-\mathbb{E}[\mathbf{S}_n^{(\infty,k)} \mid \mathcal{F}_b]\big)(\mathbf{y}-\sum_{a\neq k}\tilde{\mathbf{y}}_a^*)\\
		&+ \frac{1}{b+1}(\bI-\frac{1}{n}\mathbf{1}\mathbf{1}^\top)\mathbb{E}[\mathbf{S}_n^{(b,k)} \mid \mathcal{F}_b](\sum_{a\neq k}\tilde{\mathbf{y}}_a^*-\Gamma_M\Big(\sum_{a\neq k}\hat{\mathbf{y}}_{b}^{(a)}\Big))\\
		&= \frac{b}{b+1}\mathbf{D}_{b}^{(k)}
		+ \frac{1}{b+1}(\bI-\frac{1}{n}\mathbf{1}\mathbf{1}^\top)\mathbb{E}[\mathbf{S}_n^{(b,k)} \mid \mathcal{F}_b](\sum_{a\neq k}\tilde{\mathbf{y}}_a^*-\Gamma_M\Big(\sum_{a\neq k}\hat{\mathbf{y}}_{b}^{(a)}\Big))\\
		& :=  \frac{b}{b+1}\mathbf{D}_{b}^{(k)}
		+ \frac{1}{b+1}(\bI-\frac{1}{n}\mathbf{1}\mathbf{1}^\top)\mathbb{E}[\mathbf{S}_n^{(b,k)} \mid \mathcal{F}_b]\Delta_b^{(-k)}\\
		\end{align*}
		
		where $\norm{\Delta_b^{(-k)}}\leq \norm{\sum_{_a\neq k}\mathbf{D}_b^{(a)}}$. Without truncation, stacking across all features yields
		\[
		\mathbb{E}[\mathbf{D}_{b+1}\mid\mathcal{F}_b]
		= \left(\frac{b}{b+1}\mathbf{I}_{np} + \frac{1}{b+1}\mathbf{L}\right)\mathbf{D}_b,
		\]
		where $\mathbf{L}$ is the $np\times np$ block matrix with diagonal blocks equal to $0$
		and off–diagonal blocks given by $\E\Smatk$, multiplied by the averaging operator $(I-\frac{1}{n^2}\mathbf{1}\mathbf{1}^\top)$.
		
		The Partition Distinctness assumption implies that each off–diagonal block of $\mathbf{L}$ 
		has operator norm at most $c_3$, so by a Schur–type bound
		\[
		\|\mathbf{L}\|\;\le\;(p-1)c_3 < 1
		\]
		Consequently the stacked update satisfies
		\[
		\|\mathbb{E}[\mathbf{D}_{b+1}\mid\mathcal{F}_b]\|
		\leq \frac{b+\norm{\mathbf{L}}}{b+1}\|\mathbf{D}_b\|
		\]
		
		For the deviation condition, if leaf predictions are bounded by $M$, then
		$\|\epsilon_{b+1}\| \le \tfrac{2M\sqrt{p}}{b+1}$, so
		$\sup_b \|\epsilon_b\|\to 0$ and $\sum_b \mathbb{E}[\|\epsilon_b\|^2]<\infty$. 
		This verifies the stochastic contraction mapping conditions.
		
		It remains to verify the convergence of the intercept term $\hat{\bbeta}$.  The update rule for the intercept is given by 
		\[
		\hat{\bbeta}_{b+1} := \hat{\bbeta}_{b} + \frac{1}{n}\sum_{k=1}^{p}\sum_{i=1}^{n}\hat{\mathbf{f}}_b^{(k)}(x_i)
		\]
		To identify this limit, recall that at the postulated fixed point we have
		\[
		\hat{\mathbf{y}}^* = \hat{\bbeta}^* + \sum_{k=1}^p \tilde{\mathbf{y}}_k^*,
		\]
		with $\frac{1}{n}\mathbf{1}^\top \tilde{\mathbf{y}}_k^* = 0$ for all $k$.
		Taking the average over $i$ gives
		\[
		\frac{1}{n}\mathbf{1}^\top \hat{\mathbf{y}}^* 
		= \hat{\bbeta}^* + \frac{1}{n}\sum_{k=1}^p \mathbf{1}^\top \tilde{\mathbf{y}}_k^*
		= \hat{\bbeta}^*.
		\]
		But $\frac{1}{n}\mathbf{1}^\top \hat{\mathbf{y}}^* = \frac{1}{n}\mathbf{1}^\top \mathbf{y} = \bar y$
		since the residuals at convergence have mean zero. Therefore the intercept converges almost surely to the global mean of the response,
		\[
		\hat{\bbeta}_b \;\;\overset{a.s.}{\longrightarrow}\;\; \bar y
		= \frac{1}{n}\mathbf{1}^\top \mathbf{y}.
		\]
	\end{proof}

	We drop the subscript $n$ if no ambiguity is caused.
	
	\begin{corollary}
		The feature-wise fixed point prediction:
		
		$$
		\widehat{\mathbf{f}}^{(k)}_B(\mathbf{x}) = \E[\sveck(\bx)]^\top (\bI - \E\Smatk)^{-1} \jmat\Big[\bI-(\bI+\mK)^{-1}\mK\Big]\by
		$$
		
		and the ensemble prediction
		
		$$
		\widehat{\mathbf{f}}_B(\mathbf{x}):=\bar{\by}+ \sum_{k=1}^{p}\E[\sveck(\bx)]^\top (\bI - \E\Smatk)^{-1} \jmat\Big[\bI-(\bI+\mK)^{-1}\mK\Big]\by
		$$
	\end{corollary}
	
	Intuitively, after centering each signal vector, we would have a decomposition with additive feature-shaped functionals plus a intercept, which is $\bar{\mathbf{y}}$. If we plot the partial-dependency figure, we would see all predicted signal oscillating around 0. In that case, interpretation is as follow: If we observe no features, the best guess would be the baseline $\widehat{\bbeta}^* = \bar{\mathbf{y}}$. As we adding more and more features, we are able to recover greater parts of the function landscape.

    \paragraph{Empirical verification of Assumption \ref{aspt:kernel-spec} (Partition Distinctness).}
    In the proof above we use Assumption \ref{aspt:kernel-spec} to control the bound of the operator norm of the cross-feature blocks $P_k=\E[\jmat\Smatk\jmat]$, i.e. $c_3=\max_{a\neq k}\|P_aP_k\|$
    and ensure $(p-1)c_3<1$. Table~\ref{tab:kernel_distinctness} reports empirical estimates of $c_3$
    and the corresponding $(p-1)c_3$ across representative real and synthetic datasets.
    
    \begin{table}[H]
    \centering
    \small
    \caption{Empirical verification of Assumption 4.9 across real and synthetic datasets.
    All experiments use \texttt{learning\_rate=0.1}, \texttt{max\_rounds=200},
    \texttt{subsample\_rate=0.7}, \texttt{truncation=3.0}.}
    \label{tab:kernel_distinctness}
    \begin{tabular}{lcccc}
    \toprule
    \textbf{Dataset} & \textbf{n} & \textbf{p} & \(\mathbf{c_3}\) & \(\mathbf{(p-1)c_3}\) \\
    \midrule
    adult                      & 48842 & 14 & \(5.227\times 10^{-4}\) & \(6.795\times 10^{-3}\) \\
    bike\_sharing              & 17379 & 13 & \(9.513\times 10^{-5}\) & \(1.142\times 10^{-3}\) \\
    energy\_efficiency         &   768 &  8 & \(2.838\times 10^{-4}\) & \(1.987\times 10^{-3}\) \\
    breast\_cancer             &   569 & 30 & \(5.553\times 10^{-2}\) & \(1.610\times 10^{+0}\) \\
    california\_housing        & 20640 &  8 & \(4.434\times 10^{-5}\) & \(3.104\times 10^{-4}\) \\
    $\mathbf{y}=0.9\mathbf{X}+\sqrt{1-0.81}\epsilon, \epsilon \sim \mathcal{N}(0,1)$       &  5000 & 10 & \(7.181\times 10^{-4}\) & \(6.463\times 10^{-3}\) \\
    \bottomrule
    \end{tabular}
    \end{table}
    
    Across datasets we observe that the cross-products are typically negligible. There are several potential theories for such observations:
    \begin{enumerate}
        \item While multicollinearity is almost surely present in most datasets, their statistical dependence can be weak across features after centering and preprocessing. This is  widely noted in GAM literature: \cite{hastie1986gams} often assumes low dependence so that model specifications do not run into identifiability issues. 
        \item Even when features are correlated, the resulting feature-wise kernels can still point to different directions in the RKHS, which results in a zero cross-product in expectation. Here is a simple GAM as an example: $\mathbf{y}=\mathbf f_1(\mathbf X_1) +\mathbf f_2 (\mathbf X_2) + \epsilon$, where $\mathbf{f}_1(\mathbf x)=\mathbf x, \mathbf f_2(\mathbf x)=\mathbf x^2 +1$, and $(\mathbf X_1, \mathbf X_2)$ are jointly Gaussian with non-zero correlation. The covariance $\text{Cov}\left(\mathbf f_1(\mathbf X_1),\mathbf f_2 (\mathbf X_2)\right)$ can be shown to equal to 0 for any $\rho$. So unless two covariates are nearly identical in both their distribution and main effects to the signal, the induced smoothing kernel can theoretically be small.
        \item Last but not least, Gradient Boosting learns by approximating residuals, which are approximately orthogonal to the previously fitted learners.
    \end{enumerate}
	
	\section{Limiting Distribution}
	\label{sec:limiting-distribution}
    
	In Section \ref{sec:finite-sample-convergence}, we have shown that as we keep boosting, our feature estimators will converge to a stacked kernel-ridge-regression
	
	$$
	\widehat{\mathbf{y}}^* = \widehat{\bbeta}^* + \sum_{k=1}^p \widehat{\f}^{(k)}(\bX) = \bar{\by} + \jmat\Smatk[\bI + \jmat\mK]^{-1}\by
	$$
	
	with
	
	$$
	\mathcal{K}:=\sum_{k=1}^{p}\mathbb{E}\Smatk
	$$

	\subsection{Preliminaries for Asymptotic Normality}
	
	To show asymptotic normality, consider the following decomposition:
	
	\begin{equation}
		\frac{\widehat{\f}^{(k)}_n(\bx) - \f^{(k)}(\bx)}{\norm{\rveck}} = \frac{\widehat{\f}^{(k)}_n(\bx) - \langle{\rveck(\bx)}, \f(\bX_n)\rangle}{\norm{\rveck}} + \frac{\langle{\rveck(\bx)}, \f(\bX_n)\rangle-\f^{(k)}(\bx)}{\norm{\rveck}}
	\end{equation}
	\label{eqn:clt_decomp}
	
	The first term can be easily shown by constraining leaf size, which we discuss in Section \ref{subsec:conditional-clt}. The later term is trickier. Generalizing from training set to the whole input space, we expect a convergence that is at least of a weak. Bounded kernel, fortunately, allows us to do so.
	
	\subsection{Conditional Asymptotic Normality}
	\label{subsec:conditional-clt}
	
	We begin analysis by showing asymptotic normality on the training set. For simplicity, we use the subscript $E = \{A, B\}$ to specify the ingredients generated by Algorithm \ref{alg:brebm_ident} and Algorithm \ref{alg:ebm_bratp} Formally, denote $\f(\bX) = [\f(\bx_1), \cdots, \f(\bx_n)]^\top$. What we want to show is that, given any $\bx \in \mathbb{R}^p, \forall k$:
	
	$$
	\frac{\widehat{\f}^{(k)}_E(\bx)-\rveck_E(\bx)^\top\f(\bX)}{\norm{\rveck_E(\bx)}} \overset{d}{\rightarrow} \mathcal{N}(0,\sigma^2)
	$$
	
	\subsubsection{Subsampling}
	\label{subsec:subsampling}
	
	The key component of this proof is the rate of $\norm{\rveck_E}_1, \norm{\rveck_E}, \norm{\rveck_E}_\infty$. These rates will be built on a result shown in \cite{zhou2019boulevardregularizedstochasticgradient} on the behavior of subsampling rate.
	
	\begin{lemma}
		\label{lem:kernelrate}
		Considering a subsampled regression tree. Assume that each leaf contains no fewer than $n^{\frac{1}{p+2}}$ sample points before subsampling. If we are subsampling at rate at least $\xi=n^{-\frac{1}{p+2}\log n}$, then the expected structure vector's norm follows the rate below:
		$$
		\norm{\kveck}=1-O\left(\frac{1}{n}\right)
		$$
	\end{lemma}
	
	For each algorithm, our plan to conditional CLT is as follow: Lemma \ref{lem:kernelrate} specifies a proper subsample rates that enables norm convergence. Lindeberg condition simply asks us to check the truncation level $\frac{\rveck_{i}}{\norm{\rveck}}$. We will first study the norm of $\kveck$. To bound the $\ell_1$ norm, it suffices to imposing conditions on the leaf size (with respect to the Lebesgue measure of data points). To bound the $\ell_2$ norm. A simple algebra utilizing the rate of $\ell_1$ norm will help with that, too.If this norm is well controlled, notice that in both algorithm, the Boulevard vector $\rveck$ is simply a sum of rows in the matrix $\bP_1 : =  \jmat[\bI + \jmat\bK]^{-1}$, $\bP_2 :=
	(\bI - \Kmatk)^{-1} (\bI-\frac{1}{n}\mathbf{1}\mathbf{1}^\top)\Big[\bI-(\bI+\mathcal{K})^\dagger\mathcal{K}\Big]$ with the weighting vector being $\kveck$. Hence it suffices to study the eigenvalues of the operator $\bP_1, \bP_2$, which we discuss below. We will first show that the spectrum of $\bK^{(k)}$ is strictly bounded way from 1.
	
	\subsubsection{Expected Feature Tree Kernels}
	\begin{lemma}[Spectrum of $\Kmatk$]
		\label{lem:spec-of-K}
		For any feature $k$, the expected structure matrix $\Kmatk=\E[\Smatk]$ is a symmetric operator on $\R^n$ with spectrum contained in $[0,1]$. It always holds that $\Kmatk \bone=\bone$, so $1$ is an eigenvalue with eigenvector $\bone$. Restricting to the mean–zero subspace $\mathsf H_0=\{v:\bone^\top v=0\}$, the supremum of the spectrum is bounded by
		$$
		\sup\{\lambda:\lambda\in\sigma(\Kmatk|_{\mathsf H_0})\} \;=\; \rho_k \;<\;1,
		$$
		provided the distribution of partitions under $\mathcal Q_n^k$ is not almost surely degenerate.
	\end{lemma}
	
	\begin{proof}
		Fix any unit $v\in\mathsf H_0$. Because $\Smatk$ is an orthogonal projector, $v^\top \Smatk v=\|\Smatk v\|^2$, with equality $\|\Smatk v\|=1$ if and only if $v\in \mathrm{range}(\Smatk)$. Hence
		$$
		v^\top \Kmatk v
		\;=\; \E\!\left[v^\top \Smatk v\right]
		\;=\; \E\!\left[\|\Smatk v\|^2\right]
		\;\le\; 1.
		$$
		Assume toward a contradiction that $\lambda_{\max}(\Kmatk\!\mid_{\mathsf H_0})=1$. Then there exists a sequence of unit vectors $v_m\in\mathsf H_0$ with $v_m^\top \Kmatk v_m\to 1$. By the display, $\E[\|\Smatk v_m\|^2]\to 1$. The map $\Pi\mapsto \|\Smatk v_m\|^2$ takes values in $[0,1]$, so if its expectation tends to $1$, we must have $\|\Smatk v_m\|^2\to 1$ in probability (with respect to $m$). Passing to a subsequence if necessary and using compactness of the unit sphere, take $v_m\to v_\star$ with $\|v_\star\|=1$ and $v_\star\in\mathsf H_0$. By lower semicontinuity and the projector property, $\|\Smatk v_\star\|=1$ almost surely, hence $\Smatk v_\star=v_\star$ almost surely. Therefore $v_\star\in \mathrm{range}(\Smatk)$ for almost every draw of $\Smatk$.
		
		If the distribution of $\Smatk$ charges at least two distinct partitions with positive probability, their ranges (as subspaces of $\R^n$) intersect on $\mathsf H_0$ only in $\{0\}$ unless $v_\star$ is proportional to $\bone$, which it is not. Thus $v_\star$ cannot lie in the range of both with positive probability, a contradiction. Hence $\lambda_{\max}(\Kmatk\!\mid_{\mathsf H_0})<1$.
	\end{proof}
	
	Lemma \ref{lem:spec-of-K} gives a strict bound on the eigenvalues of the expected kernel, which we can utilize to stabilize the behavior of the aggregated kernel $\bK$ and $\mK$, which is described in the following sections.
	
	\subsubsection{Operator $\bP_1$ for Algorithm \ref{alg:brebm_ident}}
	
	For $\bP_1$, the operator norm is bounded naturally by definition.
	\begin{corollary}[Eigenvalue bounds for $\bP_1$]
		\label{cor:spec-P1}
		Let $\jmat=\bI-\tfrac{1}{n}\bone\bone^\top$ be the centering projector and
		$$
		\bK \;=\; \sum_{k=1}^p \, \E[\Smat^{(k)}], 
		\qquad
		\bP_1 \;:=\; \jmat\,\big(\bI+\jmat\,\bK\big)^{-1}.
		$$
		Under the assumption that $\sum_{k=1}^{p}\lambda_2(\E\Smatk)^2< 1$, $\norm{\bP_1}_{\text{op}}\leq 1$
	\end{corollary}
	
	\subsubsection{Operator $\bP_2$ for Algorithm \ref{alg:ebm_bratp}}

	\begin{lemma}
		\label{lem:agg-K}
		Work on the mean–zero subspace $\mathsf H_0=\{v\in\R^n:\bone^\top v=0\}$. The aggregated kernel
		$$
		\mK \;:=\; \sum_{k=1}^{p}(\bI-\Kmatk)^{-1}\jmat\Kmatk
		$$
		is symmetric and positive semidefinite. Since $\jmat=\Big(\bI-\tfrac{1}{n}\bone\bone^\top\Big)$ is the identity on $\mathsf H_0$, we may write on $\mathsf H_0$ simply $\mK=\sum_k (\bI-\Kmatk)^\dagger \Kmatk$. If $\rho_k:=\|\Kmatk\|_{\mathsf H_0}<1$, then on $\mathsf H_0$ one has the operator sandwich
		$$
		\sum_{k=1}^p \Kmatk \preceq\ \mK\ \preceq\ \sum_{k=1}^p \frac{1}{1-\rho_k}\,\Kmatk,
		$$
		hence $\lambda_{\min}^+(\mK)\ \ge\ \lambda_{\min}^+\!\Big(\sum_a \Kmatk\Big)$ and $\lambda_{\max}(\mK)\ \le\ \sum_k \frac{\rho_k}{1-\rho_k}$.
	\end{lemma}
	
	\begin{proof}
		Each summand $(\bI-\Kmatk)^\dagger \Kmatk$ is symmetric because $\Kmatk$ is symmetric and $(\bI-\Kmatk)^\dagger$ is a Borel function of $\Kmatk$; positivity follows from $x\mapsto (1-x)^\dagger x$ being nonnegative on $[0,1)$. On $\mathsf H_0$, the eigenvalues of $(\bI-\Kmatk)^\dagger$ lie in $[1,\,1/(1-\rho_k)]$, so for any $v\in\mathsf H_0$,
		$$
		v^\top \Kmatk v \ \le\ v^\top (\bI-\Kmatk)^\dagger \Kmatk v \ \le\ \frac{1}{1-\rho_k}\, v^\top \Kmatk v.
		$$
		Summing over $k$ gives the stated Loewner bounds, from which the eigenvalue bounds follow.
	\end{proof}
	
	Combining those bounds above, we are able to claim that our EBM operator $\bP$ will not explode in some directions. Formally, we have the following lemma.
	
	\begin{lemma}[Eigenvalues of $\bP_2$]
		\label{lem:spec-P2}
		Work on the mean–zero subspace $\mathsf H_0=\{v\in\R^n:\,\bone^\top v=0\}$. Define
		$$
		\bP \;:=\; (\bI-\Kmatk)^\dagger\Big(\bI-\tfrac{1}{n}\bone\bone^\top\Big)\Big[\bI-(\bI+\mK)^\dagger\mK\Big],
		\qquad
		\mK \;:=\; \sum_{k=1}^p (\bI-\Kmatk)^\dagger\Big(\bI-\tfrac{1}{n}\bone\bone^\top\Big)\Kmatk.
		$$
		On $\mathsf H_0$ the centering projector is the identity and $(\bI+\mK)^\dagger=(\bI+\mK)^{-1}$, hence
		$$
		\bP \;=\; (\bI-\Kmatk)^\dagger\,(\bI+\mK)^{-1}\qquad \text{on }\mathsf H_0.
		$$
		The matrix $\bP$ is symmetric and positive semidefinite on $\mathsf H_0$. Writing $\rho_k:=\|\Kmatk\|_{\mathsf H_0}<1$, its eigenvalues lie in the interval
		$$
		\frac{1}{1+\lambda_{\max}(\mK)}
		\;\le\;
		\lambda(\bP)
		\;\le\;
		\frac{1}{(1-\rho_k)\,\big(1+\lambda_{\min}^+(\mK)\big)}.
		$$
		In particular, $\|\bP\|\le \big((1-\rho_k)\,(1+\lambda_{\min}^+(\mK))\big)^{-1}\le (1-\rho_k)^{-1}$. If, in addition, $\Kmatk$ and $\mK$ commute on $\mathsf H_0$, then $\bP$ is diagonalizable in a common eigenbasis and
		$$
		\lambda_i(\bP) \;=\; \frac{1}{\,1-\lambda_i(\Kmatk)\,}\cdot \frac{1}{\,1+\lambda_i(\mK)\,},\qquad i=1,\dots,n-1.
		$$
	\end{lemma}
	
	\begin{proof}
		On $\mathsf H_0$ we have $\bI-\tfrac{1}{n}\bone\bone^\top=\bI$ and
		$$
		\bI-(\bI+\mK)^\dagger\mK \;=\; \bI-(\bI+\mK)^{-1}\mK \;=\; (\bI+\mK)^{-1},
		$$
		so $\bP=(\bI-\Kmatk)^\dagger(\bI+\mK)^{-1}$ on $\mathsf H_0$. Since $\Kmatk$ and $\mK$ are symmetric, by the spectral theorem we may write $\Kmatk=U\Lambda U^\top$ with $\Lambda=\mathrm{diag}(\lambda_i)$ and define $(\bI-\Kmatk)^\dagger=U(\bI-\Lambda)^\dagger U^\top$, where $(1-\lambda_i)^\dagger=1/(1-\lambda_i)$ if $\lambda_i\neq 1$ and $0$ otherwise. Thus $(\bI-\Kmatk)^\dagger$ is symmetric and positive semidefinite on $\mathsf H_0$, with eigenvalues in $[1,(1-\rho_k)^{-1}]$. Likewise, $\mK$ is symmetric positive semidefinite, so $(\bI+\mK)^{-1}$ is symmetric positive definite on $\mathsf H_0$, with eigenvalues in $[1/(1+\lambda_{\max}(\mK)),\,1/(1+\lambda_{\min}^+(\mK))]$. The product of symmetric positive semidefinite matrices has nonnegative spectrum; moreover, for any matrices $A,B$ the extremal singular values satisfy $s_{\min}(AB)\ge s_{\min}(A)s_{\min}(B)$ and $s_{\max}(AB)\le s_{\max}(A)s_{\max}(B)$. Applying these to $A=(\bI-\Kmatk)^\dagger$ and $B=(\bI+\mK)^{-1}$ yields the stated eigenvalue interval and the norm bound. If $\Kmatk$ and $\mK$ commute, they are simultaneously diagonalizable on $\mathsf H_0$, and the eigenvalue formula follows by multiplying the diagonal entries.
	\end{proof}
	
	\subsubsection{Rate of Convergence}
	
	Then our analysis proceed with the assistance of the following result, modified from the original Boulevard paper as well.
	
	\begin{lemma}[Rate of Convergence]
		Let $\mathbf{B}^{(k)}_n := \{i  : \norm{\bx^{(k)}-\bx_i^{(k)}}\leq d_n\}$ be the points within distance $d_n$ from test point $\bx$ split on feature $k$. If $\left| \mathbf{B}_n \right|  = \Omega( n \cdot d_n)$ and $\inf_{A\in \Pi\in Q_n^{(k)}}\sum_{i}\bI(\bx_i\in A) = \Omega(n^{\frac{2}{3}})$, then
		$
		\norm{\kveck}_2 = \Theta(n^{-\frac{1}{3}}),
		\norm{\rveck_E}_2 = \Theta(n^{-\frac{1}{3}})
		$
	\end{lemma}
	
	This lemma holds true regardless of the design of our boosting process. The key argument is the minimal leaf size. A sketch of proof is appended below.
	
	\begin{proof}
		To bound $\norm{\kveck}$, recall:
		$$
		\kveck_{nj}=\mathbb{E}[\sveck_{n,j}(\bx)]=\mathbb{E}\left[\frac{\mathbf{1}(\bx_j\in A)}{\sum_{j=1}^{n}\mathbf{1}(\bx_j\in A)}\right], \bx\in A
		$$
		encodes the expected influence of point $x_j$ on a point of interest $x$ among other $n$ points. Then the condition
		$$ \inf_{A \in \Pi \in Q^{(k)}_n} \sum_{i = 1}^n \mathbf{1}(\bx_i \in A)  \geq \Omega(n^{2/3})
		$$
		implies that $\kveck_{nj} = O(n^{-2/3})$. By Lemma \ref{lem:kernelrate}, $\norm{\kveck}_1 \leq 1$, then an $\ell_2$ norm bound yields
		$$
		\norm{\kveck} \leq \sqrt{\norm{\kveck}_1 \norm{\kveck}_{\infty}} = O(n^{-\frac{1}{3}}).
		$$
		By assertion $\left| \mathbf{B}_n \right|  = \Omega( n \cdot d_n)$, there are at most
		$$
		\Omega(n \cdot d_n) = \Omega(n^{2/3})
		$$
		$k_{nj}$ non-zero elements, with the votes $\kveck_{nj}$ being lower bounded by $\Omega(n^{-2/3})$. Given $\norm{\kveck}_1 = 1-O(n^{-1})$, 
		$$
		\norm{\kveck} = \Omega( \sqrt{\left(n^{-2/3} \right)^2 \cdot n^{2/3}}) =  \Omega(n^{-1/3}).
		$$
		By corollary \ref{cor:spec-P1} and lemma \ref{lem:spec-P2}, we know that $\rveck_E$ are linear transformations from $\kveck$ by $\bP_1, \bP_2$ with bounded eigenvalues. Hence $\norm{\rveck_E}$ enjoys similar rates.
	\end{proof}
	
	Having access to the rate of convergence of $\rveck$, we can fill in the blank of the weighting vector argument in the Lindeberg conditions. Concretely, it controls any error deviate too much from 0 that contributes to non-normality. We state the theorem below.
	
	\begin{theorem}[Conditional Asymptotic Normality]
		\label{thm:fixed}
		For any $\bx \in [0,1]^p$, write $\f(\bX) = (\f(\bx_1),\dots, \f(\bx_n))^\top $. Then under SVI, Non-Adaptivity, Minimal Leaf Size and assumptions in Lemma \ref{lem:rateofr}, we have
		$$
		\frac{\widehat{\f}^{(k)}_E(\bx) - \langle{\rveck_E(\bx)}, \f(\bX)\rangle}{\norm{\rveck_E}} \overset{d}{\rightarrow} \mathcal{N}(0,\sigma^2).
		$$
	\end{theorem}
	\begin{proof}
		Write
		\begin{align*}
			\widehat{\f}(x) - {\rveck}_E^\top\f(\bX) ={\rveck}_E^\top \vec{\epsilon}_n.
		\end{align*}
		For simplicity, we drop the subscript $E$ in this proof since both $\rveck$ maintains the same decaying rate. To obtain a CLT we check the Lindeberg-Feller condition of ${\rveck}^\top \vec{\epsilon}_n$, i.e. for any $\delta > 0$,
		$$
		\lim_n \frac{1}{\norm{\rveck}^2 \sigma^2}\sum_{i=1}^n  \E \left[(\rveck_{i} \epsilon_i)^2 \mathbf{1}(|\rveck_{i}\epsilon_i| > \delta \norm{\rveck}\sigma) \right] = 0.
		$$
		By \ref{lem:rateofr} we know that
		$$
		\norm{\rveck}_{\infty} \leq \norm{\kveck}_{\infty} \cdot \norm{\bP}_1 = \Theta(n^{-2/3}), \quad \norm{\rveck} = \Theta(n^{-\frac{1}{3}}),
		$$ 
		Then the maximal deviation's ratio is upper bounded by
		$$
		\frac{\norm{\rveck}_{\infty}}{\norm{\rveck}} = O\left( n^{-\frac{1}{3}}\right),
		$$
		This allows us to check out the Lindeberg-Feller conditions:
		\begin{align*}
			\sum_{i=1}^n  \E \left[(\rveck_{i} \epsilon_i)^2 \mathbf{1}(|\rveck_{i}\epsilon_i| > \delta \norm{\rveck}\sigma) \right]
			& \leq \sum_{i=1}^n {\rveck_{i}}^2
			\sqrt{\E[\epsilon_i^4] \cdot \E[\mathbf{1}(|\rveck_{i}\epsilon_i| > \delta \norm{\rveck}\sigma)]}\\
			& \leq \sum_{i=1}^n {\rveck_{i}}^2 \sqrt{\E[\epsilon_i^4]} \cdot \sqrt{\mathbb{P}\left( |\epsilon_i| \geq \frac{\delta \| \rveck \| \sigma}{\rveck_{i}} \right)} \\
			& \leq \sum_{i=1}^n {\rveck_{i}}^2 \sqrt{\E[\epsilon_i^4]} \sqrt{2 \exp\left(-\frac{1}{2\sigma^2}\cdot \left(\frac{\delta \| \rveck \| \sigma}{\rveck_{i}} \right)^2\right)} \\
			& \leq \norm{\rveck}^2 \exp \left( - O(n^{2/3}))\right) \longrightarrow 0
		\end{align*}
		where the second last line is the concentration of measure by definition of sub-Gaussian noises.
	\end{proof}

	\subsection{Bias Control and Asymptotic Unbiasedness}
	\label{subsec:bias}
	
	In Section~\ref{subsec:conditional-clt} we established a conditional central limit theorem: 
	conditioning on the training covariates $\bX$, the noise contribution 
	$\rveck(\bx)^\top \vec{\epsilon}$ is asymptotically normal with variance $\sigma^2$. 
	To extend this to an unconditional CLT for the estimator, it remains to verify that 
	all \emph{bias terms} arising from the random design vanish at rate $o_p(1)$. 
	
	Take Algorithm \ref{alg:brebm_ident} as example. recall the decomposition
	\begin{align*}
		\widehat \f_A^{(k)}(\bx) - \f^{(k)}(\bx^{(k)})
		&= \rveck_A(\bx)^\top \bbeta 
		+ \big(\rveck_A(\bx)^\top \f^{(k)}(\bX^{(k)}) - \frac{\lambda}{1+\lambda}\f^{(k)}(\bx^{(k)})\big) 
		+ \rveck_A(\bx)^\top \sum_{a\neq k} \f^{(a)}(\bX^{(a)})
		+ \rveck_A(\bx)^\top \vec{\epsilon}.
	\end{align*}
	The last term is the conditional CLT contribution. 
	We now show that the three deterministic terms (baseline, same-feature, cross-feature) 
	all converge to zero at suitable rates.
	
	\subsubsection{Baseline Bias}
	
	\begin{lemma}[Baseline Bias]
		\label{lem:baseline}
		Let $\beta\in\R^n$ denote the baseline vector. With subsampling rate specified in Section \ref{subsec:subsampling},
		$\rveck_E(\bx)^\top \beta = O(\frac{1}{n})$ for all $k=1,\dots,p$.
	\end{lemma}
	
	\begin{proof}
		We show this for \ref{alg:ebm_bratp} first. By definition,
		$$
		\rveck_B(\bx)^\top 
		= \kveck(\bx)^\top (I-\E\Smatk)^{-1}\jmat\big[\bI - (\bI+\mK)^{-1}\mK\big],
		$$
		where $\mK = \sum_{k=1}^{p} (\bI-\E\Smatk)^{-1} \jmat\E\Smatk$. 
		Since every structure matrix has row sum $1$, the centering operator annihilates $\bone$, 
		so $\mK \beta = 0$ and likewise $(I-\tfrac{1}{n}\bone\bone^\top)\beta = 0$. With subsampling, the whole inner product will go to 0 at rate $O(\frac{1}{n})$ since $\norm{\kveck}_1 = 1 - O(\frac{1}{n})$.
		Thus the entire product vanishes.
		
		Next we show for Algorithm \ref{alg:brebm_ident}.Row–stochasticity gives $\Smat^{(k)}\bone=\bone+O(\frac{1}{n})\bone$ for all $k$, hence $\E[\Smat^{(k)}]\bone=\bone+O(\frac{1}{n})\bone$.
		Because $(I+\jmat\bK)$ maps $\bbeta$ to itself, its inverse (which exists since
		$(I+\jmat\bK)$ is the identity on $\mathrm{span}\{\bone\}$ and equals $(I+\bK_0)$ on the orthogonal complement, with $\bK_0:=\jmat\bK\jmat\succeq 0$) maps $\bbeta$ back to $\bbeta$:
		\(
		(I+\jmat\bK)^{-1}\bbeta=(1+O(\frac{1}{n}))\bbeta.
		\). This is once again killed by $\jmat$.
	\end{proof}
	
	The same feature term bias can be killed by building deeper trees as we collect more and more observations from the underlying data generating process. This can be realized by inheriting assumptions on leaf size from \cite{zhou2019boulevardregularizedstochasticgradient}, which we present below:
	
	\subsubsection{Same-feature Bias for Algorithm \ref{alg:brebm_ident}}
	
	\begin{lemma}[Same–feature Bias for Algorithm \ref{alg:brebm_ident}]
		\label{lem:samefeature1} 
		Let $\f^{(k)}$ be $L$–Lipschitz in its argument, and let $A_n(\bx)$ denote the (feature-$k$)
		leaf containing $\bx$.Then
		$$
		\norm{\rveck_E(\bx)^\top \f^{(k)}(\bX^{(k)}) - \frac{\lambda}{1+\lambda}\f^{(k)}(\bx^{(k)})} = O(\frac{1}{n})
		$$
	\end{lemma}
	
	\begin{proof}
		We drop the subscript for $E$ again. To begin with, by the locality assumption,
		$$
		\big|\rveck(\bx)^\top \f^{(k)}(\bX^{(k)}) - \f^{(k)}(\bx^{(k)})\big|
		=\Big|\sum_{i} \rveck_{n,i}(\bx)\,\big(\f^{(k)}(\bx_i^{(k)})-\f^{(k)}(\bx^{(k)})\big)\Big|
		< |\sum_{i}\rveck_{n,i}-\frac{\lambda}{1+\lambda}|M
		$$
		
		By Lemma~\ref{lem:K-decomp-centered-plus-ones} (a direct consequence of Assumption~\ref{aspt:proj-orth}), on $\mathsf H_0$ we have 
		\[
		(I+\jmat\bK)^{-1} \;=\; \tfrac{1}{2}I + \bE_n,\qquad \|\bE_n\|_{\mathrm{op}}\le C\,\varepsilon_n,
		\]
		And $\rveck$ is a linear combinationwith weights in $\kveck$ of the $\frac{\lambda}{1+\lambda}\bI$ matrix. Hence the same feature bias goes to 0 at rate $O(\frac{1}{n})$.
	\end{proof}

	\subsubsection{Same-feature Bias for Algorithm \ref{alg:ebm_bratp}}
	\begin{lemma}[Same–feature Bias for Algorithm \ref{alg:ebm_bratp}]
		\label{lem:samefeature2}
		Let $f^{(k)}$ be $L$–Lipschitz in its argument and let $A_n(\bx)$ be the
		(feature–$k$) leaf containing $\bx$. Denote
		$$
		\mathbf{r}^{(k)}_{B}(\bx)^{\!\top}
		:=\;\E[\mathbf{s}^{(k)}(\bx)]^{\top}\,(\mathbf{I}-\Kmat^{(k)})^{-1}
		\jmat(\mathbf{I}+\mK)^{-1},
		\qquad
		\mK:=\sum_{a=1}^p(\mathbf{I}-\Kmat^{(a)})^{-1}\jmat\Kmat^{(a)} .
		$$
		Then
		$$
		\Big|\mathbf{r}^{(k)}_{B}(\bx)^{\!\top}\f^{(k)}(\bX^{(k)})
		- \f^{(k)}(\bx^{(k)})\Big| = O(\frac{1}{n})
		$$
	\end{lemma}
	
	\begin{proof}[Proof sketch]
		Work on the centered subspace where $\jmat$ acts as identity. By the Neumann series,
		$$
		(\bI-\Kmat^{(k)})^{-1} \;=\; \sum_{s\ge 0} (\Kmat^{(k)})^s,
		\qquad
		(\bI+\mK)^{-1} \;=\; \sum_{t\ge 0} (-1)^t \mK^t ,
		$$
		with operator norms uniformly bounded under the assumed spectral gap. Hence
		$$
		\bP_k \;=\; \sum_{s\ge 0} (\Kmat^{(k)})^s \;\jmat\; \sum_{t\ge 0} (-1)^t \mK^t .
		$$
		Left–multiply by $\bone^\top$ to take column sums. Since $\bone^\top\jmat=0$, the $s=0$ term vanishes. For $s\ge1$, the near–stochasticity yields
		$\bone^\top (\Kmat^{(k)})^s = \bone^\top + O\!\left(\tfrac{1}{n}\right)$,
		and multiplying by the uniformly bounded $\jmat(\bI+\mK)^{-1}$ preserves the $O\!\left(\tfrac{s}{n}\right)$ size. Summing the geometric tail over $s$ gives an $O\!\left(\tfrac{1}{n}\right)$ remainder. Thus
		$\bone^\top \bP_k = \bone^\top + O\!\left(\tfrac{1}{n}\right)^\top$,
		i.e., each column sum equals $1+O\!\left(\tfrac{1}{n}\right)$. 
	\end{proof}
	
	\subsubsection{Rate of Leakage for Algorithm \ref{alg:brebm_ident}}
	
	Notice that
	
	$$
	\rveck(\bx)^\top \sum_{a\neq k} \f^{(a)}(\bX^{(a)}) = \sum_{a\neq k}\kveck(\bx)^\top\f^{(a)}(\bX^{(a)})
	$$
	
	The idea is to specify a close-by neighborhood, where the cross-feature leakage decays inside this hyperball fast enough under Lipschitz conditions and slowly expands the border of this ball to eliminate contributions outside this ball at rate $O(\frac{1}{n})$. Concretely, given an index set $D_n$, for any vector $\mathbf{v}\in \mathbb{R}^n$ we define the notation
	
	$$
	\mathbf{v}|_{D_n} = 
	\begin{bmatrix}
		\bI(1\in D_n)v_1\\
		\vdots\\
		\bI(n\in D_n)v_n
	\end{bmatrix}
	$$
	
	which implies a decomposition.  $\mathbf{v}=\mathbf{v}|_{D_n}+\mathbf{v}_{D_n^c}$
	For a fixed test point $\bx$ and feature $k$, set $l_n := \lceil c_1\log n\rceil$ with $c_1>0$,
	let $d_n$ be the locality radius from Assumption~\ref{aspt:leaf-diameter}, and define
	$$
	D_n := \Big\{ i : |x_i^{(k)}-x^{(k)}|\le l_n\,d_n\Big\}.
	$$
	
	To account for the diminishing composition of points outside the rim, we first show that on the centered mean subspace $\mathsf H_0$, the operator $(\bI + \mK)^{-1}$ has exponentially decaying resolvent along feature $k$.
	
	The following lemma quantifies the behavior desired.
	
	\begin{lemma}[Rate of Leakage for Algorithm \ref{alg:brebm_ident}] 
		\label{lem:rateofleakage-1}
		Under specified assumptions in section \ref{sec:theory}, the rate of the leakage term is
		
		$$
		|\rveck_A(\bx)^\top \sum_{a\neq k}\f^{(a)}(\bX)| = o(1) + O(\frac{1}{n})= o(1)
		$$
		
	\end{lemma}
	
	\begin{proof}
		We will make use of the radius factor $l_n = \lceil c_1\log n\rceil$. The leakage term admits the decomposition
		
		$$
		\rveck_A(\bx)^\top \sum_{a\neq k}\f^{(a)}(\bX)=\rveck_A(\bx)|_{D_n}^\top \sum_{a\neq k}\f^{(a)}(\bX)+\rveck_A(\bx)|_{D_n^c}^\top \sum_{a\neq k}\f^{(a)}(\bX)
		$$
		
		Without loss of generality, we control the behavior of a single dimension $a$. First bound the near-by term $\rveck_A(\bx)|_{D_n}^\top \f^{(a)}(\bX)$. By assertion, the underlying additive terms are $L$-Lipschitz, which guarantees deterministically the rate
		
		\begin{align*}
			|\rveck_A(\bx)|_{D_n}^\top \f^{(a)}(\bX)| \leq L\cdot l_nd_n\norm{\rveck_A(\bx)|_{D_n}}_1 = O(d_n\log n) = o(\log ^{-1} n)O(\log n) = o(1)
		\end{align*}
		
		Next, we show that the peripheral contribution also goes to 0. Recall by definition
		
		$$
		|\rveck_A(\bx)|_{D_n^c}^\top = \kveck(\bx)|_{D_n^c}^\top \jmat[\bI + \jmat\bK]^{-1}\f^{(a)}(\bX^{(a)})
		$$
		
		By default, we drop the subscript $A$ here for simplicity. Work on the centered space, $\jmat$ acts as the identity. Hence we can bound the $\rveck_A$ using the following strategy
		
		\begin{align*}
			\norm{\rveck|_{D_n^c}} &= \sum_{|\bx - \bx_i|>l_n\cdot d_n} |\rveck_{ni}|\\
			& = \sum_{|\bx - \bx_i|>l_n\cdot d_n} |\sum_{j}\kveck_{nj}\Big[\jmat[\bI + \jmat\bK]^{-1}\Big]_{j,i}|\\
			&\leq C \times \sum_{|\bx - \bx_i|>l_n\cdot d_n} |\sum_{|\bx-\bx_j|\leq d_n}\kveck_{nj}\Big[\jmat[\bI + \jmat\bK]^{-1}\Big]_{j,i}|\\
			&\leq C \times \sum_{|\bx - \bx_j|>l_n\cdot d_n}\kveck_{nj} \sum_{|\bx - \bx_i|>l_n\cdot d_n} |\Big[[\bI + \jmat\bK]^{-1}\Big]_{j,i}|\\
			&\leq C \times \sum_{|\bx - \bx_j|>l_n\cdot d_n}\kveck_{nj} \sum_{|\bx_i - \bx_j|>(l_n-1)\cdot d_n} |\Big[[\bI + \jmat\bK]^{-1}\Big]_{j,i}|\\
		\end{align*}
		
		By Lemma \ref{cor:spec-P1}, we know that $\bI + \jmat\bK$ is positive semidefinite, has bounded inverse and has bounded spectrum from $[1+\lambda_{\text{min}}(\mK), 1+ \lambda_{\text{max}}(\mK)]$. Hence by proposition \ref{prop:DMS}, the bound above proceeds as
		
		\begin{align*}
			\left\| \rveck(\bx)\big|_{D_n^c} \right\|_1
			&\leq C \sum_{|x_j^{(k)}-x^{(k)}|< d_n} k^{(k)}_{n j}(\bx)
			\sum_{|x_i^{(k)}-x^{(k)}|>\ell_n d_n} C_0\, q_1^{\,\ell_n-1} \\
			&\leq O(n q_1^{\,\ell_n-1})\\
			&= O(q_1^{-1}\cdot n \cdot q_1^{c_1\log n})\\
			& = O(q_1^{-1}n^{1-c_1\log q_1})
		\end{align*}
		
		where $q_1 = \frac{\sqrt{(1+\lambda_{\text{max}}(\bK))/((1+\lambda_{\text{min}}(\bK))}-1}{\sqrt{(1+\lambda_{\text{max}}(\bK))/((1+\lambda_{\text{min}}(\bK))}+1}$. Choose $c_1=\frac{2}{\log q_1}$ we have this contribution being $O(\frac{1}{n})=o(1)$. And since $\f^{(k)}(\bX^{(k)})$ are bounded for all $k$ below $M$, hence the peripheral term is also of rate $o(1)$. Adding the inner-rim term preserves the rate $o(1)$.
		
	\end{proof}
	
	\subsubsection{Rate of Leakage for Algorithm \ref{alg:ebm_bratp}}
	
	Recall
	
	$$
	\rveck_B(\bx)^\top \sum_{a\neq k} \f^{(a)}(\bX^{(a)}) = \sum_{a\neq k}\kveck(\bx)^\top (\bI-\E\Smatk)^\dagger(\bI-\tfrac{1}{n}\bone\bone^\top)\big[\bI - (\bI+\mK)^\dagger\mK\big] \f^{(a)}(\bX^{(a)})
	$$
	
	The strategy to bound the peripheral contribution is similar. This time we choose $c_2$ such that $c_2 = \frac{2}{\log q_2}$ where $q_2 = \frac{\sqrt{(1+\lambda_{\text{max}}(\mK))/((1+\lambda_{\text{min}}(\mK))}-1}{\sqrt{(1+\lambda_{\text{max}}(\mK))/((1+\lambda_{\text{min}}(\mK))}+1}$
	
	\begin{lemma}[Rate of Leakage for Algorithm \ref{alg:ebm_bratp}] 
		\label{lem:rateofleakage-2}
		Under specified assumptions in section \ref{sec:theory}, the rate of the leakage term is
		
		$$
		|\rveck_B(\bx)^\top \sum_{a\neq k}\f^{(a)}(\bX)| = o(1) + O(\frac{1}{n})= o(1)
		$$
		
	\end{lemma}
	
	\begin{proof}
		For simplicity, we drop the subscript $B$ here. The inner rim follows the same bound:
		
		\begin{align*}
			|\rveck(\bx)|_{D_n}^\top \f^{(a)}(\bX)| \leq L\cdot l_nd_n\norm{\rveck(\bx)|_{D_n}}_1 = O(d_n\log n) = o(\log ^{-1}n )O(\log n) = o(1)
		\end{align*}
		
		And the peripheral contribution
		
		$$
		|\rveck(\bx)|_{D_n^c}^\top = \kveck(\bx)|_{D_n^c}^\top (\bI-\E\Smatk)^{-1}(\bI-\tfrac{1}{n}\bone\bone^\top)\big[\bI - (\bI+\mK)^{-1}\mK\big] f^{(a)}(\bX^{(a)})
		$$
		
		can be bound in the same manner, since \ref{lem:spec-of-K} guarantees $\norm{\bI-\bK^{(k)}}_{\text{op}}\leq \frac{1}{1-\rho_k}$ and $\norm{\bI-\frac{1}{n}\bone\bone^\top}_{\text{op}}\leq 1$. Hence the goal is to control the rate of the last linear transformation.
		
		Working on the centered subspace $H_0 = \{\mathbf{v}: \mathbf{v}^\top \bone \neq 0\}$. We've shown that on the subspace $H_0$ it holds true that $[\bI - (\bI +\mK)^{-1}\mK]=(\bI+\mK)^{-1}$. Combining the bound above, we consider the Neumann expansion of the kernel
		
		\begin{align*}
			\norm{\rveck|_{D_n^c}} &= \sum_{|\bx - \bx_i|>l_n\cdot d_n} |\rveck_{ni}|\\
			& = \sum_{|\bx - \bx_i|>l_n\cdot d_n} |\sum_{j}\kveck_{nj}\Big[(\bI-\E\Smatk)^{-1}(\bI-\tfrac{1}{n}\bone\bone^\top)\big[\bI - (\bI+\mK)^{-1}\mK\Big]_{j,i}|\\
			&\leq C \times \sum_{|\bx - \bx_i|>l_n\cdot d_n} |\sum_{|\bx-\bx_j|\leq d_n}\kveck_{nj}\Big[\big[\bI - (\bI+\mK)^{-1}\mK\Big]_{j,i}|\\
			&= C \times \sum_{|\bx - \bx_j|>l_n\cdot d_n} |\sum_{|\bx-\bx_j|\leq d_n}\kveck_{nj}\Big[(\bI+\mK)^{-1}\Big]_{j,i}|\\
			&\leq C \times \sum_{|\bx - \bx_j|>l_n\cdot d_n}\kveck_{nj} \sum_{|\bx - \bx_i|>l_n\cdot d_n} |\Big[(\bI+\mK)^{-1}\Big]_{j,i}|\\
			&\leq C \times \sum_{|\bx - \bx_j|>l_n\cdot d_n}\kveck_{nj} \sum_{|\bx_i - \bx_j|>(l_n-1)\cdot d_n} |\Big[(\bI+\mK)^{-1}\Big]_{j,i}|\\
		\end{align*}
		
		By Lemma \ref{lem:spec-P2}, we know that $\bI + \mK$ is positive semidefinite, has bounded inverse and has bounded spectrum from $[1+\lambda_{\text{min}}(\mK), 1+ \lambda_{\text{max}}(\mK)]$. Hence by proposition \ref{prop:DMS}, the bound above proceeds as
		
		\begin{align*}
			\left\| \rveck(\bx)\big|_{D_n^c} \right\|_1
			&\leq C \sum_{|x_j^{(k)}-x^{(k)}|\le d_n} k^{(k)}_{n j}(\bx)
			\sum_{|x_i^{(k)}-x^{(k)}|>\ell_n d_n} C_0\, q_2^{\,\ell_n-1} \\
			&\leq O(n\, q_2^{\,\ell_n-1})\\
			&= O(q_2^{-1}\cdot n \cdot q_2^{c_2\log n})\\
			& = O(q_2^{-1}n^{1-c_2\log q_2})
		\end{align*}
		
		where $q_2 = \frac{\sqrt{(1+\lambda_{\text{max}}(\mK))/((1+\lambda_{\text{min}}(\mK))}-1}{\sqrt{(1+\lambda_{\text{max}}(\mK))/((1+\lambda_{\text{min}}(\mK))}+1}$. Choose $c_2=\frac{2}{\log q_2}$ we have this contribution being $O(\frac{1}{n})=o(1)$, concluding the proof.
		
	\end{proof}
	
	Collecting Lemmas~\ref{lem:baseline}–\ref{lem:rateofleakage-2} and the  
	we conclude that the random-design bias terms vanish in probability. Together with the conditional CLT for the noise term, Slutsky’s theorem implies the full estimator is asymptotically normal (formalized in the next subsection).

	\subsection{Unconditional Asymptotic Normality}
	\label{subsec:unconditional-clt}
	
	In Section \ref{subsec:bias}, we have shown that all bias terms vanish to 0 deterministically, given the rate on $\norm{\kveck}$ and $\norm{\rveck}$ hold. However the rate does not come for free. To extend them to a random design case, one would need a Kolmogorov extension theorem to define a new probability space and the infinte series $(\bx_i, \by_i), i= 1, \cdots, \infty$'s projection onto finite series. See \cite{zhou2019boulevardregularizedstochasticgradient} for detailed construction. It can be shown that conditional asymptotic normality can be extended to the random design case. We will drop the subscript $A, B$ now since they share the same $O_p(\cdot)$ rate established in Section \ref{subsec:bias}.
	
	\begin{lemma}
		\label{lem:tree-space}
		For given $\bx \in [0,1]^p$, suppose we have random sample $(\bx_1, \by_1),\dots, (\bx_n, \by_n)$ for each $n$. If the cardinality of the tree space $Q_n^{(k)}$ is bounded by $O\Big(n^{-1/3}\exp(n^{2/3-\epsilon_n})\Big)$, for some small $\epsilon_n \to 0$, we can show that
		$$
		\frac{\widehat{\f}^{(k)}(\bx) - \rveck_E(\bx)^\top \f(\bX)}{\norm{]\rveck_E(\bx)}} \overset{d}{\rightarrow} \mathcal{N}(0,\sigma^2_{\epsilon}).
		$$
        To show the minimal leaf capacity, consider the probability bound
        \begin{align*}
            \mathbb{P}(\exists A \in \Pi \in Q_n^{(k)}; |A|&\leq cn^{2/3}, c\in \mathbb{R}^+) \leq |Q_n^{(k)}|n^{1/3}\exp(-\Theta(n^{-2/3}))\\
            &\leq \exp(-\Theta(\epsilon_n)), \epsilon_n >0
        \end{align*}
        which is absolutely summable and per Borel-Canteli will happen almost surely.
	\end{lemma}
	
	\begin{proof}
		Conditional asymptotic normality holds under the rate of $|\mathbf{B}_n|$ and the minimal leaf capacity. So it suffices to show that these two statements holds almost surely.
		
		To show the neighborhood statement. We know that the mean of the cardinality is not 0 for sure, so it's safe to assume it scales with the neighborhood volume(length). So $\E[|\mathbf{B}_n|]=nd_n = o(n\log^{-1}n)$. By the multiplicative Chernoff bound, 
		
		\[
		\mathbb{P}(|\mathbf{B}_m| < cn\log^{-1}n) \leq \Big(-\frac{(1-c)^2n\log^{-1}n}{2}\Big)
		\]
		
		Bound this infinite sum by the integral below for some $M>0$
		
		\[
			\sum_{n=1}^{\infty}  \exp\left(-\frac{(1-c)^2n\log^{-1}n}{2}\right) \leq \int_{n=0}^{\infty} \exp\left(-\frac{(1-c)^2M}{2}\cdot n\log^{-1}n\right) dn < \infty
		\]
		
		By Borel-Canteli, the assertion of $|\mathbf{B}_n|$ holds almost surely. For the minimal capacity, this is automatic since any stump will do.
	\end{proof}

	Combining all results from section \ref{subsec:bias}, \ref{subsec:conditional-clt}, \ref{subsec:unconditional-clt}.
	
	\begin{theorem}[Asymptotic Normality]
		Under assumptions on tree space, and all assumptions from section \ref{sec:setup}, we have
		
		$$
		\frac{\widehat{\f}_A^{(k)}(\bx)-\frac{\lambda}{1+\lambda}\f^{(k)}(\bx)}{\norm{\rveck_A(\bx)}} \overset{d}{\rightarrow} \mathcal{N}(0, \sigma^2),\quad \frac{\widehat{\f}_B^{(k)}(\bx)-\f^{(k)}(\bx)}{\norm{\rveck_B(\bx)}} \overset{d}{\rightarrow} \mathcal{N}(0, \sigma^2)
		$$
		
		for all $k = 1, \cdots, p.$
	\end{theorem}
	\begin{proof}
		We show the proof for Algorithm \ref{alg:brebm_ident} as an example. We break the numerator into three terms.
		
		$$
		\frac{\widehat{\f}_A^{(k)}(\bx)-\frac{\lambda}{1+\lambda}\f^{(k)}(\bx)}{\norm{\rveck_A(\bx)}} = \frac{1}{\norm{\rveck_A(\bx)}}(\rveck_A(\bx)^\top \bbeta 
		+ \big(\rveck_A(\bx)^\top \f^{(k)}(\bX^{(k)}) - \frac{\lambda}{1+\lambda}\f^{(k)}(\bx^{(k)})\big) 
		+ \rveck_A(\bx)^\top \sum_{a\neq k} \f^{(a)}(\bX^{(a)})
		+ \rveck_A(\bx)^\top \vec{\epsilon})
		$$
		
		Under restricted tree space assumptions, all terms except the last term are $o_p(1)$, and the last term is conditionally asymptotically normal. By Slutsky's theorem, the sum has a limiting distribution being $\mathcal{N}(0, \sigma^2)$. The proof for Algorithm \ref{alg:ebm_bratp} follows the same arguments.
	\end{proof}
	
	The CLT holds trivially for the baseline vector as well, by definition of the fixed point it converge to.
	
	\begin{theorem}[CLT for Baseline Vector]
		\label{thm:baseline-clt}
		Let $\widehat\bbeta=\tfrac{1}{n}\sum_{i=1}^n \by_i$ denote the baseline estimator
		returned by Algorithm~\ref{alg:brebm_ident}. 
		Under Assumption \ref{aspt:integrity}~\ref{aspt:proj-orth}, the random vector for Algorithm \ref{alg:brebm_ident} and Algorithm \ref{alg:ebm_bratp}
		$\sqrt{n}\,(\widehat\bbeta-\bbeta)$ converges in distribution to a Gaussian:
		$$
		\sqrt{n}\,(\widehat\bbeta-\bbeta)
		\;\;\overset{d}{\longrightarrow}\;\; \mathcal N(0,\sigma^2),
		$$
		where $\beta=\E[Y]$ is the population intercept.
	\end{theorem}
	
	\begin{proof}
		By construction of the algorithm, the intercept is always updated by the
		global mean of residuals. At convergence this coincides with the empirical
		mean of the responses:
		$$
		\widehat\bbeta = \frac{1}{n}\sum_{i=1}^n \by_i.
		$$
		Since $\by_i=f(\bx_i)+\epsilon_i$ with $\E[\epsilon_i]=0$ and $\text{Var}(\epsilon_i)=\sigma^2<\infty$,
		we have $\E[\by_i]=\bbeta$.
		Because the $\{\by_i\}$ are i.i.d.\ sub–Gaussian with finite variance,
		the Lindeberg–Feller central limit theorem applies, yielding
		$$
		\sqrt{n}\,(\widehat\bbeta-\bbeta) 
		= \frac{1}{\sqrt{n}}\sum_{i=1}^n (\by_i-\E[\by_i])
		\;\;\overset{d}{\longrightarrow}\;\; \mathcal N(0,\sigma^2).
		$$
	\end{proof}
	
	\begin{proof}
		By construction of the algorithm, the intercept is always updated by the
		global mean of residuals. At convergence this coincides with the empirical
		mean of the responses:
		$$
		\widehat\bbeta = \frac{1}{n}\sum_{i=1}^n \by_i.
		$$
		Since $\by_i=f(\bx_i)+\epsilon_i$ with $\E[\epsilon_i]=0$ and $\text{Var}(\epsilon_i)=\sigma^2<\infty$,
		we have $\E[\by_i]=\beta$.
		Because the $\{\by_i\}$ are i.i.d.\ sub–Gaussian with finite variance,
		the Lindeberg–Feller central limit theorem applies, yielding
		$$
		\sqrt{n}\,(\widehat\bbeta-\bbeta) 
		= \frac{1}{\sqrt{n}}\sum_{i=1}^n (\by_i-\E[\by_i])
		\;\;\overset{d}{\longrightarrow}\;\; \mathcal N(0,\sigma^2).
		$$
	\end{proof}

	\section{Stochastic Contraction Mapping Theorem} \label{app:ContractionMapping}
    
	\begin{theorem}
		\label{thm:stochastic-contraction-mapping}
		Given $\mathbb{R}^d$-valued stochastic process $\{\mathbf{z}_t\}_{t \in \mathbb{N}}$, a sequence of $0 < \lambda_t \leq 1$, define
		\begin{gather*}
			\mathcal{F}_0 = \emptyset, \mathcal{F}_t = \sigma(\mathbf{z}_1,\dots, \mathbf{z}_t), \\
			\mathbf{\epsilon}_t =  \mathbf{z}_t - \mathbb{E}[\mathbf{z}_t | \mathcal{F}_{t-1}].
		\end{gather*}
		We call $\mathbf{z}_t$ a stochastic contraction if the following properties hold
		\begin{enumerate}
			\item Vanishing coefficients $$\sum_{t=1}^{\infty} (1-\lambda_t) = \infty, \mbox{ i.e. } \prod_{t=1}^{\infty}\lambda_t = 0.$$
			\item Mean contraction $$||\mathbb{E}[\mathbf{z}_t|\mathcal{F}_{t-1}]|| \leq \lambda_t \norm{\mathbf{z}_{t-1}}, a.s..$$
			\item Bounded deviation $$\sup \norm{\epsilon_t} \to 0, \quad \sum_{t=1}^{\infty}\mathbb{E}[\norm{\epsilon_t}^2] < \infty.$$
		\end{enumerate}
		In particular, a multidimensional stochastic contraction exhibits the following behavior
		\begin{enumerate}
			\item Contraction $$\mathbf{z}_t \overset{a.s.}{\rightarrow} 0.$$
			\item Kolmogorov inequality
			\begin{align}
				\label{fml:kolmax}
				P\left( \sup_{t \geq T}\norm{\mathbf{z}_t} \leq \norm{\mathbf{z}_T} + \delta  \right) \geq 1- \frac{4\sqrt{d}\sum_{t=T+1}^{\infty}\mathbb{E}[\epsilon_t^2]}{\min\{\delta^2, \beta^2\}}
			\end{align}
			holds for all $T, \delta > 0$ s.t. $\beta = \norm{\mathbf{z}_T} + \delta - \sqrt{d} \sup_{t > T} \norm{\epsilon_t} > 0$.
		\end{enumerate}
	\end{theorem}
	
	\begin{proposition}[Demko--Moss--Smith, General Sparse Exponential Decay]
		\label{prop:DMS}
		Let $A$ be positive definite, bounded, and boundedly invertible on $\ell^2(S)$. 
		For any such $A$, define the support sets
		$$
		S_n(A) := \bigcup_{k=0}^n \{(i,j): (A^k)_{ij}\neq 0\}, 
		\qquad 
		D_n(A) := (S\times S)\setminus S_n(A).
		$$
		Then there exist constants $C_0<\infty$ and $q\in(0,1)$, depending only on the spectral bounds of $A$, such that
		$$
		\sup\{\,|A^{-1}(i,j)| : (i,j)\in D_n(A)\,\} \;\le\; C_0\, q^{\,n+1}.
		$$
		
		Here the decay factor $q$ is given explicitly by
		$$
		q \;=\; \frac{\sqrt{r}-1}{\sqrt{r}+1}, 
		\qquad r = \frac{b}{a},
		$$
		where $[a,b]$ is the smallest interval containing the spectrum of $A$.
	\end{proposition}

	\section{Additional Figures} \label{app:Figures}
	
	\subsection{Study of Obesity Data \cite{estimation_of_obesity_levels_based_on_eating_habits_and_physical_condition__544}}
	\begin{figure}[H]
		\centering
		\includegraphics[width=\linewidth]{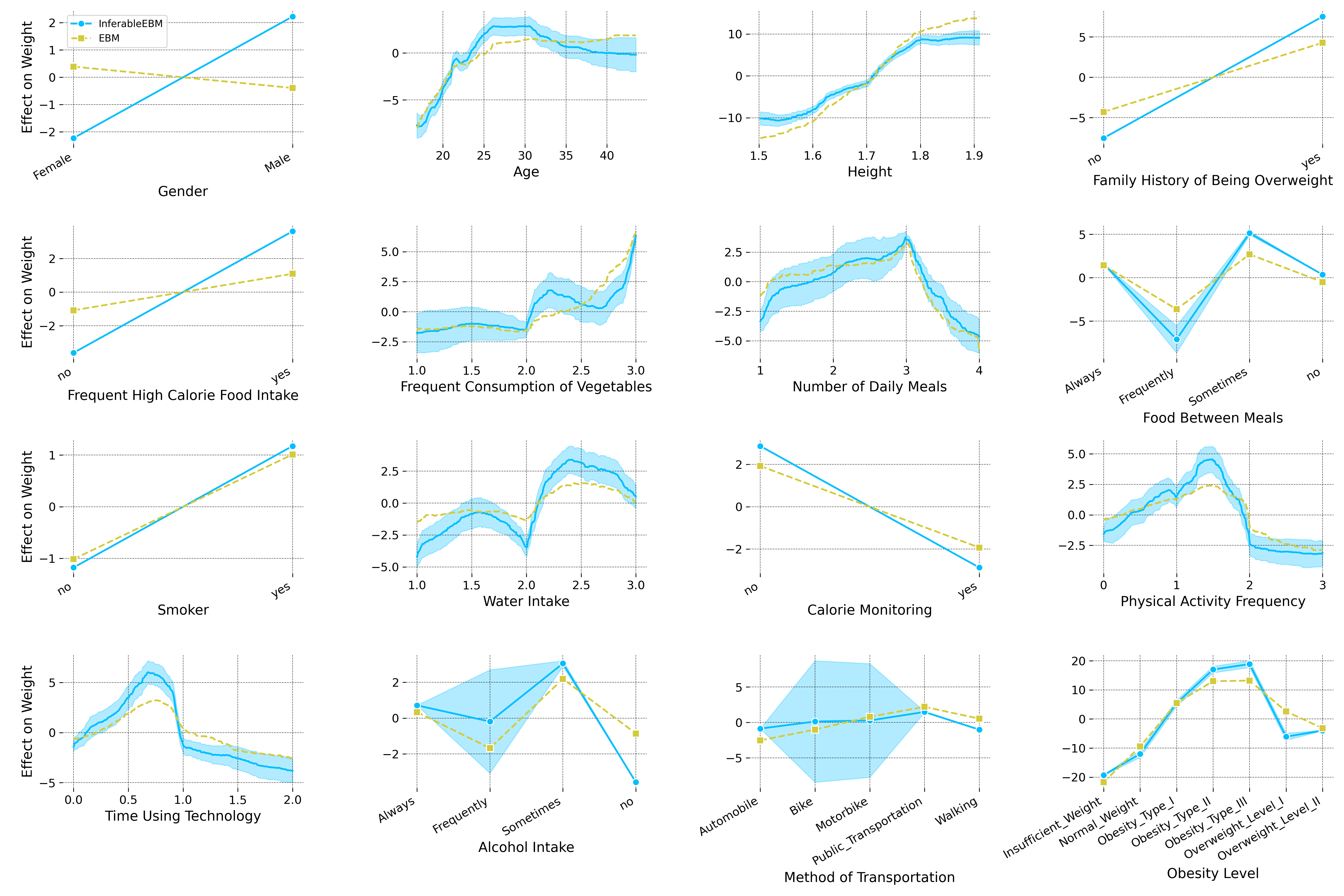}
		\caption{Full set of estimated feature effects for Algorithm \ref{alg:brebm_ident} and EBM proper \cite{nori2019interpretmlunifiedframeworkmachine}, when regressing weight on the given set of covariates. Algorithm \ref{alg:brebm_ident} by-and-large yields qualitatively similar estimates, with confidence intervals for each term displayed in the shaded area. The terms roughly agree apart from the effect of gender on weight, which we argue should be higher for men than for women.}
		\label{fig:obesity-feature-effects}
	\end{figure}
	
\end{document}